\documentclass[nohyperref]{article}

\usepackage{microtype}
\usepackage{graphicx}
\usepackage{subfigure}
\usepackage{booktabs} 

\usepackage{hyperref}

\usepackage[accepted]{icml2023}

\usepackage{amsmath}
\usepackage{amssymb}
\usepackage{mathtools}
\usepackage{amsthm}
\usepackage{bbm}
\usepackage{wrapfig}
\usepackage[utf8]{inputenc} 
\usepackage[T1]{fontenc}    
\usepackage{hyperref}       
\usepackage{url}            
\usepackage{booktabs}       
\usepackage{amsfonts}       
\usepackage{nicefrac}       
\usepackage{microtype}      
\usepackage{xcolor}         
\usepackage[utf8]{inputenc} 
\usepackage[T1]{fontenc}    
\usepackage{url}            
\usepackage{booktabs}       
\usepackage{amsfonts}       
\usepackage{nicefrac}       
\usepackage{microtype}      
\usepackage{lipsum}		
\usepackage{graphicx}
\usepackage{natbib}
\usepackage{doi}
\usepackage{bbm}
\usepackage{caption}
\usepackage[american]{babel}
\usepackage{filecontents}
\usepackage{microtype}
\usepackage{subfigure}
\usepackage{booktabs} 
\usepackage[normalem]{ulem}
\usepackage{graphicx}
\usepackage{caption}
\usepackage{tikz}
\usetikzlibrary{arrows.meta, automata,
                positioning,
                quotes}
 \usepackage{amsmath, amssymb}
\usepackage{pst-node}
\usepackage{color}

\usepackage{algorithmic}
\usepackage{algorithm}
\usepackage{wrapfig}
\usepackage{lipsum}
\usepackage{verbatim}
\usepackage{amsmath}
\allowdisplaybreaks
\usepackage{amsthm}
\usepackage{microtype}
\usepackage{color}
\usepackage{eucal}
\usepackage{epsfig,amsmath,amssymb,amsfonts,amstext,amsthm,mathrsfs}
\usepackage{latexsym,graphics,epsf,epsfig,psfrag}
\usepackage{dsfont,color,epstopdf,fixmath}
\usepackage{enumitem}
\usepackage{algorithm,algorithmic,latexsym}
\theoremstyle{plain}
\newtheorem{theorem}{Theorem}[section]

\newtheorem{lemma}{Lemma}[section]

\theoremstyle{definition}
\newtheorem{definition}{Definition}[section]
\newtheorem{assumption}{Assumption}
\newtheorem{example}{Example}[section]
\theoremstyle{remark}
\newtheorem{remark}{Remark}[section]
\newcommand{\kp}{\mathsf P}
\newcommand{\cp}{\mathcal{P}}
\newcommand{\mcs}{\mathcal{S}}
\newcommand{\mca}{\mathcal{A}}
\newcommand{\nn}{\nonumber}
\newcommand{\mE}{\mathbb{E}}

\newcommand{\spa}{\text{Span}}

\usepackage[capitalize,noabbrev]{cleveref}

\usepackage[textsize=tiny]{todonotes}

\icmltitlerunning{Model-Free Robust Average-Reward Reinforcement Learning}

\begin{document}

\twocolumn[
\icmltitle{Model-Free Robust Average-Reward Reinforcement Learning}



\icmlsetsymbol{equal}{*}

\begin{icmlauthorlist}
\icmlauthor{Yue Wang}{a1}
\icmlauthor{Alvaro Velasquez}{a2}
\icmlauthor{George Atia}{a3}
\icmlauthor{Ashley Prater-Bennette}{a4}
\icmlauthor{Shaofeng Zou}{a1}

\end{icmlauthorlist}

\icmlaffiliation{a1}{University at Buffalo}
\icmlaffiliation{a2}{University of Colorado Boulder}
\icmlaffiliation{a3}{University of Central Florida}
\icmlaffiliation{a4}{Air Force Research Laboratory}
\icmlcorrespondingauthor{Yue Wang}{ywang294@buffalo.edu}
\icmlcorrespondingauthor{Shaofeng Zou}{szou3@buffalo.edu}

\icmlkeywords{Robust Bellman equation, }

\vskip 0.3in
]



\printAffiliationsAndNotice{}  
  
\begin{abstract}
Robust Markov decision processes (MDPs) address the challenge of model uncertainty  by optimizing the worst-case performance over an uncertainty set of MDPs. In this paper, we focus on the robust average-reward MDPs under the model-free setting. We first theoretically characterize the structure of solutions to the robust average-reward Bellman equation, which is essential for our later convergence analysis. We then design two model-free algorithms, robust relative value iteration (RVI) TD and robust RVI Q-learning, and theoretically prove their convergence to the optimal solution. We provide several widely used uncertainty sets as examples, including those defined by the contamination model, total variation, Chi-squared divergence, Kullback-Leibler (KL) divergence and Wasserstein distance.
\end{abstract}
\section{Introduction}
Reinforcement learning (RL) has demonstrated remarkable success in applications like robotics, finance, and games. However,  RL often suffers from a severe performance degradation when deployed in real-world environments, which is often due to the model mismatch between the training and testing environments. Such a model mismatch could be a result of non-stationary conditions, modeling errors between simulated and real systems, external perturbations and adversarial attacks. 
To address this challenge, a framework of robust MDPs and robust RL was developed in \cite{nilim2004robustness,iyengar2005robust,bagnell2001solving}, where an uncertainty set of MDPs is constructed, and a pessimistic approach is adopted to optimize the worst-case performance over the uncertainty set. Such a minimax approach guarantees performance for all MDPs within the uncertainty set, making it robust to model mismatch. 

Two performance criteria are commonly used for infinite-horizon MDPs: 1) the discounted-reward setting, where the reward is discounted exponentially with time; and 2) the average-reward setting, where the long-term average-reward over time is of interest. For systems that operate for an extended period of time, e.g., queue control, inventory management in supply chains, or communication networks, it is more important to optimize the average-reward since policies obtained from the discounted-reward setting may be myopic and have poor long-term performance \cite{kazemi2022translating}. 
However, much of the work on robust MDPs has focused on the discounted-reward setting, and robust average-reward MDPs remain largely unexplored. 

Robust MDPs under the two criteria are fundamentally different. 
%
%
%
%
Compared to the discounted-reward setting, the average-reward setting places equal weight on immediate and future rewards, thus depends on the limiting state-action frequencies of the underlying MDPs. This makes it more challenging to investigate than the discounted-reward setting. For example, to establish a contraction in the discounted-reward setting, it generally suffices to have a discount factor strictly less than one, whereas in the average-reward setting, convergence depends on the structure of the MDP. Studies on robust average-reward MDPs are quite limited in the literature, e.g.,  \cite{tewari2007bounded,lim2013reinforcement,wang2023robust}, and are model-based, where the uncertainty set is fully known to the learner. However, the more practical model-free setting, where only samples from the nominal MDP (the centroid of the uncertainty set) can be observed, has yet to be explored. Algorithms and fundamental 
convergence guarantees in this setting have not been established yet, which is the focus of this paper.


\subsection{Challenges and Contributions}
In this paper, we develop a fundamental characterization of solutions to the robust average-reward Bellman equation, design model-free algorithms with provable optimality guarantees, and construct an \textit{unbiased} estimator of the robust average-reward Bellman operator for various types of uncertainty sets. Our results fill the gap in model-free algorithms for robust average-reward RL, and are fundamentally different from existing studies on robust discounted RL and robust and non-robust average-reward MDPs.
In the following, we summarize the major challenges and our key contributions.

\textbf{We characterize the fundamental structure of solutions to the robust average-reward Bellman equation.} 
Value-based approaches typically converge to a solution to the (robust) Bellman equation. For example, TD and Q-learning converge to the \textit{unique} fixed-point that yields an optimal solution to the Bellman equation in the discounted setting \cite{sutton2018reinforcement}. However, in the (robust) average-reward setting, solutions to the (robust) average-reward Bellman equation may not be unique \cite{puterman1994markov,wang2023robust}, and the fundamental structure of the solution space usually plays an important role in proving convergence (e.g., \cite{wan2021learning,abounadi2001learning,tsitsiklis1999average}). In the non-robust average-reward setting, the solution to the average-reward Bellman equation is unique up to an additive constant \cite{puterman1994markov}. This result is largely due to the nice linear structure of the non-robust average-reward Bellman equation. However, in the robust setting, where the worst-case performance is of interest, the robust average-reward Bellman equation is non-linear, which poses a significant challenge in characterizing its solution. We demonstrate that the worst-case transition kernel may not be unique and then show that any solution to the robust average-reward Bellman equation is the relative value function for some worst-case transition kernel (up to an additive constant). We note that this structure is of key importance later in our convergence analysis.


 \textbf{We develop model-free approaches for robust policy evaluation and optimal control.} We then design model-free algorithms for robust average-reward RL. To ensure stability of the iterates, we adopt a similar idea as in the non-robust average-reward setting \cite{puterman1994markov,bertsekas2011dynamic}, which is to subtract an offset function and design robust RVI TD and Q-learning algorithms. 
 Unlike the convergence proof for model-free algorithms for non-robust average-reward MDPs, e.g., \cite{wan2021learning,abounadi2001learning} and robust discounted MDPs, e.g., \cite{wang2021online,liu2022distributionally}, where the Bellman operator is either linear or a contraction, our robust average-reward Bellman operator is neither linear nor a contraction, which makes the convergence analysis challenging. Nevertheless, based on the solution structure of the robust average-reward Bellman equation discussed above, we prove that our algorithms converge to solutions to the robust average-reward Bellman equations. Specifically, for robust RVI TD, its output converges to the worst-case average-reward; and for the robust RVI Q-learning, the greedy policy w.r.t. the Q-function converges to an optimal robust policy.  
 

\textbf{We construct unbiased estimators  with bounded variance  for robust Bellman operators under various uncertainty set models.} One popular approach in RL is  bootstrapping, i.e., use $\mathbf H Q$ as an estimate of the true $Q$-function, where $\mathbf H$ is the (robust) Bellman operator. In the non-robust setting, an unbiased estimate of $\mathbf H$ can be easily derived because $\mathbf H$ is linear in the transition kernel. This linearity, however, does not hold in the robust setting since we minimize over the transition kernel to account for the worst-case and, therefore, directly plugging in the empirical  transition kernel results in a biased estimate. 
In this paper, we employ the multi-level Monte-Carlo method \cite{blanchet2015unbiased} and construct unbiased estimators for five popular uncertainty models, including those defined by contamination models, total variation, Chi-squared divergence, Kullback-Leibler (KL) divergence, and Wasserstein distance. 
We then prove that our estimator is unbiased and has a bounded variance. These properties play an important role in establishing the convergence of our algorithms.



\subsection{Related Work}

\noindent
\textbf{Robust average-reward MDPs.} Studies on robust average-reward MDPs are quite limited in the literature. Model-based robust average-reward MDPs were first studied in \cite{tewari2007bounded}, where the authors showed the existence of the Blackwell optimal policy for a specific uncertainty set of bounded parameters. Results for general uncertainty sets were developed in recent work \cite{wang2023robust}, following a fundamentally different approach from \cite{tewari2007bounded}. However, these two methods are model-based. The work of \cite{lim2013reinforcement} designed a model-free algorithm for the uncertainty set defined by total variation and characterized its regret bound. However, their method cannot be generalized to other uncertainty sets. In this paper, we design the first model-free method for general uncertainty sets and provide fundamental insights into the model-free robust average-reward RL problem.

\noindent
\textbf{Robust discounted-reward MDPs.}
Model-based methods for robust discounted MDPs were studied in \cite{iyengar2005robust,nilim2004robustness,bagnell2001solving,satia1973markovian,wiesemann2013robust,lim2019kernel,xu2010distributionally,yu2015distributionally,lim2013reinforcement,tamar2014scaling,neufeld2022robust}, where the uncertainty set is assumed to be known, and the problem can be solved using robust dynamic programming. Later, the studies were generalized to the model-free setting \cite{roy2017reinforcement,badrinath2021robust,wang2021online,wang2022policy,tessler2019action,liu2022distributionally,zhou2021finite,yang2021towards,panaganti2021sample,goyal2018robust,kaufman2013robust,ho2018fast,ho2021partial}. Our work focuses on the average-reward setting. First, the robust average-reward Bellman operator does not come with a simple contraction property like the one for the discounted setting, and the average-reward depends on the limiting behavior of the underlying MDP. Moreover, the average-reward Bellman function is a function of two variables: the average-reward and the relative value function, whereas its discounted-reward counterpart is a function of only the value function. These challenges make the robust average-reward problem more intricate. Furthermore, existing studies mostly focus on a certain type of uncertainty set, e.g., contamination \cite{wang2021online} and total variation \cite{panaganti2022robust}; in this paper, our method can be used to solve a wide range of uncertainty sets.

\noindent
\textbf{Non-robust average-reward MDPs.} Early contributions to non-robust average-reward MDPs include a fundamental characterization of the problem and model-based methods \cite{puterman1994markov,bertsekas2011dynamic}. Model-free methods in the tabular setting, e.g.,   RVI Q-learning \cite{abounadi2001learning} and differential Q-learning \cite{wan2021learning,wan2022convergence}, were developed recently and are both shown to converge to the optimal average-reward. There is also work on average-reward RL with function approximation, e.g., \cite{zhang2021finite,tsitsiklis1999average,zhang2021average,yu2009convergence}. In this paper, we focus on the robust setting, where the key challenge lies in the non-linearity of the robust average-reward Bellman equation, whereas it is linear in the non-robust setting.

\section{Preliminaries and Problem Formulation}


\vspace{0.2cm}
\noindent\textbf{Average-reward MDPs.}
An MDP  $(\mathcal{S},\mathcal{A}, \mathsf P, r)$ is specified by: a state space $\mcs$, an action space $\mca$, a transition kernel $\mathsf P=\left\{\kp^a_s \in \Delta(\mcs), a\in\mca, s\in\mcs\right\}$\footnote{$\Delta(\mcs)$ denotes the $(|\mcs|-1)$-dimensional probability simplex on $\mcs$. }, where $\kp^a_s$ is the distribution of the next state over $\mcs$ upon taking action $a$ in state $s$, and a reward function $r: \mcs\times\mca \to [0,1]$. At each time step $t$, the agent at state $s_t$ takes an action $a_t$, the  environment then transitions to the next state $s_{t+1}\sim \kp^{a_t}_{s_t}$, and provides a reward signal $r_t\in [0,1]$. 

A stationary policy $\pi: \mcs\to \Delta(\mca)$ maps a state to a distribution over $\mca$, following which the agent takes action $a$ at state $s$ with probability $\pi(a|s)$. 
Under a transition kernel $\kp$, the average-reward of $\pi$ starting from $s\in\mcs$ is defined as
\begin{align}
    g_\kp^\pi(s)\triangleq \lim_{T\to\infty} \mE_{\pi,\kp}\bigg[\frac{1}{T}\sum^{T-1}_{n=0} r_t|S_0=s \bigg].
\end{align}
The relative value function is defined to measure the cumulative difference between the reward and  $g^\pi_\kp$:
\begin{align}\label{eq:relativevaluefunction}
    V^\pi_\kp(s)\triangleq \mE_{\pi,\kp}\bigg[\sum^\infty_{t=0} (r_t-g^\pi_\kp)|S_0=s \bigg].
\end{align}
Then $(g^\pi_\kp, V^\pi_\kp)$  satisfies the following Bellman equation \cite{puterman1994markov}:
\begin{align}
    V^\pi_\kp(s)=\mE_{\pi,\kp}\bigg[r(s,A)-g^\pi_\kp(s)+\sum_{s'\in\mcs} p^A_{s,s'}V^\pi_\kp (s') \bigg]. 
\end{align}

\vspace{0.2cm}
\noindent\textbf{Robust average-reward MDPs.}
For robust MDPs, the transition kernel is assumed to be in some uncertainty set $\mathcal{P}$. At each time step, the environment transits to the next state according to an arbitrary transition kernel $\kp\in\cp$. In this paper, we focus on the $(s,a)$-rectangular uncertainty set \cite{nilim2004robustness,iyengar2005robust}, i.e., $\mathcal{P}=\bigotimes_{s,a} \mathcal{P}^a_s$, where $\mathcal{P}^a_s \subseteq \Delta(\mcs)$. 
Popular uncertainty sets include those defined by the contamination model \cite{hub65,wang2022policy},  total variation \cite{lim2013reinforcement}, Chi-squared divergence \cite{iyengar2005robust}, Kullback-Leibler (KL) divergence 
\cite{hu2013kullback} and Wasserstein distance \cite{gao2022distributionally}. We will investigate these uncertainty sets in detail in Section \ref{sec:case}.

We investigate the worst-case average-reward over the uncertainty set of MDPs. Specifically, define the  robust average-reward of a policy $\pi$ as 
\begin{align}\label{eq:Vdef}
    g^\pi_\cp(s)\triangleq \min_{\kappa\in\bigotimes_{n\geq 0} \mathcal{P}} \lim_{T\to\infty}\mathbb{E}_{\pi,\kappa}\left[\frac{1}{T}\sum^{T-1}_{t=0}r_t|S_0=s\right],
\end{align}
where $\kappa=(\mathsf P_0,\mathsf P_1...)\in\bigotimes_{n\geq 0} \mathcal{P}$. It was shown in \cite{wang2023robust} that the worst case under the time-varying model is equivalent to the one under the stationary model:
\begin{align}\label{eq:5}
    g^\pi_\cp(s)= \min_{\kp\in\mathcal{P}} \lim_{T\to\infty}\mathbb{E}_{\pi,\kp}\left[\frac{1}{T}\sum^{T-1}_{t=0}r_t|S_0=s\right].
\end{align}
Therefore, we limit our focus to the stationary model. We refer to the minimizers of \eqref{eq:5} as the worst-case transition kernels for the policy $\pi$, and denote the set of all possible worst-case transition kernels by $\Omega^\pi_g$, i.e., $\Omega^\pi_g \triangleq \{\kp\in\cp: g^\pi_\kp=g^\pi_\cp \}$. As shown in \Cref{ex} in the appendix, the worst-case transition kernel may not be unique.

We focus on the model-free setting, where only samples from the nominal MDP (centroid of the uncertainty set) are available. We investigate two problems: 1) given a policy $\pi$, estimate its robust average-reward $g^\pi_\cp$, and 2) find an optimal robust policy that optimizes the robust average-reward:
\begin{align}
    \max_{\pi} g^\pi_\cp(s), \text{ for any }s\in\mcs.
\end{align}
We denote
by $g^*_{\cp}(s)\triangleq \max_{\pi} g^\pi_{\cp}(s)$ the optimal robust average-reward.

\section{Robust RVI TD for Policy Evaluation}\label{sec:td}
In this section, we study the problem of policy evaluation, which aims to estimate the robust average-reward $g^\pi_\cp$ for a fixed policy $\pi$. 


For technical convenience, we make the following assumption to guarantee that the average-reward is independent of the initial state \cite{abounadi2001learning,wan2021learning,zhang2021average,zhang2021policy,chen2022learning,wang2023robust}.
\begin{assumption}\label{ass:sameg}
    The Markov chain induced by $\pi$ is a unichain for all $\kp\in\cp$. 
\end{assumption}
In general, the average-reward depends on the initial state. For example, imagine a policy that induces a multichain in an MDP with two closed communicating classes. A learning algorithm would be able to learn the average-reward for each communicating class; however, the average-rewards for the two classes may be different. To remove this complexity, it is common and convenient to rule out this possibility.
Under Assumption \ref{ass:sameg}, the average-reward w.r.t. any $\kp\in\cp$ is identical for any start state, i.e., $g^\pi_\kp(s)=g^\pi_\kp(s'), \forall s,s'\in\mcs$.


We first revisit the robust average-reward Bellman equation in \cite{wang2023robust}, and further characterize the structure of its solutions. For $V\in\mathbb R^{|\mcs|}$, denote by  $\sigma_{\cp^a_s}(V)\triangleq\min_{p\in\cp^a_s} p V$, $\kp_V(s,a)\triangleq \arg\min_{p\in\cp^a_s} p V$ and let $\kp_V=\{\kp_V(s,a),s\in\mcs,a\in\mca\}$. 
\begin{theorem}[Robust Bellman equation]\label{thm:robust Bellman} 
If $(g,V)$ is a solution to the robust Bellman equation
\begin{align}\label{eq:bellman}
    V(s)=\sum_a \pi(a|s) (r(s,a)-g+\sigma_{\cp^a_s}(V)), \forall s,
\end{align} 
then
1) $g=g^\pi_\cp$ \cite{wang2023robust};
2) $\kp_V\in\Omega^\pi_g$; 
3) $V=V^\pi_{\kp_V}+ce$ for some $c\in\mathbb R$, where $e$ denotes the vector $(1,1,...,1)\in\mathbb{R}^{|\mcs|}$.
\end{theorem}
The robust Bellman equation in \eqref{eq:bellman} was initially developed in \cite{wang2023robust}. The authors also defined a robust relative value function: 
\begin{align}\label{eq:rrvi}
    V_{\mathcal P}^\pi(s)\triangleq\min_{\kp\in\mathcal P}\mE_{\pi,\kp}\left[\sum^\infty_{n=0} (r_n-g^\pi_\cp)|S_0=s \right],
\end{align}
which is the worst-case relative value function, and showed that $(g^\pi_\cp,V_{\mathcal P}^\pi)$ is a solution to \eqref{eq:bellman}. However, conversely, it may not be the case that any solution $(g,V)$ to \eqref{eq:bellman} can be written as $V=V_{\mathcal P}^\pi+ce$ for some $c\in\mathbb R$. This is in contrast to the results for the non-robust average-reward setting, where the set of solutions to the non-robust average-reward Bellman equation can be written as $\{(g^\pi_\kp, V^\pi_\kp+ce):c\in\mathbb{R}\}$ \cite{puterman1994markov}.


In Theorem \ref{thm:robust Bellman}, we show that for any solution $(g,V)$ to \eqref{eq:bellman}, the transition kernel $P_V\in\Omega^\pi_g$, i.e., it is a worst-case transition kernel for $g^\pi_\cp$. Moreover, any solution to \eqref{eq:bellman} can be written as $V=V^\pi_{\kp_V}+ce$ for some constant $c$. This is different from the non-robust setting, as $V^\pi_{\kp_V}$ actually depends on $V$. As will be seen later, this result is crucial to establish the convergence of our robust RVI TD. 


\Cref{thm:robust Bellman} also reveals the fundamental difference between the robust and the non-robust average-reward settings. 
Under the non-robust setting, the solution set to the Bellman equation can be written as $\{(g^\pi_\kp, V^\pi_\kp+ce):c\in\mathbb{R}\}$. The solution is uniquely determined by the transition kernel (up to some constant vector $ke$). In contrast, in the robust setting, the robust Bellman equation is no longer linear. Any solution $V$ to \eqref{eq:bellman} is a relative value function w.r.t. \textit{some} worst-case transition kernel $\kp\in\Omega^\pi_g$ (up to some additive constant vector), i.e., $V\in\{ V^\pi_\kp+ce : \kp\in\Omega^\pi_g,c\in\mathbb{R}\}$. 

A natural question that arises is whether, for any $\kp\in\Omega^\pi_g$,  $(g^\pi_\cp,V^\pi_\kp)$ is a solution to \eqref{eq:bellman}?
\Cref{lemma:1}  refutes this.
\begin{lemma}\label{lemma:1}
    There exists a robust MDP such that for some $\kp\in\Omega^\pi_g$,  $(g^\pi_\cp,V^\pi_\kp)$ is not a solution to \eqref{eq:bellman}.
\end{lemma}
\Cref{lemma:1} implies that the solution set to \eqref{eq:bellman} is a subset of $\{ V^\pi_\kp+ce, \kp\in\Omega^\pi_g,c\in\mathbb{R}\}$. 
Note that explicit characterization of the solution set to \eqref{eq:bellman} is challenging due to its non-linear structure; however, result 3 in \Cref{thm:robust Bellman} suffices for the convergence proof (see appendix for proofs). 

Motivated by the robust Bellman equation in \eqref{eq:bellman}, we propose a model-free robust RVI TD algorithm in Algorithm \ref{alg:mfsTD}, where $\hat{\mathbf{T}}$ and function $f$ will be discussed later.
\begin{algorithm}[htb]
\caption{Robust RVI TD}
\label{alg:mfsTD}
\textbf{Input}: $V_0, \alpha_n, n=0,1,...,N-1$
\begin{algorithmic}[1] 
\FOR{$n=0,1,...,N-1$}
\FOR{all $s\in\mcs$}
\STATE {$V_{n+1}(s)\leftarrow V_n(s)+\alpha_n (\hat{\mathbf T}V_n(s)-f(V_n)-V_n(s))$}
\ENDFOR
\ENDFOR
\end{algorithmic}
\end{algorithm}

Note that \eqref{eq:bellman} can be written as $V=\mathbf T V-g$, where $\mathbf T$ is the robust average-reward Bellman operator. Since in the model-free setting $\mathcal P$ is unknown, in \Cref{alg:mfsTD}, we construct $\hat{\mathbf T} V$ as an estimate of $\mathbf TV$ satisfying 
\begin{align}\label{eq:t1}
\mE[\hat{\mathbf T}V]=\mathbf T V,\quad\text{Var}[\hat{\mathbf T}V(s)]\leq C(1+\|V\|^2),
\end{align} 
for some constant $C>0$. 
In this paper, if not specified, $\|\cdot\|$ denotes the infinity norm $\|\cdot\|_\infty$. 

It is challenging to construct such $\hat{\mathbf T}$ as $\mathbf T$ is non-linear in the nominal transition kernel from which samples are generated.
In \Cref{sec:case}, we will present in detail how to construct such $\hat{\mathbf T}$ for various uncertainty set models.

To make the iterates stable, we follow the idea of RVI in the non-robust setting and introduce an offset function $f$ satisfying the following assumption \cite{puterman1994markov}. 
\begin{assumption}\label{ass:f}
    $f: \mathbb{R}^{|\mcs|}\to\mathbb{R}$ is $L_f$-Lipschitz and satisfies  $f(e)=1, f(x+ce)=f(x)+c, f(cx)=cf(x)$, $\forall c\in\mathbb{R}$. 
\end{assumption}
Assumption \ref{ass:f} can be easily satisfied, e.g., $f(V)=V(s_0)$ for some reference state $s_0\in\mcs$, and $f(V)=\frac{\sum_s V(s)}{|\mcs|}$ \cite{abounadi2001learning}. 
Compared with the discounted setting, $f$ is critical here. As we discussed above, in the average-reward setting, the solution to the Bellman equation $V+ce$ can be arbitrarily large because $c$ can be any real number. This may lead to a non-convergent sequence $V_n$ (see, e.g., example 8.5.2 of \cite{puterman1994markov}). Hence, a function $f$ is introduced to "offset" $V_n$ and keep the iterates stable. Also, $f(V_n)$ serves as an estimator of the average-reward $g^\pi_\cp$, as we shall see later. 


We then assume the Robbins-Monro condition on the stepsize, and further show the convergence of robust RVI TD.
\begin{assumption}\label{ass:step-size}
    The stepsize $\{\alpha_n\}_{n=0}^\infty$ satisfies the Robbins-Monro condition, i.e., 
    $
        \sum^\infty_{n=0} \alpha_n=\infty, \quad
        \sum^\infty_{n=0} \alpha^2_n<\infty.
    $
\end{assumption}
\begin{theorem}[Convergence of robust RVI TD]\label{thm:td converge}
Under Assumptions \ref{ass:sameg},\ref{ass:f},\ref{ass:step-size}, and if $\hat{\mathbf{T}}$ satisfies \eqref{eq:t1}, then almost surely, $(f(V_n),V_n)$ converges to a solution to \eqref{eq:bellman} which may depend on the initialization.
\end{theorem}
The result implies that $f(V_n)\to g^\pi_\cp$ a.s., which means our robust RVI TD converges to the worst-case average-reward for the given policy $\pi$.

\begin{remark}
Our robust RVI TD algorithm is shown to converge to a solution to \eqref{eq:bellman}. This model-free result is the same as the model-based results in \cite{wang2023robust}. Though it was shown in \cite{wang2023robust} that $(g^\pi_\cp,V_{\mathcal P}^\pi)$ (defined in \eqref{eq:rrvi}) is a solution to \eqref{eq:bellman}, it is not guaranteed that the convergence is to $(g^\pi_\cp,V_{\mathcal P}^\pi+ce)$ for some $c$.
\end{remark}
\begin{remark}
As we discussed following the statement of Theorem \ref{thm:robust Bellman}, result 3) of Theorem \ref{thm:robust Bellman} is crucial to the convergence proof of \Cref{thm:td converge}.  Specifically, it is necessary in order to characterize the equilibrium of the associated ODE, and thus the limit of the iterates $f(V_n)\to g^\pi_\cp$.
\end{remark}

\section{Robust RVI Q-Learning for Control}
In this section, we study the problem of optimal control, which aims to find a policy that optimizes the robust average-reward: $\pi^*=\arg\max_\pi g^\pi_\cp.$


Similar to Assumption \ref{ass:sameg}, we make the following assumption to guarantee that the average-reward is independent of the initial state \cite{abounadi2001learning,li2022stochastic}.
\begin{assumption}\label{ass:g*same}
    The Markov chain induced by any $\kp\in\cp$ and any $\pi$ is a unichain. 
\end{assumption}

We first revisit the optimal robust Bellman equation in \cite{wang2023robust} and further present a characterization of its solutions. Consider a $Q$-function $Q: \mcs\times\mca\rightarrow \mathbb R$, and define $V_Q(s)=\max_a Q(s,a),\forall s\in\mcs$.


\begin{lemma}[Optimal robust Bellman equation]\label{thm:optimal robust Bellman}
If $(g,Q)$ is a solution to the optimal robust Bellman equation
\begin{align}\label{eq:optimal bellman}
    Q(s,a)=r(s,a)-g+\sigma_{\cp^a_s}(V_Q), \forall s,a ,
\end{align} 
then
1)  $g=g^*_\cp$ \cite{wang2023robust};
2) the greedy policy w.r.t. $Q$: $\pi_Q(s)=\arg\max_a Q(s,a)$ is an optimal robust policy \cite{wang2023robust};
3) $V_Q=V^{\pi_Q}_\kp+ce$ for some $\kp\in\Omega^{\pi_Q}_g, c\in\mathbb{R}$. 
\end{lemma}
According to result 2 in Lemma \ref{thm:optimal robust Bellman}, finding a solution to \eqref{eq:optimal bellman} is sufficient to get the optimal robust average-reward and to derive the optimal robust policy. 
We note that a complete characterization of the solution set to \eqref{eq:optimal bellman} can be obtained similarly to the one in result 3 of \Cref{thm:robust Bellman}. Here, we only provide its structure to simplify the presentation and avoid cumbersome notation. This structure as outlined in \Cref{thm:optimal robust Bellman} is sufficient for our convergence analysis.


We hence present the following model-free robust RVI Q-learning algorithm. 
\begin{algorithm}[htb]
\caption{Robust RVI Q-learning}
\label{alg:Q}
\textbf{Input}: $Q_0, \alpha_n$
\begin{algorithmic}[1] 
\FOR{$n=0,...,N-1$}
\FOR{all $s\in\mcs, a\in\mca$}
\STATE {$Q_{n+1}(s,a)\leftarrow Q_n(s,a)+\alpha_n \big(\hat{\mathbf H}Q_n(s,a)-f(Q_n)-Q_n(s,a)\big)$}
\ENDFOR
\ENDFOR
\end{algorithmic}
\end{algorithm}

Similar to the robust RVI TD algorithm, denote the optimal robust Bellman operator by ${\mathbf H}Q(s,a)\triangleq r(s,a)+\sigma_{\cp^a_s}(V_Q)$, and we construct an estimate $\hat{\mathbf H}$ such that for some finite constant $C$,
\begin{align}
    \mE[\hat{\mathbf H}Q]&=\mathbf H Q,\quad    \text{Var}[\hat{\mathbf H}Q(s,a)]\leq C(1+\|Q\|^2).\label{eq:q2}
\end{align}
In \Cref{sec:case}, we will present in detail how to construct such $\hat{\mathbf H}$ for various uncertainty set models.   

The following theorem shows the convergence of the robust RVI Q-learning to the optimal robust average-reward $g^*_{\cp}$ and the optimal robust policy $\pi^*$. 
\begin{theorem}[Convergence of robust RVI Q-learning]\label{thm:Q general converge}
Under Assumptions \ref{ass:f}, \ref{ass:step-size} and \ref{ass:g*same}, and if $\hat{\mathbf H}$ satisfies \eqref{eq:q2}, then almost surely,  $(f(Q_n),Q_n)$ converges to a solution to \eqref{eq:optimal bellman}, i.e., $f(Q_n)$ converges to $g^*_{\cp}$, and the greedy policy $\pi_{Q_n}(s)\triangleq \arg\max_a Q_n(s,a)$ converges to an optimal robust average-reward $\pi^*$.
\end{theorem}


\section{Case Studies}\label{sec:case}
In the previous two sections, we showed that if an unbiased estimator with bounded variance is available for the robust Bellman operator, then both robust algorithms proposed converge to the optimum. In this section, we present the design of these estimators for various uncertainty set models. 

The major challenge in designing the estimated operators satisfying \eqref{eq:t1} and \eqref{eq:q2} lies in estimating the support function $\sigma_{\cp_s^a}(V)$ unbiasedly with bounded variance using samples from the nominal transition kernel.
However, the nominal transition kernel $\kp$  in general is different from the worst-case transition kernel, and the straightforward estimator is shown to be biased. Specifically, if we centered at the empirical transition kernel $\hat{\kp}$ and construct the uncertainty set $\hat{\cp}$, then the estimator is biased:
$
    \mE[\sigma_{{\hat{\cp}}_s^a}(V)]\neq \sigma_{\cp_s^a}(V). 
$

We consider several widely-used uncertainty models including the contamination model, the total variation model, the Chi-square model, the KL-divergence model, and the Wasserstein distance model. 
We show that our estimators are unbiased and have bounded variance in the following theorem. We will present the design in later sections. 
\begin{theorem}\label{thm:case thm}
For each uncertainty set, denote its corresponding estimators by $\hat{\mathbf T}$ and $\hat{\mathbf H}$ as in \Cref{sec:con,sec:non}. Then, there exists some constant $C$, such that \eqref{eq:t1} and \eqref{eq:q2} hold.
\end{theorem}


In the following sections, we construct an operator $\hat{\sigma}_{\cp^a_s}$ to estimate the support function $\sigma_{\cp^a_s}$, $\forall s\in\mcs,a\in\mca$ for each uncertainty set. We further define the estimated robust Bellman operators as
$
    \hat{\mathbf T}V(s)\triangleq\sum_a \pi(a|s) (r(s,a)+\hat{\sigma}_{\cp^a_s}(V))
$
and 
$
    \hat{\mathbf H}Q(s,a)\triangleq r(s,a)+\hat{\sigma}_{\cp^a_s}(V_Q).
$
\subsection{Linear Model: Contamination Uncertainty Set}\label{sec:con}
The $\delta$-contamination uncertainty set is
$
    \cp^a_s=\{(1-\delta)\kp^a_s+\delta q: q\in\Delta(\mcs) \}, 
$
where $0<\delta<1$ is the radius.
Under this uncertainty set, the support function can be computed as 
$
    \sigma_{\cp^a_s}(V)=(1-\delta)\kp^a_s V+\delta \min_s V(s),
$
and this is linear in the nominal transition kernel $\kp^a_s$. We hence use the transition to the subsequent state to construct our estimator:
\begin{align}\label{eq:12}
    \hat{\sigma}_{\cp^a_s}(V)\triangleq (1-\delta)\gamma V(s')+\delta\min_x V(x),
\end{align}
where $s'$ is a subsequent state sample after $(s,a)$. 

\subsection{Non-Linear Models}\label{sec:non}
Unlike the $\delta$-contamination model, most uncertainty sets result in a non-linear support function of the nominal transition kernel. We will employ the approach of multi-level Monte-Carlo which is widely used in quantile estimation under stochastic environments \cite{blanchet2015unbiased,blanchet2019unbiased,wang2022unbiased} to construct an unbiased estimator with bounded variance. 
 
For any $s,a$, we first generate $N$ according to a geometric distribution with parameter $\Psi\in(0,1)$. Then, we take action $a$ at state $s$ for $2^{N+1}$ times, and observe $r(s,a)$ and the subsequent state $\{s_i'\}, i=1,...,2^{N+1}$. We divide these $2^{N+1}$ samples into two groups:  samples with odd indices, and samples with even indices. We then individually calculate the empirical distribution of $s'$ using the even-index samples, odd-index ones, all the samples, and the first sample:
$\hat{\kp}^{a,E}_{s,N+1}= \frac{1}{2^{N}}\sum_{i=1}^{2^{N}}\mathbbm{1}_{s_{2i}'},  
\quad\hat{\kp}^{a,O}_{s,N+1}= \frac{1}{2^{N}}\sum_{i=1}^{2^{N}}\mathbbm{1}_{s_{2i-1}'},$
$\hat{\kp}^{a}_{s,N+1}= \frac{1}{2^{N+1}}\sum_{i=1}^{2^{N+1}}\mathbbm{1}_{s_i'},\quad\hat{\kp}^{a,1}_{s,N+1}= \mathbbm{1}_{s_1'}.
$
Then, we use these estimated transition kernels as nominal kernels to construct four estimated uncertainty sets $\hat{\cp}^{a,E}_{s,N+1},\hat{\cp}^{a,O}_{s,N+1},\hat{\cp}^{a}_{s,N+1},\hat{\cp}^{a,1}_{s,N+1}$. The multi-level estimator is then defined as 
\begin{align}\label{eq:est1}
    \hat{\sigma}_{\cp^a_s}(V)\triangleq \sigma_{\hat{\cp}^{a,1}_{s,N+1}}(V)+\frac{\Delta_N(V)}{p_N},
\end{align}
where  $p_N=\Psi(1-\Psi)^{N}$ and
\begin{align}
    \Delta_N(V)\triangleq \sigma_{\hat{\cp}^{a}_{s,N+1}}(V)-\frac{\sigma_{\hat{\cp}^{a,E}_{s,N+1}}(V)+\sigma_{\hat{\cp}^{a,O}_{s,N+1}}(V)}{2}.\nn
\end{align}
We note that in previous results of the multi-level Monte-Carlo estimator \cite{blanchet2015unbiased,blanchet2019unbiased,wang2022unbiased}, several assumptions are needed to show that the estimator is unbiased. These assumptions, however, do not hold in our cases. For example, the function $\sigma_{\cp}(V)$ is not continuously differentiable. Hence, their analysis cannot be directly applied here. 

We then present four examples of non-linear uncertainty sets. Under each example, a solution to the support function $\sigma_{\cp}(V)$ is given, and by plugging it into \eqref{eq:est1} the unbiased estimator can then be constructed. More details can be found in \Cref{sec:case for TD} and \Cref{sec:case for Q}.

\textbf{Total Variation Uncertainty Set.}
The total variation uncertainty set is  
$
    \cp^a_s=\{q\in\Delta(|\mcs|): \frac{1}{2}\|q-\kp^a_s\|_1\leq \delta \},
$
and the support function can be computed using its dual function \cite{iyengar2005robust}: 
\begin{align}\label{eq:tv dual}
    \sigma_{\cp^a_s}(V)
    &=\max_{\mu\geq 0}\big(\kp^a_s(V-\mu)-\delta \spa(V-\mu) \big).
\end{align}

\textbf{Chi-square Uncertainty Set.}
The Chi-square uncertainty set is 
$
    \cp^a_s=\{q\in\Delta(|\mcs|): d_c(\kp^a_s,q)\leq \delta \},
$
where $d_c(q,p)=\sum_{s}\frac{(p(s)-q(s))^2}{p(s)}$.
Its support function can be computed using its dual function \cite{iyengar2005robust}: 
\begin{align}\label{eq:cs dual}
    \sigma_{\cp^a_s}(V)=\max_{\mu\geq 0}\big(\kp^a_s(V-\mu)-\sqrt{\delta \text{Var}_{\kp^a_s} (V-\mu)} \big).
\end{align}

\textbf{Kullback–Leibler (KL) Divergence Uncertainty Set.}
The KL-divergence between two distributions $p,q$ is defined as 
$
    D_{KL}(q||p)=\sum_s q(s)\log\frac{q(s)}{p(s)},
$
and the uncertainty set defined via KL divergence is
\begin{align}
\cp_s^a=\left\{q: D_{KL}(q||\kp_s^a)\leq \delta\right\},\forall s\in\mcs, a\in\mca.
\end{align}  
Its support function can be efficiently solved using the duality result in \cite{hu2013kullback}:
\begin{align}
\sigma_{\cp^a_s}(V)=
    -\min_{\alpha\geq 0} \left(\delta\alpha+\alpha\log \left( \mE_{\kp^a_s}\big[ e^{\frac{-V}{\alpha}}\big]\right) \right).
\end{align}

The above estimator for the KL-divergence uncertainty set has also been developed in \cite{liu2022distributionally} for robust discounted MDPs. Its extension to our average-reward setting is similar. 

\textbf{Wasserstein Distance Uncertainty Sets.}
Consider the metric space $(\mathcal{S},d)$ by defining some distance metric $d$. For some parameter $l\in[1,\infty)$ and two distributions $p,q\in\Delta(\mathcal{S})$, define the $l$-Wasserstein distance between them as 
$W_l(q,p)=\inf_{\mu\in\Gamma(p,q)}\|d\|_{\mu,l}$, where $\Gamma(p,q)$ denotes the distributions over $\mathcal{S}\times\mathcal{S}$ with marginal distributions $p,q$, and $\|d\|_{\mu,l}=\big(\mE_{(X,Y)\sim \mu}\big[d(X,Y)^l\big]\big)^{1/l}$. The Wasserstein distance uncertainty set is then defined as 
\begin{align}
    \cp^a_s=\left\{q\in\Delta(|\mcs|): W_l(\kp^a_s,q)\leq \delta \right\}.
\end{align}
To solve the support function w.r.t. the Wasserstein distance set, we first prove the following duality lemma.
\begin{lemma}
It holds that 
\begin{align}\label{eq:wd dual}
    &\sigma_{\cp^a_s}(V)=\sup_{\lambda\geq 0}\left(-\lambda\delta^l+\mE_{\kp^a_{s}}\big[\inf_{y}\big(V(y)+\lambda d(S,y)^l \big)\big] \right).
\end{align}
\end{lemma}

Thus, the support function can be solved using its dual form, and the estimator can then be constructed following  \eqref{eq:est1}. 


\section{Experiments}\label{sec:exp}
We numerically verify our previous convergence results and demonstrate the robustness of our algorithms. 
Additional experiments can be found in \Cref{sec:add_exp}.

\subsection{Convergence of Robust RVI TD and Q-Learning}\label{sec:exp_gar}
We first verify the convergence of our robust RVI TD and Q-learning algorithms under a Garnet problem $\mathcal{G}(30,20)$ \cite{archibald1995generation}. There are $30$ states and $20$ actions. The nominal transition kernel $\kp=\{ \kp^a_s,s\in\mcs,a\in\mca \}$ is randomly generated by a normal distribution: $\kp^a_s \sim \mathcal{N}(1,\sigma^a_s)$ and then normalized, and the reward function $r(s,a)\sim \mathcal{N}(1,\mu^a_s)$, where $\mu^a_s,\sigma^a_s \sim \textbf{Uniform}[0,100]$. We set $\delta=0.4$, $\alpha_n=0.01$,  $f(V)=\frac{\sum_s V(s)}{|\mcs|}$ and $f(Q)=\frac{\sum_{s,a}Q(s,a)}{|\mcs||\mca|}$. Due to the space limit, we only show the results under the Chi-square and Wasserstein Distance models. The results under the other three uncertainty sets are presented in \Cref{sec:add_exp}.

For policy evaluation, we evaluate the robust average-reward of the uniform policy $\pi(a|s)=\frac{1}{|\mca|}$. We implement our robust RVI TD algorithm under different uncertainty models. We run the algorithm independently for $30$ times and plot the average value of $f(V)$ over all $30$ trajectories. We also plot the 95th and 5th percentiles of the 30 curves as the upper and lower envelopes of the curves. To compare, we plot the true robust average-reward computed using the model-based robust value iteration method in \cite{wang2023robust}. It can be seen from the results in Figure \ref{Fig.TD} that our robust RVI TD algorithm converges to the true robust average-reward value.

\begin{figure}[!h]
\vskip -0.2in
\begin{center}
\subfigure{
\label{Fig.CSTD}
\includegraphics[width=0.47\linewidth]{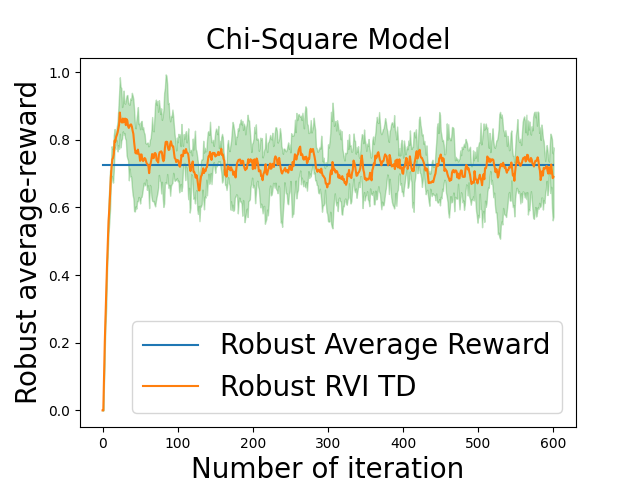}}
\subfigure{
\label{Fig.WDTD}
\includegraphics[width=0.47\linewidth]{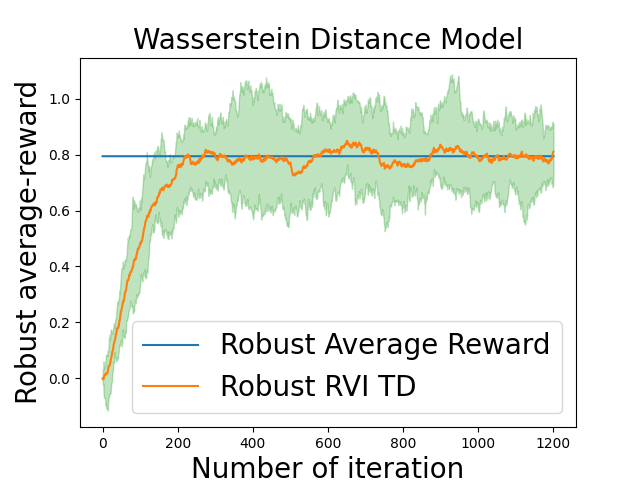}}
\caption{Robust RVI TD Algorithm.}
\label{Fig.TD}
\end{center}
\vskip -0.2in
\end{figure}
 
We then consider policy optimization. We run our robust RVI Q-learning independently for 30 times. The curves in Figure \ref{Fig.Q} show the average value of $f(Q)$ over 30 trajectories, and the upper/lower envelopes are the 95/5 percentiles. We also plot the optimal robust average-reward $g^*_\cp$ computed by the model-based RVI method in \cite{wang2023robust}. Our robust RVI Q-learning converges to the optimal robust average-reward $g^*_\cp$ under each uncertainty set, which verifies our theoretical results. 
\begin{figure}[!h]
\vskip -0.2in
\begin{center}
\subfigure{
\label{Fig.CSQ}
\includegraphics[width=0.47\linewidth]{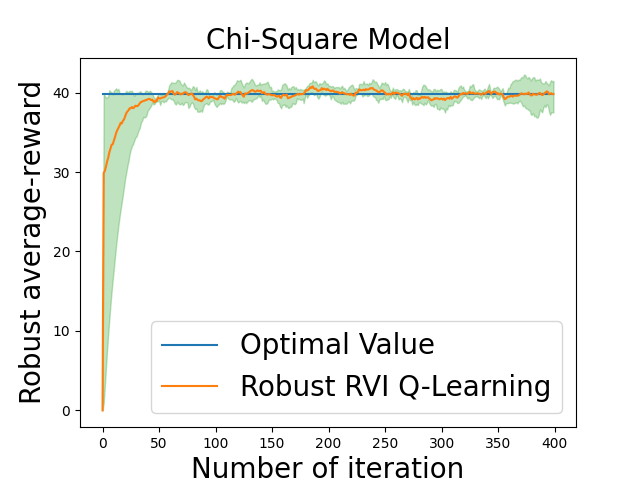}}
\subfigure{
\label{Fig.WDQ}
\includegraphics[width=0.47\linewidth]{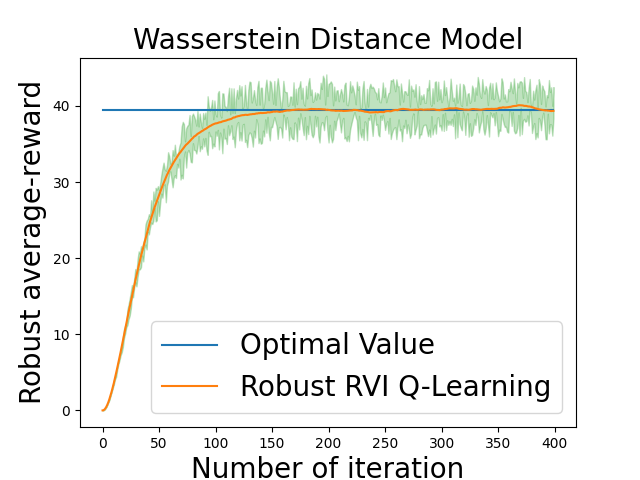}}
\caption{Robust RVI Q-Learning Algorithm.}
\label{Fig.Q}
\end{center}
\vskip -0.2in
\end{figure}
\subsection{Robustness of Robust RVI Q-Learning}
We then demonstrate the robustness of our robust RVI Q-learning  by showing that our method achieves a higher average-reward when there is model deviation between training and evaluation. 
\subsubsection{Recycling Robot}
We first consider the recycling robot problem (Example 3.3 \cite{sutton2018reinforcement}). A mobile robot running on a rechargeable battery aims to collect empty soda cans. It has 2 battery levels: low and high. The robot can either 1) search for empty cans; 2) remain stationary and wait for someone to bring it a can; 3) go back to its home base to recharge. Under low (high) battery level, the robot finds an empty can with probabilities $\alpha$ ($\beta$), and remains at the same battery level.  If the robot goes out to  search but finds nothing, it will run out of its battery and can only be carried back by human. More details can be found in \cite{sutton2018reinforcement}. 

In this experiment, the uncertainty lies in the probabilities $\alpha,\beta$ of finding an empty can if the robot chooses the action `search'. We set $\delta=0.4$ and implement our algorithms and vanilla Q-learning under the nominal environment ($\alpha=\beta=0.5$) with stepsize $0.01$. To show the difference among the policies that the algorithms learned, we plot the difference of $Q$ values at low battery level in \Cref{Fig.robot1}. In the low battery level, the robust algorithms find conservative policies which choose to wait instead of search, whereas the vanilla Q-learning finds a policy that chooses to search. To test the robustness of the obtained policies, we evaluate the average-reward of the learned policies in perturbed environments. Specifically, let $x$ denote the amplitude of the perturbation. Then, we estimate the worst performance of the two policies over the testing uncertainty set $(0.5-x,0.5+x)$, and plot them in \Cref{Fig.robot2}. It can be seen that when the perturbation is small, the true worst-case kernels (w.r.t. $\delta$ during training) are far from the testing environment, and hence the vanilla Q-learning has a higher reward; however, as the perturbation level becomes larger, the testing environment gets closer to the worst-case kernels, and then our robust algorithms perform better. It can be seen that the performance of Q-learning decreases rapidly while our robust algorithm is stable and outperforms the non-robust Q-learning. This implies that our algorithm is robust to the model uncertainty.


\begin{figure}[ht]
\vskip -0.2in
\begin{center}
\subfigure[$Q(\text{low,wait})$-$Q(\text{low,search})$ ]{
\label{Fig.robot1}
\includegraphics[width=0.47\linewidth]{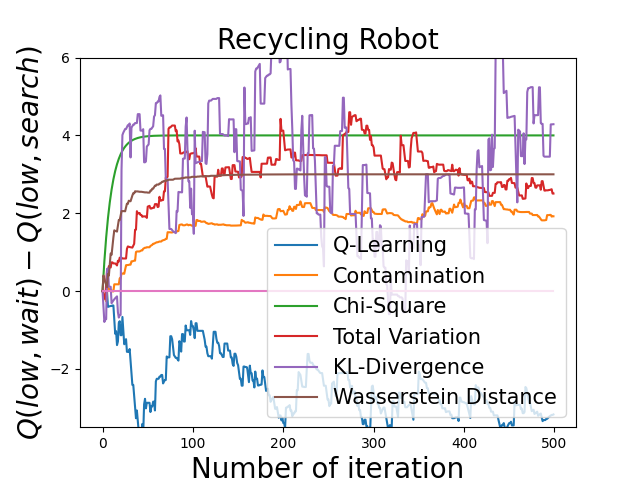}}
\subfigure[Perturbed environment]{
\label{Fig.robot2}
\includegraphics[width=0.47\linewidth]{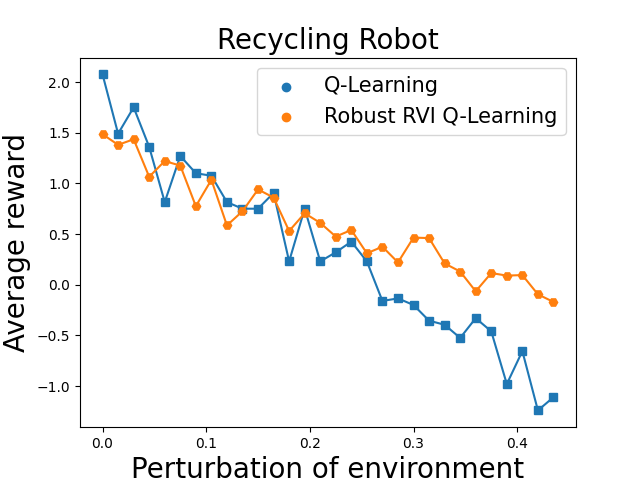}}
\caption{Recycling Robot.}
\label{Fig.robot}
\end{center}
\vskip -0.2in
\end{figure}

\subsubsection{Inventory Control Problem}
We now consider the supply chain problem \cite{giannoccaro2002inventory,kemmer2018reinforcement,liu2022distributionally}. At the beginning of each month, the manager of a warehouse inspects the current inventory of a product. Based on the current stock, the manager decides whether or not to order additional stock from a supplier. During this month, if the customer demand is satisfied, the warehouse can make a sale and obtain profits; but if not, the warehouse will obtain a penalty associated with being unable to satisfy customer demand for the product. The warehouse also needs to pay the holding cost for the remaining stock and new items ordered. The goal is to maximize the average profit. 

We let $s_t$ denote the inventory at the beginning of the $t$-th month, $D_t$ be a random demand during this month, and $a_t$ be the number of units ordered by the manager. We assume that $D_t$ follows some distribution and is independent over time. When the agent takes action $a_t$, the order cost is $a_t$, and the holding cost is $3 \cdot (s_t+a_t)$. If the demand $D_t \leq s_t+a_t$, then selling the item brings $5 \cdot D_t$ in total; but if the demand $D_t>s_t+a_t$, then there will not be any sale and a penalty of $-15$ will be received. We set $\mcs=\{0,1,...,16\}$ and $\mca=\{0,...,8\}$. 

We first set $\delta=0.4$ and  $\alpha_t=0.01$, and implement our algorithms and vanilla Q-learning under the nominal environment where $D_t\sim \textbf{Uniform}(0,16)$ is generated following the uniform distribution. To verify the robustness, we test the obtained policies under different perturbed environments. More specifically, we perturb the distribution of the demand to $D_t\sim U_{(m,b)}$, where 
\begin{equation}
    U_{(m,b)}(x)= \left\{ \begin{array}{lr} \frac{1}{|\mcs|}+b\frac{|\mcs|-2}{2|\mcs|}, \text{ if } x\in\{m,m+1\},&\nn\\
    \frac{1-b}{|\mcs|}, \text{ else}.& 
    \end{array}
    \right.
\end{equation}

The results are plotted in \Cref{Fig.ic}. We first fix $m=0$ and plot the performance under different values of $b$ in \Cref{Fig.icb}, then we fix $b=0.25$ and plot the performance under different values of $m$ in \cref{Fig.icm}. 

As the results show, when $b$ is small, i.e., the perturbation of the environment is small, the non-robust Q-learning obtains higher reward than our robust methods; as $b$ becomes larger, the performance of the non-robust method decreases rapidly, while our robust methods are more robust and outperform the non-robust one. When $b$ is fixed, our robust methods outperform  the non-robust Q-learning, which also demonstrates the robustness of our methods. 
\begin{figure}[ht]
\vskip -0.2in
\begin{center}
\subfigure[$m=0$ ]{
\label{Fig.icb}
\includegraphics[width=0.47\linewidth]{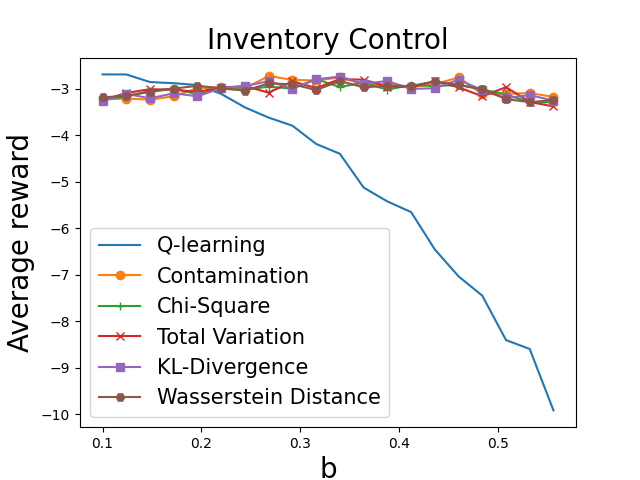}}
\subfigure[$b=0.25$]{
\label{Fig.icm}
\includegraphics[width=0.47\linewidth]{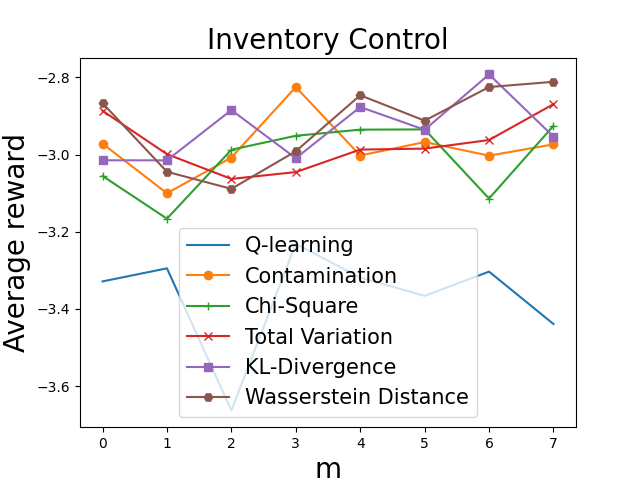}}
\caption{Inventory Control.}
\label{Fig.ic}
\end{center}
\vskip -0.2in
\end{figure}
\section{Conclusion}
In this paper, we developed the first model-free algorithms with provable convergence and optimality guarantee for robust average-reward RL under a broad range of uncertainty set models. We characterized the fundamental structure of solutions to the robust average-reward Bellman equation, which is crucial for the convergence analysis. We designed model-free robust algorithms based on the ideas of relative value iteration for non-robust average-reward MDPs and the robust average-reward Bellman equation. 
We developed concrete solutions to five popular uncertainty sets, where we generalized the idea of multi-level Monte-Carlo and constructed an unbiased estimate of the non-linear robust average-reward Bellman operator.  

\section{Acknowledgments}
This work is supported by the National Science Foundation under Grants CCF-2106560, CCF-2007783, CCF-2106339, and CCF-1552497. This material is based upon work supported under the AI Research Institutes program by National Science Foundation and the Institute of Education Sciences, U.S. Department of Education through Award \# 2229873 - National AI Institute for Exceptional Education. Any opinions, findings and conclusions or recommendations expressed in this material are those of the author(s) and do not necessarily reflect the views of the National Science Foundation, the Institute of Education Sciences, or the U.S. Department of Education.

\bibliography{icml}

\begin{thebibliography}{53}
\providecommand{\natexlab}[1]{#1}
\providecommand{\url}[1]{\texttt{#1}}
\expandafter\ifx\csname urlstyle\endcsname\relax
  \providecommand{\doi}[1]{doi: #1}\else
  \providecommand{\doi}{doi: \begingroup \urlstyle{rm}\Url}\fi

\bibitem[Abounadi et~al.(2001)Abounadi, Bertsekas, and
  Borkar]{abounadi2001learning}
Abounadi, J., Bertsekas, D., and Borkar, V.~S.
\newblock Learning algorithms for {M}arkov decision processes with average
  cost.
\newblock \emph{SIAM Journal on Control and Optimization}, 40\penalty0
  (3):\penalty0 681--698, 2001.

\bibitem[Archibald et~al.(1995)Archibald, McKinnon, and
  Thomas]{archibald1995generation}
Archibald, T., McKinnon, K., and Thomas, L.
\newblock {On the generation of {M}arkov decision processes}.
\newblock \emph{Journal of the Operational Research Society}, 46\penalty0
  (3):\penalty0 354--361, 1995.

\bibitem[Badrinath \& Kalathil(2021)Badrinath and
  Kalathil]{badrinath2021robust}
Badrinath, K.~P. and Kalathil, D.
\newblock Robust reinforcement learning using least squares policy iteration
  with provable performance guarantees.
\newblock In \emph{Proc. International Conference on Machine Learning (ICML)},
  pp.\  511--520. PMLR, 2021.

\bibitem[Bagnell et~al.(2001)Bagnell, Ng, and Schneider]{bagnell2001solving}
Bagnell, J.~A., Ng, A.~Y., and Schneider, J.~G.
\newblock Solving uncertain {M}arkov decision processes.
\newblock 2001.

\bibitem[Bertsekas(2011)]{bertsekas2011dynamic}
Bertsekas, D.~P.
\newblock {Dynamic Programming and Optimal Control 3rd edition, volume II}.
\newblock \emph{Belmont, MA: Athena Scientific}, 2011.

\bibitem[Blanchet \& Glynn(2015)Blanchet and Glynn]{blanchet2015unbiased}
Blanchet, J.~H. and Glynn, P.~W.
\newblock Unbiased {M}onte {C}arlo for optimization and functions of
  expectations via multi-level randomization.
\newblock In \emph{2015 Winter Simulation Conference (WSC)}, pp.\  3656--3667.
  IEEE, 2015.

\bibitem[Blanchet et~al.(2019)Blanchet, Glynn, and Pei]{blanchet2019unbiased}
Blanchet, J.~H., Glynn, P.~W., and Pei, Y.
\newblock Unbiased multilevel {M}onte {C}arlo: Stochastic optimization,
  steady-state simulation, quantiles, and other applications.
\newblock \emph{arXiv preprint arXiv:1904.09929}, 2019.

\bibitem[Borkar(2009)]{borkar2009stochastic}
Borkar, V.~S.
\newblock \emph{Stochastic approximation: a dynamical systems viewpoint},
  volume~48.
\newblock Springer, 2009.

\bibitem[Borkar \& Soumyanatha(1997)Borkar and Soumyanatha]{borkar1997analog}
Borkar, V.~S. and Soumyanatha, K.
\newblock An analog scheme for fixed point computation. i. theory.
\newblock \emph{IEEE Transactions on Circuits and Systems I: Fundamental Theory
  and Applications}, 44\penalty0 (4):\penalty0 351--355, 1997.

\bibitem[Brockman et~al.(2016)Brockman, Cheung, Pettersson, Schneider,
  Schulman, Tang, and Zaremba]{brockman2016openai}
Brockman, G., Cheung, V., Pettersson, L., Schneider, J., Schulman, J., Tang,
  J., and Zaremba, W.
\newblock {OpenAI Gym}.
\newblock \emph{arXiv preprint arXiv:1606.01540}, 2016.

\bibitem[Chen et~al.(2022)Chen, Jain, and Luo]{chen2022learning}
Chen, L., Jain, R., and Luo, H.
\newblock Learning infinite-horizon average-reward {M}arkov decision processes
  with constraints.
\newblock \emph{arXiv preprint arXiv:2202.00150}, 2022.

\bibitem[Gao \& Kleywegt(2022)Gao and Kleywegt]{gao2022distributionally}
Gao, R. and Kleywegt, A.
\newblock Distributionally robust stochastic optimization with {W}asserstein
  distance.
\newblock \emph{Mathematics of Operations Research}, 2022.

\bibitem[Giannoccaro \& Pontrandolfo(2002)Giannoccaro and
  Pontrandolfo]{giannoccaro2002inventory}
Giannoccaro, I. and Pontrandolfo, P.
\newblock Inventory management in supply chains: a reinforcement learning
  approach.
\newblock \emph{International Journal of Production Economics}, 78\penalty0
  (2):\penalty0 153--161, 2002.

\bibitem[Goyal \& Grand-Clement(2018)Goyal and Grand-Clement]{goyal2018robust}
Goyal, V. and Grand-Clement, J.
\newblock Robust {M}arkov decision process: Beyond rectangularity.
\newblock \emph{arXiv preprint arXiv:1811.00215}, 2018.

\bibitem[Ho et~al.(2018)Ho, Petrik, and Wiesemann]{ho2018fast}
Ho, C.~P., Petrik, M., and Wiesemann, W.
\newblock Fast {B}ellman updates for robust {MDP}s.
\newblock In \emph{Proc. International Conference on Machine Learning (ICML)},
  pp.\  1979--1988. PMLR, 2018.

\bibitem[Ho et~al.(2021)Ho, Petrik, and Wiesemann]{ho2021partial}
Ho, C.~P., Petrik, M., and Wiesemann, W.
\newblock Partial policy iteration for {L1}-robust {M}arkov decision processes.
\newblock \emph{Journal of Machine Learning Research}, 22\penalty0
  (275):\penalty0 1--46, 2021.

\bibitem[Hu \& Hong(2013)Hu and Hong]{hu2013kullback}
Hu, Z. and Hong, L.~J.
\newblock Kullback-{L}eibler divergence constrained distributionally robust
  optimization.
\newblock \emph{Available at Optimization Online}, pp.\  1695--1724, 2013.

\bibitem[Huber(1965)]{hub65}
Huber, P.~J.
\newblock A robust version of the probability ratio test.
\newblock \emph{Ann. Math. Statist.}, 36:\penalty0 1753--1758, 1965.

\bibitem[Iyengar(2005)]{iyengar2005robust}
Iyengar, G.~N.
\newblock Robust dynamic programming.
\newblock \emph{Mathematics of Operations Research}, 30\penalty0 (2):\penalty0
  257--280, 2005.

\bibitem[Kaufman \& Schaefer(2013)Kaufman and Schaefer]{kaufman2013robust}
Kaufman, D.~L. and Schaefer, A.~J.
\newblock Robust modified policy iteration.
\newblock \emph{INFORMS Journal on Computing}, 25\penalty0 (3):\penalty0
  396--410, 2013.

\bibitem[Kazemi et~al.(2022)Kazemi, Perez, Somenzi, Soudjani, Trivedi, and
  Velasquez]{kazemi2022translating}
Kazemi, M., Perez, M., Somenzi, F., Soudjani, S., Trivedi, A., and Velasquez,
  A.
\newblock Translating omega-regular specifications to average objectives for
  model-free reinforcement learning.
\newblock In \emph{Proceedings of the 21st International Conference on
  Autonomous Agents and Multiagent Systems}, pp.\  732--741, 2022.

\bibitem[Kemmer et~al.(2018)Kemmer, von Kleist, de~Rochebou{\"e}t,
  Tziortziotis, and Read]{kemmer2018reinforcement}
Kemmer, L., von Kleist, H., de~Rochebou{\"e}t, D., Tziortziotis, N., and Read,
  J.
\newblock Reinforcement learning for supply chain optimization.
\newblock In \emph{European Workshop on Reinforcement Learning}, volume~14,
  2018.

\bibitem[Li et~al.(2022)Li, Wu, and Lan]{li2022stochastic}
Li, T., Wu, F., and Lan, G.
\newblock Stochastic first-order methods for average-reward {M}arkov decision
  processes.
\newblock \emph{arXiv preprint arXiv:2205.05800}, 2022.

\bibitem[Lim \& Autef(2019)Lim and Autef]{lim2019kernel}
Lim, S.~H. and Autef, A.
\newblock Kernel-based reinforcement learning in robust {M}arkov decision
  processes.
\newblock In \emph{Proc. International Conference on Machine Learning (ICML)},
  pp.\  3973--3981. PMLR, 2019.

\bibitem[Lim et~al.(2013)Lim, Xu, and Mannor]{lim2013reinforcement}
Lim, S.~H., Xu, H., and Mannor, S.
\newblock Reinforcement learning in robust {M}arkov decision processes.
\newblock In \emph{Proc. Advances in Neural Information Processing Systems
  (NIPS)}, pp.\  701--709, 2013.

\bibitem[Liu et~al.(2022)Liu, Bai, Blanchet, Dong, Xu, Zhou, and
  Zhou]{liu2022distributionally}
Liu, Z., Bai, Q., Blanchet, J., Dong, P., Xu, W., Zhou, Z., and Zhou, Z.
\newblock Distributionally robust {$Q$}-learning.
\newblock In \emph{Proc. International Conference on Machine Learning (ICML)},
  pp.\  13623--13643. PMLR, 2022.

\bibitem[Neufeld \& Sester(2022)Neufeld and Sester]{neufeld2022robust}
Neufeld, A. and Sester, J.
\newblock Robust ${Q}$-learning algorithm for {M}arkov decision processes under
  {W}asserstein uncertainty.
\newblock \emph{arXiv preprint arXiv:2210.00898}, 2022.

\bibitem[Nilim \& El~Ghaoui(2004)Nilim and El~Ghaoui]{nilim2004robustness}
Nilim, A. and El~Ghaoui, L.
\newblock Robustness in {{M}arkov} decision problems with uncertain transition
  matrices.
\newblock In \emph{Proc. Advances in Neural Information Processing Systems
  (NIPS)}, pp.\  839--846, 2004.

\bibitem[Panaganti \& Kalathil(2021)Panaganti and
  Kalathil]{panaganti2021sample}
Panaganti, K. and Kalathil, D.
\newblock Sample complexity of robust reinforcement learning with a generative
  model.
\newblock \emph{arXiv preprint arXiv:2112.01506}, 2021.

\bibitem[Panaganti et~al.(2022)Panaganti, Xu, Kalathil, and
  Ghavamzadeh]{panaganti2022robust}
Panaganti, K., Xu, Z., Kalathil, D., and Ghavamzadeh, M.
\newblock Robust reinforcement learning using offline data.
\newblock \emph{arXiv preprint arXiv:2208.05129}, 2022.

\bibitem[Puterman(1994)]{puterman1994markov}
Puterman, M.~L.
\newblock {M}arkov decision processes: Discrete stochastic dynamic programming,
  1994.

\bibitem[Roy et~al.(2017)Roy, Xu, and Pokutta]{roy2017reinforcement}
Roy, A., Xu, H., and Pokutta, S.
\newblock Reinforcement learning under model mismatch.
\newblock In \emph{Proc. Advances in Neural Information Processing Systems
  (NIPS)}, pp.\  3046--3055, 2017.

\bibitem[Satia \& Lave~Jr(1973)Satia and Lave~Jr]{satia1973markovian}
Satia, J.~K. and Lave~Jr, R.~E.
\newblock {M}arkovian decision processes with uncertain transition
  probabilities.
\newblock \emph{Operations Research}, 21\penalty0 (3):\penalty0 728--740, 1973.

\bibitem[Sutton \& Barto(2018)Sutton and Barto]{sutton2018reinforcement}
Sutton, R.~S. and Barto, A.~G.
\newblock \emph{Reinforcement Learning: An Introduction}.
\newblock The MIT Press, Cambridge, Massachusetts, 2018.

\bibitem[Tamar et~al.(2014)Tamar, Mannor, and Xu]{tamar2014scaling}
Tamar, A., Mannor, S., and Xu, H.
\newblock Scaling up robust {MDP}s using function approximation.
\newblock In \emph{Proc. International Conference on Machine Learning (ICML)},
  pp.\  181--189. PMLR, 2014.

\bibitem[Tessler et~al.(2019)Tessler, Efroni, and Mannor]{tessler2019action}
Tessler, C., Efroni, Y., and Mannor, S.
\newblock Action robust reinforcement learning and applications in continuous
  control.
\newblock In \emph{Proc. International Conference on Machine Learning (ICML)},
  pp.\  6215--6224. PMLR, 2019.

\bibitem[Tewari \& Bartlett(2007)Tewari and Bartlett]{tewari2007bounded}
Tewari, A. and Bartlett, P.~L.
\newblock Bounded parameter {M}arkov decision processes with average reward
  criterion.
\newblock In \emph{International Conference on Computational Learning Theory},
  pp.\  263--277. Springer, 2007.

\bibitem[Tsitsiklis \& Van~Roy(1999)Tsitsiklis and
  Van~Roy]{tsitsiklis1999average}
Tsitsiklis, J.~N. and Van~Roy, B.
\newblock Average cost temporal-difference learning.
\newblock \emph{Automatica}, 35\penalty0 (11):\penalty0 1799--1808, 1999.

\bibitem[Wan \& Sutton(2022)Wan and Sutton]{wan2022convergence}
Wan, Y. and Sutton, R.~S.
\newblock On convergence of average-reward off-policy control algorithms in
  weakly-communicating {MDPs}.
\newblock \emph{arXiv preprint arXiv:2209.15141}, 2022.

\bibitem[Wan et~al.(2021)Wan, Naik, and Sutton]{wan2021learning}
Wan, Y., Naik, A., and Sutton, R.~S.
\newblock Learning and planning in average-reward {M}arkov decision processes.
\newblock In \emph{Proc. International Conference on Machine Learning (ICML)},
  pp.\  10653--10662. PMLR, 2021.

\bibitem[Wang \& Wang(2022)Wang and Wang]{wang2022unbiased}
Wang, G. and Wang, T.
\newblock Unbiased multilevel {M}onte {C}arlo methods for intractable
  distributions: Mlmc meets mcmc.
\newblock \emph{arXiv preprint arXiv:2204.04808}, 2022.

\bibitem[Wang \& Zou(2021)Wang and Zou]{wang2021online}
Wang, Y. and Zou, S.
\newblock Online robust reinforcement learning with model uncertainty.
\newblock In \emph{Proc. Advances in Neural Information Processing Systems
  (NeurIPS)}, 2021.

\bibitem[Wang \& Zou(2022)Wang and Zou]{wang2022policy}
Wang, Y. and Zou, S.
\newblock Policy gradient method for robust reinforcement learning.
\newblock In \emph{Proc. International Conference on Machine Learning (ICML)},
  volume 162, pp.\  23484--23526. PMLR, 2022.

\bibitem[Wang et~al.(2023)Wang, Velasquez, Atia, Prater-Bennette, and
  Zou]{wang2023robust}
Wang, Y., Velasquez, A., Atia, G., Prater-Bennette, A., and Zou, S.
\newblock Robust average-reward {Markov} decision processes.
\newblock In \emph{Proc. Conference on Artificial Intelligence (AAAI)}, 2023.

\bibitem[Wiesemann et~al.(2013)Wiesemann, Kuhn, and
  Rustem]{wiesemann2013robust}
Wiesemann, W., Kuhn, D., and Rustem, B.
\newblock Robust {M}arkov decision processes.
\newblock \emph{Mathematics of Operations Research}, 38\penalty0 (1):\penalty0
  153--183, 2013.

\bibitem[Xu \& Mannor(2010)Xu and Mannor]{xu2010distributionally}
Xu, H. and Mannor, S.
\newblock Distributionally robust {M}arkov decision processes.
\newblock In \emph{Proc. Advances in Neural Information Processing Systems
  (NIPS)}, pp.\  2505--2513, 2010.

\bibitem[Yang et~al.(2021)Yang, Zhang, and Zhang]{yang2021towards}
Yang, W., Zhang, L., and Zhang, Z.
\newblock Towards theoretical understandings of robust {M}arkov decision
  processes: Sample complexity and asymptotics.
\newblock \emph{arXiv preprint arXiv:2105.03863}, 2021.

\bibitem[Yu \& Bertsekas(2009)Yu and Bertsekas]{yu2009convergence}
Yu, H. and Bertsekas, D.~P.
\newblock Convergence results for some temporal difference methods based on
  least squares.
\newblock \emph{IEEE Transactions on Automatic Control}, 54\penalty0
  (7):\penalty0 1515--1531, 2009.

\bibitem[Yu \& Xu(2015)Yu and Xu]{yu2015distributionally}
Yu, P. and Xu, H.
\newblock Distributionally robust counterpart in {M}arkov decision processes.
\newblock \emph{IEEE Transactions on Automatic Control}, 61\penalty0
  (9):\penalty0 2538--2543, 2015.

\bibitem[Zhang et~al.(2021{\natexlab{a}})Zhang, Wan, Sutton, and
  Whiteson]{zhang2021average}
Zhang, S., Wan, Y., Sutton, R.~S., and Whiteson, S.
\newblock Average-reward off-policy policy evaluation with function
  approximation.
\newblock In \emph{Proc. International Conference on Machine Learning (ICML)},
  pp.\  12578--12588. PMLR, 2021{\natexlab{a}}.

\bibitem[Zhang et~al.(2021{\natexlab{b}})Zhang, Zhang, and
  Maguluri]{zhang2021finite}
Zhang, S., Zhang, Z., and Maguluri, S.~T.
\newblock Finite sample analysis of average-reward {TD} learning and
  {$Q$}-learning.
\newblock In \emph{Proc. Advances in Neural Information Processing Systems
  (NeurIPS)}, volume~34, pp.\  1230--1242, 2021{\natexlab{b}}.

\bibitem[Zhang \& Ross(2021)Zhang and Ross]{zhang2021policy}
Zhang, Y. and Ross, K.~W.
\newblock On-policy deep reinforcement learning for the average-reward
  criterion.
\newblock In \emph{Proc. International Conference on Machine Learning (ICML)},
  pp.\  12535--12545. PMLR, 2021.

\bibitem[Zhou et~al.(2021)Zhou, Bai, Zhou, Qiu, Blanchet, and
  Glynn]{zhou2021finite}
Zhou, Z., Bai, Q., Zhou, Z., Qiu, L., Blanchet, J., and Glynn, P.
\newblock Finite-sample regret bound for distributionally robust offline
  tabular reinforcement learning.
\newblock In \emph{Proc. International Conference on Artifical Intelligence and
  Statistics (AISTATS)}, pp.\  3331--3339. PMLR, 2021.

\end{thebibliography}
\bibliographystyle{icml2023}

\newpage
\appendix
\onecolumn

\section{Proof of \Cref{lemma:1}}
We construct the following example. 
\begin{example}\label{ex}
Consider an MDP with 3 states (1,2,3) and only one action $a$, and set a $(s,a)$-rectangular uncertainty set $\cp=\cp^a_1\bigotimes\cp^a_2\bigotimes\cp^a_3$ where $\mathcal{P}^a_1=\{{\kp^a_1}_1,{\kp^a_1}_2 \}$, $\mathcal{P}^a_2=\{(0,0,1)^\top\}$ and $\mathcal{P}^a_3=\{(0,1,0)^\top\}$, where ${\kp^a_1}_1=(0,1,0)^\top,{\kp^a_1}_2=(0,0,1)^\top$. Hence, the uncertainty set contains two transition kernels $\cp=\{\kp_1,\kp_2\}$. The reward of each state is set to be $r=(r_1,r_2,r_3)$. The only stationary policy $\pi$ in this example is $\pi(i)=a, \forall i$.
\end{example}
Note that this robust MDP is a unichain and hence satisfies Assumption \ref{ass:sameg} with $g^\pi_{\kp_1}(1)=g^\pi_{\kp_1}(2)=g^\pi_{\kp_1}(3), g^\pi_{\kp_2}(1)=g^\pi_{\kp_2}(2)=g^\pi_{\kp_2}(3)$. 

Under both transition kernels $\kp_1,\kp_2$, the average-reward are identical: $g^\pi_{\kp_1}=g^\pi_{\kp_2}=0.5r_2+0.5r_3$. Hence, both $\kp_1,\kp_2$ are the worst-case transition kernels.

 According to Section A.5 of \cite{puterman1994markov}, the relative value functions w.r.t. $\kp_1,\kp_2$ can be computed as $$V^\pi_{\kp_1}=\bigg(r_1-\frac{1}{4}r_2-\frac{3}{4}r_3,\frac{1}{4}r_2-\frac{1}{4}r_3,-\frac{1}{4}r_2+\frac{1}{4}r_3\bigg)^\top,$$ $$V^\pi_{\kp_2}=\bigg(r_1-\frac{3}{4}r_2-\frac{1}{4}r_3,\frac{1}{4}r_2-\frac{1}{4}r_3,-\frac{1}{4}r_2+\frac{1}{4}r_3\bigg)^\top.$$ When $r_3>r_2$, only $V^\pi_{\kp_1}$ is the solution to \eqref{eq:bellman}; and when $r_2>r_3$, only $V^\pi_{\kp_2}$ is the solution to \eqref{eq:bellman}. Hence, this proves \Cref{lemma:1} and implies that not any relative value function w.r.t. a worst-case transition kernel is a solution to \eqref{eq:bellman}.

\section{Robust RVI TD Method for Policy Evaluation}
We define the following notation:
\begin{align}
    r_\pi(s)&\triangleq \sum_a \pi(a|s)r(s,a),\nn\\
    \sigma_{\cp_s}(V) &\triangleq \sum_a \pi(a|s) \sigma_{\cp^a_s}(V),\nn\\
    \sigma_\cp(V)&\triangleq (\sigma_{\cp_{s_1}}(V),\sigma_{\cp_{s_2}}(V),...,\sigma_{\cp_{s_{|\mcs|}}}(V))\in\mathbb{R}^{|\mcs|}. \nn
\end{align}

\subsection{Proof of \Cref{thm:robust Bellman}}
\begin{theorem}[Restatement of \Cref{thm:robust Bellman}]
If $(g,V)$ is a solution to the robust Bellman equation
\begin{align}\label{eq:32}
    V(s)=\sum_a \pi(a|s) (r(s,a)-g+\sigma_{\cp^a_s}(V)), \forall s,
\end{align} 
then
1) $g=g^\pi_\cp$ \cite{wang2023robust};
2) $\kp_V\in\Omega^\pi_g$; 
3) $V=V^\pi_{\kp_V}+ce$ for some $c\in\mathbb R$.
\end{theorem}
\begin{proof}
1). The robust Bellman equation in \eqref{eq:32} can be rewritten as 
\begin{align} 
    g+V(s)-r_\pi(s)=\sigma_{\cp_s}(V), \forall s\in\mcs. 
\end{align}
From the definition, it follows that
\begin{align}
    \sigma_{\cp_s}(V)= \sum_a \pi(a|s) \min_{\kp^a_s\in\cp^a_s} \kp^a_s V.
\end{align}
Hence, for any transition kernel $\kp=(\kp^a_s) \in \bigotimes_{s,a} \cp^a_s$, 
\begin{align} 
    g+V(s)-r_\pi(s)-\sum_a\pi(a|s)\kp^a_s V \leq  0, \forall s.
\end{align}
It can be further rewritten in matrix form as:
\begin{align}\label{eq:36}
    ge\leq r_\pi + (\kp^\pi-I) V,
\end{align}
where $\kp^\pi$ is the state transition matrix induced by $\pi$ and $\kp$, i.e., the $s$-th row of $\kp^\pi$ is 
\begin{align}
    \sum_a \pi(a|s) \kp^a_s. 
\end{align}
Note that $\kp^\pi$ has non-negative components since it is a transition matrix. Multiplying by $\kp^\pi$ on both sides, we have that 
\begin{align}\label{eq:38}
    \kp^\pi ge =ge&\leq\kp^\pi r_\pi +\kp^\pi(\kp^\pi-I) V,\nn\\
    ge&\leq (\kp^\pi)^2 r_\pi +(\kp^\pi)^2(\kp^\pi-I) V,\nn\\
    ...\nn\\
    ge&\leq (\kp^\pi)^{n-1} r_\pi +(\kp^\pi)^{n-1}(\kp^\pi-I) V.
\end{align}
Now, by summing up all these inequalities in \eqref{eq:36} and \eqref{eq:38}, we have that 
\begin{align}
    nge\leq \sum^{n-1}_{i=0} (\kp^\pi)^i r_\pi+ ((\kp^\pi)^n-I)V,
\end{align}
and hence, 
\begin{align}
    ge\leq \frac{\sum^{n-1}_{i=0} (\kp^\pi)^i r_\pi}{n}+ \frac{((\kp^\pi)^n-I)V}{n}. 
\end{align}
Let $n\to\infty$, and we have that 
\begin{align}
    ge&\leq \lim_{n\to\infty} \frac{\sum^{n-1}_{i=0} (\kp^\pi)^i r_\pi}{n}+ \lim_{n\to\infty} \frac{((\kp^\pi)^n-I)V}{n}\nn\\
    &=g^\pi_{\kp}e, 
\end{align}
where the last inequality is from the definition of $g^\pi_\kp$ and the fact that $\lim_{n\to\infty} \frac{((\kp^\pi)^n-I)V}{n}=0$.  Hence, $g\leq g^\pi_{\kp}$ for any $\kp\in\bigotimes_{s,a}\cp^a_s$.

Consider the worst-case transition kernel $\kp_V$ of $V$. The robust Bellman equation can be equivalently rewritten as 
\begin{align}
    ge=r_\pi-V+\kp_V^\pi V. 
\end{align}
This means that $(g,V)$ is a solution to the non-robust Bellman equation for transition kernel $\kp_V$ and policy $\pi$:
\begin{align}
    xe=r_\pi-Y+\kp_V^\pi Y.
\end{align}
Thus, by Thm 8.2.6 from \cite{puterman1994markov}, 
\begin{align}
    g&=g^\pi_{\kp_V},\\
    V&=V^\pi_{\kp_V}+ce, \text{ for some } c\in\mathbb{R}. 
\end{align}
However, note that 
\begin{align}
    g^\pi_{\kp_V}=g\leq g^\pi_\cp=\min_{\kp\in\cp} g^\pi_\kp \leq g^\pi_{\kp_V},
\end{align}
thus, 
\begin{align}\label{eq:47}
    g^\pi_{\kp_V}=g= g^\pi_\cp. 
\end{align}

2). From \eqref{eq:47},
\begin{align}
    g^\pi_{\kp_V}= g^\pi_\cp\:.
\end{align}
It then follows from the definition of $\Omega^\pi_g$ that $\kp_V\in\Omega^\pi_g$. 

3). Since $(g,V)$ is a solution to the non-robust Bellman equation 
\begin{align}
    xe=r_\pi-Y+\kp_V^\pi Y,
\end{align}
the claim then follows from Theorem 8.2.6 in \cite{puterman1994markov}. 
\end{proof}

\subsection{Proof of \Cref{thm:td converge}}
\begin{theorem}(Restatement of Theorem \ref{thm:td converge})\label{thm:A2}
Under Assumptions \ref{ass:sameg},\ref{ass:f},\ref{ass:step-size}, $(f(V_n),V_n)$ converges to a (possible sample path dependent) solution to \eqref{eq:bellman} a.s..
\end{theorem}
We first show the stability of the robust RVI TD algorithm in the following lemma.
\begin{lemma}\label{Thm:bounded}
Algorithm \ref{alg:mfsTD} remains bounded during the update, i.e., 
\begin{align}
    \sup_n \|V_n\|<\infty, a.s..
\end{align}
\end{lemma}
\begin{proof}
Denote by
\begin{align}
    h(V)\triangleq r_\pi + \sigma_\cp(V)-f(V)e-V. 
\end{align}
Then the update of robust RVI TD can be rewritten as 
\begin{align}
    V_{n+1}=V_n+\alpha_n (h(V_n)+M_{n+1}),
\end{align}
where $M_{n+1}\triangleq \hat{\mathbf{T}}V_n-r_\pi- \sigma_\cp(V)$ is the noise term. 

Further, define the limit function $h_\infty$: 
\begin{align}
    h_\infty(V)\triangleq\lim_{c\to\infty} \frac{h(cV)}{c}.
\end{align}
Then, from $\sigma_{\cp^a_s}(cV)=c\sigma_{\cp^a_s}(V)$ and $f(cV)=cf(V)$, it follows that
\begin{align}
   h_\infty(V)=\lim_{c\to\infty}\frac{r_\pi}{c}+\sigma_\cp(V)-f(V)e-V=\sigma_\cp(V)-f(V)e-V.
\end{align}

According to Section 2.1 and Section 3.2 of \cite{borkar2009stochastic}, it suffices to verify the following assumptions: 

(1). $h$ is Lipschitz;

(2). Stepsize $\alpha_n$ satisfies Assumption \ref{ass:step-size};

(3). Denoting by $\mathcal{F}_n$ the $\sigma$-algebra generated by $V_0,M_1,...,M_n$, then $\mE[M_{n+1}|\mathcal{F}_n]=0$, $\mE[\|M_{n+1}\|^2|\mathcal{F}_n]\leq K(1+\|V_n\|^2)$ for some constant $K>0$. 

(4). $h_\infty$ has the origin as its unique globally asymptotically stable equilibrium. 

First, note that
\begin{align}
    \|h(V_1)-h(V_2)\|&=\max_s \left|\sum_a \pi(a|s)(\sigma_{\cp^a_s}(V_1)-\sigma_{\cp^a_s}(V_2))-(f(V_1)-f(V_2))-(V_1(s)-V_2(s) )\right|\nn\\
    &\leq \max_s \left\{\left|\sum_a \pi(a|s)(\sigma_{\cp^a_s}(V_1)-\sigma_{\cp^a_s}(V_2))\right|+\left|(f(V_1)-f(V_2))\right|+\left|(V_1(s)-V_2(s) )\right|\right\}\nn\\
    &\leq (2+L_f)\|V_1-V_2\|,
\end{align}
where the last inequality follows from the fact that the support function $\sigma_\cp(\cdot)$ is $1$-Lipschitz and the assumptions on $f$ in Assumption \ref{ass:f}. Thus, $h$ is Lipschitz, which verifies (1).   

It is straightforward  that (3) is satisfied if $\mE[\hat{\mathbf{T}}V_n|\mathcal{F}_n]=r_\pi+\sigma_\cp(V_n)$ and $\text{Var}[\hat{\mathbf{T}}V_n|\mathcal{F}_n]\leq K(1+\|V_n\|^2)$. As discussed in \Cref{sec:td}, we assume the existence of an unbiased estimator $\hat{\mathbf{T}}$ with bounded variance here, and we will construct the estimator in \Cref{sec:case}.

Then, it suffices to verify condition (4), i.e., to show that the ODE 
\begin{align}\label{eq:stable eq}
    \dot{x}(t)=h_\infty(x(t))
\end{align}
has $0$ as its unique globally asymptotically stable equilibrium. 

Define an operator ${\mathbf T}_0(V)(s)\triangleq\sum_a \pi(a|s)\sigma_{\cp^a_s}(V)$. Then, any equilibrium $W$ of \eqref{eq:stable eq} satisfies
\begin{align}
    {\mathbf T}_0(W)-f(W)e-W=0.
\end{align}
This equation can be further rewritten as a set of equations: 
\begin{equation}\label{eq:stable eq2}
\left\{
\begin{aligned}
W=&{\mathbf T}_0(W)-ge, \\
g=&f(W).
\end{aligned}
\right.
\end{equation}
The equation in \eqref{eq:stable eq2} is the robust Bellman equation for a zero-reward robust MDP. Hence, from \Cref{thm:robust Bellman}, any solution $(g,W)$ to \eqref{eq:stable eq2} satisfies
\begin{align}\label{eq:35}
    g=g^\pi_\cp,
    W=V^\pi_\kp+ce,
\end{align}
where $V^\pi_\kp$ is the relative value function w.r.t. some worst-case transition kernel $\kp$ (i.e., $g^\pi_\kp=\min_{\kp\in\cp} g^\pi_\kp$), and some $c\in\mathbb{R}$. 

Hence, any equilibrium of \eqref{eq:stable eq} satisfies 
\begin{align}\label{eq:77}
    W=V^\pi_\kp+ce, f(W)=g^\pi_\cp. 
\end{align}
However, note that this robust Bellman equation is for a zero-reward robust MDP, hence for any $\kp$, 
\begin{align}
    g^\pi_\kp=\lim_{T\to\infty} \mE_\kp\left[\sum^{T-1}_{t=0} \frac{r_t}{T} \right]=0,\\
    V^\pi_\kp=\mE_\kp\left[ \sum^\infty_{t=0} (r_t-g^\pi_\kp)\right]=0, 
\end{align}
thus $g^\pi_\cp=0$ and $W=ce$ for some $c\in\mathbb{R}$. From \eqref{eq:77}, it follows that  $f(W)=f(ce)=0$, for any equilibrium $W$. From Assumption \ref{ass:f}, we have that $f(ce)=cf(e)=c=0$. This further implies that
\begin{align}
    W=V^\pi_\cp+ce=0. 
\end{align}
Thus, the only equilibrium of \eqref{eq:stable eq} is $0$. 

We then show that $0$ is globally asymptotically stable.  Recall that the zero-reward robust Bellman operator 
\begin{align}
    {\mathbf T}_0V(s)=\sum_a \pi(a|s)(\sigma_{\cp^a_s}(V)).
\end{align}
We further introduce two operators:
\begin{align}
    {\mathbf T}'_0V&\triangleq {\mathbf T}_0V-f(V)e,\\
    \tilde{\mathbf T}_0V&\triangleq {\mathbf T}_0V-g^\pi_\cp e.
\end{align}
Note that in the zero-reward robust MDP, $g^\pi_\cp=0$ and $\tilde{\mathbf T}_0={\mathbf T}_0$, but we introduce this notation for future use. 

Consider the ODEs w.r.t. these two operators: 
\begin{align}
    \dot{x}&={\mathbf T}_0'x-x,\label{eq:T'}\\
    \dot{y}&=\tilde{\mathbf T}_0y-y\label{eq:TT}. 
\end{align}
First, it can be easily shown that both ${\mathbf T}'_0$ and $\tilde{\mathbf T}_0$ are Lipschitz with constants $1+L_f$ and $1$, respectively. Hence, both two ODEs are well-posed. Also, it can be seen that \eqref{eq:T'} is the same as the ODE in \eqref{eq:stable eq}. 

Since the second equation \eqref{eq:TT} is a non-expansion (Lipschitz with parameter no larger than $1$), Theorem 3.1 of \cite{borkar1997analog} implies that any solution $y(t)$ to \eqref{eq:TT} converges to the set of
equilibrium points, i.e., 
\begin{align}
    y(t)\to \left\{W: W=\tilde{\mathbf T}_0W \right\}, a.s..
\end{align}
Similar to the discussion for ${\mathbf T}_0$, our \Cref{thm:robust Bellman} implies that the set of equilibrium points of \eqref{eq:TT} is $\{ W=ce: c\in\mathbb{R} \}$. Hence, for any solution $y(t)$ to \eqref{eq:TT}, $y(t)\to ce$ for some constant $k$ that may depend on the initial value of $y(t)$. 

Now, consider the solution $x(t)$ to \eqref{eq:T'}. According to Lemma \ref{lemma:x=y+r} (note that ${\mathbf T}_0$ here is a special case of ${\mathbf T}$ in Lemma \ref{lemma:x=y+r} with $r=0$), if the solutions $x(t),y(t)$ have the same initial value $x(0)=y(0)$, then 
\begin{align}
    x(t)=y(t)+r(t)e,
\end{align}
where $r(t)$ is a solution to $\dot{r}(t)=-r(t)+g^\pi_\cp-f(y(t)), r(0)=0$. 

Note that the solution $r(t)$ with $r(0)=0$ can be written as  
\begin{align}\label{eq:71}
    r(t)=\int^t_0 e^{-(t-s)}(g^\pi_\cp-f(y(s))) ds
\end{align}
by variation of constants formula \cite{abounadi2001learning}. If we denote the limit of $y(t)$ by $y^*=ce$, then $\lim_{t\to\infty} r(t)=g^\pi_\cp -f(y^*)$ (Lemma B.4 in \cite{wan2021learning}, Theorem 3.4 in \cite{abounadi2001learning}). Hence,  $x(t)=y(t)+r(t)e$ converges to $y^*+(g^\pi_\cp-f(y^*))e$, i.e., 
\begin{align}
    x(t)\to ce-f(ce)e=0.
\end{align}
Hence, any solution $x(t)$ to \eqref{eq:T'} converges to $0$, which is its unique equilibrium. 
This thus implies that 0 is the unique globally asymptotically stable equilibrium. Together with Theorem 3.7 in \cite{borkar2009stochastic}, it further implies the boundedness of $V_n$, which completes the proof. 
\end{proof}

We can readily prove \Cref{thm:A2}. 
\begin{proof}
In Lemma \ref{Thm:bounded}, we have shown that \begin{align}
    \sup_n \|V_n\|<\infty, a.s..
\end{align}
Thus, we have verified that conditions (A1-A3) and (A5) in \cite{borkar2009stochastic} are satisfied. Lemma 2.1 in \cite{borkar2009stochastic} thus implies that it suffices to study the solution to the ODE $\dot{x}(t)=h(x(t))$.

For the robust Bellman operator $\mathbf TV=r_\pi+\sigma_\cp(V)$, define 
\begin{align}
    {\mathbf T}'V&\triangleq {\mathbf T}V-f(V)e,\\
    \tilde{\mathbf T}V&\triangleq {\mathbf T}V-g^\pi_\cp e.
\end{align}
From Lemma \ref{lemma:x=y+r}, we know that if $x(t),y(t)$ are the solutions to equations 
\begin{align}
    \dot{x}&={\mathbf T}'x-x,\label{eq:Tx}\\
    \dot{y}&=\tilde{\mathbf T}y-y,\label{eq:Ty}
\end{align}
with the same initial value $x(0)=y(0)$, then 
\begin{align}
    x(t)=y(t)+r(t)e,
\end{align}
where $r(t)$ satisfies 
\begin{align}
    \dot{r}(t)=-r(t)+g^\pi_\cp-f(y(t)), r(0)=0. 
\end{align}
Thus, by the variation of constants formula, 
\begin{align}
r(t)=\int^t_0 e^{-(t-s)}(g^\pi_\cp-f(y(s))) ds. 
\end{align}
Note that $\tilde{\mathbf T}$ is also non-expansive, hence $y(t)$ converges to some equilibrium of \eqref{eq:Ty} (Theorem 3.1 of \cite{borkar1997analog}). The set of equilibrium points of \eqref{eq:Ty} can be characterized as
\begin{align}
    \left\{W:\tilde{\mathbf T}W=W \right\}=\left\{W:W=TW-g^\pi_\cp e \right\}=\left\{W: W(s)=\sum_a \pi(a|s)(r(s,a)-g^\pi_\cp+\sigma_{\cp^a_s}(W)), \forall s\in\mcs\right\}.
\end{align}
From \Cref{thm:robust Bellman}, any equilibrium of \eqref{eq:Ty} can be rewritten as
\begin{align}
    W=V^\pi_\kp+ce, \text{ for some }\kp\in\Omega^\pi_g,c\in\mathbb{R}.
\end{align}
Thus, $y(t)$ converges to an equilibrium denoted by $y^*$:
\begin{align}
    y(t)\to y^*\triangleq V^\pi_{\kp^*}+c^*e, \text{ for some }\kp^*\in\Omega^\pi_g, c^*\in\mathbb{R}. 
\end{align}
Similar to Lemma \ref{Thm:bounded}, it can be showed that $r(t)\to g^\pi_\cp-f(y^*)$ (Lemma B.4 in \cite{wan2021learning}, Theorem 3.4 in \cite{abounadi2001learning}). This further implies that 
\begin{align}\label{eq:84}
    x(t)\to y^*+(g^\pi_\cp-f(y^*))e= V^\pi_{\kp^*}+(c^*+g^\pi_\cp-f(y^*))e,
\end{align}
and we denote $m^*=c^*+g^\pi_\cp-f(y^*)$. 
Moreover, since $f$ is continuous (because it is Lipschitz), we have that 
\begin{align}
    f(x(t))&\to f(V^\pi_{\kp^*}+(c^*+g^\pi_\cp-f(y^*))e)\nn\\
    &=f(V^\pi_{\kp^*})+c^*+g^\pi_\cp-f(y^*)\nn\\
    &=f(V^\pi_{\kp^*})+c^*+g^\pi_\cp-f(V^\pi_{\kp^*}+c^*e)\nn\\
    &=f(V^\pi_{\kp^*})+c^*+g^\pi_\cp-f(V^\pi_{\kp^*})-c^*\nn\\
    &=g^\pi_\cp. 
\end{align}
Hence, we show that 
\begin{align}
    x(t)&\to V^\pi_{\kp^*}+m^*e,\\
     f(x(t))&\to g^\pi_\cp.
\end{align}
Following Lemma 2.1 from \cite{borkar2009stochastic}, we conclude that a.s., 
\begin{align}\label{eq:86}
    V_n&\to V^\pi_{\kp^*}+m^*e,\\
     f(V_n)&\to g^\pi_\cp,
\end{align}
which completes the proof. 
\end{proof}

\section{Robust RVI Q-Learning}
\subsection{Proof of \Cref{thm:optimal robust Bellman}}
Part of the following theorem is proved in \cite{wang2023robust}, but we include the proof for completeness. 
\begin{theorem}[Restatement of \Cref{thm:optimal robust Bellman}]
If $(g,Q)$ is a solution to the optimal robust Bellman equation
\begin{align}\label{eq:90}
    Q(s,a)=r(s,a)-g+\sigma_{\cp^a_s}(V_Q), \forall s,a ,
\end{align} 
then
1)  $g=g^*_\cp$ \cite{wang2023robust};
2) the greedy policy w.r.t. $Q$: $\pi_Q(s)=\arg\max_a Q(s,a)$ is an optimal robust policy \cite{wang2023robust};
3) $V_Q=V^{\pi_Q}_\kp+ce$ for some $\kp\in\Omega^{\pi_Q}_g, c\in\mathbb{R}$. 
\end{theorem}
\begin{proof}
    Taking the maximum on both sides of \eqref{eq:90} w.r.t. $a$, we have that
\begin{align}\label{eq:orbq:max}
    \max_a Q(s,a)=\max_a \{r(s,a)-g+\sigma_{\cp^a_s}(V_Q) \}, \forall s\in\mcs. 
\end{align}
This is equivalent to
\begin{align}
    V_Q(s)=\max_a \{r(s,a)-g+\sigma_{\cp^a_s}(V_Q) \}, \forall s\in\mcs. 
\end{align}
By Theorem 7 in \cite{wang2023robust}, we can show that $g=g^*_\cp$, which proves claim (1). 

Recall that $V_Q(s)=\max_a Q(s,a)$. It can be also written as
\begin{align}
    V_Q(s)= \sum_a \pi_Q(a|s) Q(s,a).
\end{align}
Here, we slightly abuse the notation of $\pi_Q$, and use $\pi_Q(s)$ and $\pi_Q(a|s)$ interchangeably.

Then, the optimal robust Bellman equation in \eqref{eq:orbq:max} can be rewritten as 
\begin{align}
    Q(s,\pi_Q(s))=r(s,\pi_Q(s))-g+\sigma_{\cp^{\pi_Q(s)}_s}\bigg( \sum_a \pi_Q(a|\cdot) Q(\cdot,a)\bigg).
\end{align}
Moreover, if we denote by $W(s)=Q(s,a)=Q(s,\pi_Q(s))=\max_a Q(s,a)$, then the equation above is equivalent to 
\begin{align}
    W(s)=\sum_a \pi_Q(a|s) (r(s,a)-g+\sigma_{\cp^a_s}(W)).
\end{align}
Therefore, $(W,g)$ is a solution to the robust Bellman equation for the policy $\pi_Q$ in \Cref{thm:robust Bellman}. By \Cref{thm:robust Bellman}, we have that
\begin{align}
    g&=g^{\pi_Q}_\cp, \\
    W&=V^{\pi_Q}_\kp+ce,
\end{align}
for some $\kp\in\Omega_g^{\pi_Q}$ and $c\in\mathbb{R}$. 

Combining this with the claim (1) implies that $\pi_Q$ is an optimal robust policy. Claims (2) and (3)  are thus proved. 
\end{proof}

\subsection{Proof of \Cref{thm:Q general converge}}
\begin{lemma}
If  $\hat{\mathbf H}$ satisfies that for any $Q, s\in\mcs,a\in\mca$,  $\mE[\hat{\mathbf H}Q(s,a)]=\mathbf HQ(s,a) $ and  $\text{Var}(\hat{\mathbf H}Q(s,a))\leq C(1+\|Q\|^2)$ for some constant $C$, then under Assumptions \ref{ass:f}, \ref{ass:g*same} and \ref{ass:step-size}, Algorithm \ref{alg:Q} remains bounded during the update almost surely, i.e., 
\begin{align}
    \sup_n \|Q_n\|<\infty, a.s..
\end{align}
\end{lemma}
\begin{proof}
Denote by
\begin{align}
    h(Q)\triangleq r_\pi + \sigma_\cp(V_Q)-f(Q)e-Q. 
\end{align}
Then, the update of robust RVI Q-learning can be rewritten as 
\begin{align}
    Q_{n+1}=Q_n+\alpha_n (h(Q_n)+M_{n+1}),
\end{align}
where $M_{n+1}\triangleq \hat{\mathbf{H}}Q_n-r_\pi- \sigma_\cp(V_Q)$ is the noise term. 

Further, define the limit function $h_\infty$: 
\begin{align}
    h_\infty(Q)\triangleq\lim_{c\to\infty} \frac{h(cQ)}{c}.
\end{align}
Then, note that $\sigma_{\cp^a_s}(V_{cQ})=\sigma_{\cp^a_s}(cV_{Q})=c\sigma_{\cp^a_s}(V_Q)$ for $c>0$ and $f(cQ)=cf(Q)$. It then follows that
\begin{align}
   h_\infty(Q)=\lim_{c\to\infty}\frac{r_\pi}{c}+\sigma_\cp(V_Q)-f(Q)e-Q=\sigma_\cp(V_Q)-f(Q)e-Q.
\end{align}

Similar to the proof of \Cref{Thm:bounded}, it suffices to verify the following conditions: 

(1). $h$ is Lipschitz;

(2). Stepsize $\alpha_n$ satisfies Assumption \ref{ass:step-size};

(3). $\mE[M_{n+1}|\mathcal{F}_n]=0$, and $\mE[\|M_{n+1}\|^2|\mathcal{F}_n]\leq K(1+\|Q_n\|^2)$ for some constant $K$. 

(4). $h_\infty$ has the origin as its unique globally asymptotically stable equilibrium. 

Clearly, (2) and (3) can be verified similarly to \Cref{Thm:bounded}. We then verify (1) and (4). 

Firstly, it can be shown that 
\begin{align}
    |h(Q_1)(s,a)-h(Q_2)(s,a)|&=\big| \sigma_{\cp^a_s}(V_{Q_1})-f(Q_1)-Q_1(s,a)-\sigma_{\cp^a_s}(V_{Q_2})-f(Q_2)-Q_2(s,a)\big|\nn\\
    &\leq \big| \sigma_{\cp^a_s}(V_{Q_1})-\sigma_{\cp^a_s}(V_{Q_2})\big|+ |f(Q_1)-f(Q_2)|+ |Q_1(s,a)-Q_2(s,a)|\nn\\
    &\leq \|V_{Q_1}-V_{Q_2}\| +L_f\|Q_1-Q_2\|+\|Q_1-Q_2\|\nn\\
    &\leq (2+L_f)\|Q_1-Q_2\|,
\end{align}
where the last inequality is from the fact that $\|V_{Q_1}-V_{Q_2}\|\leq \|Q_1-Q_2\|$. This implies that $h$ is Lipschitz.

To verify (4), note that the stability equation is
    \begin{align}\label{eq:103}
        \dot{X}(t)=h_\infty(X(t))=\sigma_\cp(V_X(t))-f(X(t))e-X(t),
    \end{align}
where $V_X(t)$ is a $|\mcs|$-dimensional vector with $V_X(t) (s)=\max_a X(t)(s,a).$ 

Any equilibrium $Q$ of  the stability equation \eqref{eq:103} satisfies that 
\begin{align}\label{eq:optimal equ}
    Q(s,a)=\sigma_{\cp^a_s}(V_Q)-f(Q)e,
\end{align}
which can be viewed as an optimal robust Bellman equation \eqref{eq:optimal bellman} with zero reward. Hence, by Lemma \ref{thm:optimal robust Bellman}, it implies that 
\begin{align}
    f(Q)&=g^*_\cp=0,\\
    V_Q&=V^{\pi_Q}_\kp+ce \text{ for some } \kp\in\Omega^{\pi_Q}_g, c\in\mathbb{R}. 
\end{align}
In the zero-reward MDP, we have that $V^{\pi}_\kp=0$ for any $\pi,\kp$, thus $V_Q(s)=\max_a Q(s,a)=c$ for any $s\in\mcs$. 

Note that from \eqref{eq:optimal equ}, $Q$ satisfies that 
\begin{align}
    Q(s,a)=\sigma_{\cp^a_s}(V_Q)=\sigma_{\cp^a_s} (ce)=c. 
\end{align}

Since $f(Q)=0$, it implies that
\begin{align}
    f(Q)=f(ce)=c=0.
\end{align}
Therefore,
\begin{align}
    c=0,\\
    Q=0. 
\end{align}
Thus, $0$ is the unique equilibrium of the stability equation. 

We then show that $0$ is globally asymptotically stable. Define the zero-reward optimal robust Bellman operator 
\begin{align}
    {\mathbf H}_0Q(s,a)=\sigma_{\cp^a_s}(V_Q),
\end{align}
and further introduce two operators 
\begin{align}
    {\mathbf H}'_0Q(s,a)&=\sigma_{\cp^a_s}(V_Q)-f(Q),\\
    \tilde{\mathbf H}_0Q(s,a)&=\sigma_{\cp^a_s}(V_Q)-g^*_\cp.
\end{align}

It is straightforward to verify that $\tilde{\mathbf H}_0$ is non-expansive. Hence by \cite{borkar1997analog}, the solution $y(t)$ to equation 
\begin{align}\label{eq:200}
    \dot{y}=\tilde{\mathbf H}_0y-y
\end{align}
converges to the set of equilibrium points 
\begin{align}\label{eq:201}
    \{W: W(s,a)=\sigma_{\cp^a_s}(V_W)-g^*_\cp\}, a.s..
\end{align}
This again can be viewed as an optimal robust Bellman equation with zero-reward. Hence, any equilibrium $W$ of \eqref{eq:200} satisfies 
\begin{align}
    \max_a W(s,a)=c, \forall s. 
\end{align}
This together with \eqref{eq:201} further implies that  the equilibrium $W$ of \eqref{eq:200} satisfies 
\begin{align}
    W(s,a)=\sigma_{\cp^a_s}(V_W)=\sigma_{\cp^a_s}(ce)=c,
\end{align}
and hence $y(t)$ converges to $\{ce: c\in\mathbb{R}\}$. We denote its limit by $y^*=c^*e$. 

Lemma \ref{lemma:optimal x=y+r} implies the solution $x(t)$ to the ODE $\dot{x}={\mathbf H}_0'(x)-x$ can be decomposed as $x(t)=y(t)+r(t)e$, where $r(t)$ satisfies $\dot{r}(t)=-r(t)+g^*_\cp-f(y(t)), r(0)=0$. 

Then, similar to Lemma \ref{Thm:bounded}, Lemma B.4 in \cite{wan2021learning} and Theorem 3.4 in \cite{abounadi2001learning}, it can be shown that $r(t)\to g^*_\cp-f(y(t))=-c^*$. Hence, 
\begin{align}
    x(t)\to 0,
\end{align}
which proves the asymptotic stability.


Thus, we conclude that $0$ is the unique globally asymptotically stable
equilibrium of the stability equation, which implies the boundedness of $\{Q_n\}$ together with results from Section 2.1 and 3.2 from \cite{borkar2009stochastic}. 
\end{proof}

\begin{theorem}[Restatement of \Cref{thm:Q general converge}]
The sequence $\{Q_n\}$ generated by Algorithm \ref{alg:Q} converges to a solution $Q^*$ to the optimal robust Bellman equation a.s., and $f(Q_n)$ converges to the optimal robust average-reward $g^*_\cp$ a.s.. 
\end{theorem}
\begin{proof}
According to Lemma 1 from \cite{borkar2009stochastic} and Theorem 3.5 from \cite{abounadi2001learning}, the sequence $\{Q_n\}$ converge to the same limit as the solution $x(t)$ to the ODE $\dot{x}={\mathbf H}'x-x$. Hence the proof can be completed by showing convergence of $x(t)$ and $f(x(t))$.

For the optimal robust Bellman operator, 
\begin{align}
    {\mathbf H}Q(s,a)=r(s,a)+\sigma_{\cp^a_s}(V_Q),
\end{align}
define two operators
\begin{align}
    {\mathbf H}'Q&\triangleq {\mathbf H}Q-f(Q)e,\\
    \tilde{\mathbf H}Q&\triangleq {\mathbf H}Q-g^*_\cp e.
\end{align}
From Lemma \ref{lemma:optimal x=y+r}, we know that if $x(t),y(t)$ are the solutions to equations 
\begin{align}
    \dot{x}&={\mathbf H}'x-x,\label{eq:oTx}\\
    \dot{y}&=\tilde{\mathbf H}y-y,\label{eq:oTy}
\end{align}
with the same initial value $x(0)=y(0)$, then 
\begin{align}
    x(t)=y(t)+r(t)e,
\end{align}
where $r(t)$ satisfies 
\begin{align}
    \dot{r}(t)=-r(t)+g^*_\cp-f(y(t)), r(0)=0. 
\end{align}

It can be easily verified that $\tilde{\mathbf H}$ is  non-expansive. Hence $y(t)$ converges to the set of equilibrium points of of \eqref{eq:oTy} (Theorem 3.1 of \cite{borkar1997analog}), which can be characterized as
\begin{align}
    \left\{W:\tilde{\mathbf H}W=W \right\}=\left\{W:W={\mathbf H}W-g^*_\cp e \right\}=\left\{W: W(s,a)=r(s,a)-g^*_\cp+\sigma_{\cp^a_s}(V_W), \forall s,a\right\}.
\end{align}
From Lemma \ref{thm:optimal robust Bellman}, any equilibrium $W$ satisfies 
\begin{align}
     V_W=V^{\pi_W}_\kp+ce, \text{ for some }\kp\in\Omega^{\pi_W}_g,c\in\mathbb{R},
\end{align}
and $\pi_W$ is robust optimal. We denote the limit of $y(t)$ by $W^*$. 

Similar to \eqref{eq:84} to \eqref{eq:86}, 
it can be shown that $r(t)\to g^*_\cp-f(W^*)$. This further implies that 
\begin{align}
    x(t)\to W^*+(g^*_\cp-f(W^*))e\triangleq W^*+m^*e,
\end{align}
where $m^*=g^*_\cp-f(W^*)$. Note that $W^*+m^*e$ is a solution to the optimal robust Bellman equation, hence $x(t)$ converges to a solution to \eqref{eq:optimal bellman}. 
 Moreover, since $f$ is continuous (because it is Lipschitz), we have that 
\begin{align}
    f(x(t))&\to f(W^*+m^*e)\nn\\
    &=f(W^*)+g^*_\cp-f(W^*)\nn\\
    &=g^*_\cp. 
\end{align}
This completes the proof. 
\end{proof}

\section{Case Studies for Robust RVI TD}\label{sec:case for TD}
In this section, we provide the proof of the first part of \Cref{thm:case thm}, i.e., that $\hat{\mathbf T}$ is  unbiased and has bounded variance under each uncertainty model. 

We first show a lemma, by which the problem can be reduced to investigating whether $\hat{\sigma}_{\cp^a_s}$ is unbiased and has bounded variance. 
\begin{lemma}
    If 
    \begin{align}
        \mE[\hat{\sigma}_{\cp^a_s}V]=\sigma_{\cp^a_s}(V), \forall s,a,
    \end{align}
    and moreover, there exists a constant $C$, such that
    \begin{align}
        \text{Var}(\hat{\sigma}_{\cp^a_s}V)\leq C(1+\|V\|^2), \forall s,a,
    \end{align}
    then 
    \begin{align}
\mE[\hat{\mathbf{T}} V(s)]=\mathbf T V(s), \forall s,
    \end{align}
    and
    \begin{align}
        \text{Var}(\hat{\mathbf T}V(s))\leq |\mca|C(1+\|V\|^2), \forall s.
    \end{align}
\end{lemma}
\begin{proof}
    From the definition, $\hat{\mathbf T}V(s)=\sum_a \pi(a|s)(r(s,a)+\hat{\sigma}_{\cp^a_s}V)$. Thus,
    \begin{align}
        \mE[\hat{\mathbf{T}} V(s)]&=\mE\bigg[\sum_a \pi(a|s)(r(s,a)+\hat{\sigma}_{\cp^a_s}V)\bigg]\nn\\
        &=\sum_a \pi(a|s)(r(s,a)+\mE[\hat{\sigma}_{\cp^a_s}V])\nn\\
        &=\sum_a \pi(a|s)(r(s,a)+\sigma_{\cp^a_s}(V))=\mathbf{T}V(s),
    \end{align}
which shows that $\hat{\mathbf{T}}$ is unbiased. On the other hand, we have that
\begin{align}
    \text{Var}(\hat{\mathbf T}V(s))&=\mE\bigg[ \bigg(\sum_a \pi(a|s)(r(s,a)+\hat{\sigma}_{\cp^a_s}V)- \mE\bigg[\sum_a \pi(a|s)(r(s,a)+\hat{\sigma}_{\cp^a_s}V)\bigg]\bigg)^2\bigg]\nn\\
    &=\mE\bigg[ \bigg(\sum_a \pi(a|s)(r(s,a)+\hat{\sigma}_{\cp^a_s}V)- \sum_a \pi(a|s)(r(s,a)+\mE\big[\hat{\sigma}_{\cp^a_s}V\big]\bigg)^2\bigg]\nn\\
    &=\mE\bigg[ \bigg(\sum_a \pi(a|s)(\hat{\sigma}_{\cp^a_s}V)-\mE\big[\hat{\sigma}_{\cp^a_s}V\big]\bigg)^2\bigg]\nn\\
    &\overset{(a)}{\leq} \mE\bigg[\sum_a \pi(a|s)(\hat{\sigma}_{\cp^a_s}V-\mE\big[\hat{\sigma}_{\cp^a_s}V\big])^2\bigg]\nn\\
    &=  \sum_a \pi(a|s)\mE\bigg[ (\hat{\sigma}_{\cp^a_s}V-\mE\big[\hat{\sigma}_{\cp^a_s}V^2\big])^2 \bigg]\nn\\
    &{\leq} \sum_a \pi(a|s)\text{Var}(\hat{\sigma}_{\cp^a_s}V)\nn\\
    &\leq |\mca|C(1+\|V\|^2),
\end{align}
where $(a)$ is because $(\mE[X])^2\leq \mE[X^2]$, which completes the proof. 
\end{proof}
This lemma implies that to prove \Cref{thm:case thm}, it suffices to show that $\hat{\sigma}_{\cp^a_s}$ is unbiased and has bounded variance.

\subsection{Contamination Uncertainty Set}
\begin{theorem}
$\hat{\mathbf T}$ defined in \eqref{eq:12} is unbiased and has bounded variance.   
\end{theorem}
\begin{proof}  
First, note that 
\begin{align}
    V_{n+1}(s)&= V_n(s)+\alpha_n (r(s,a)+((1-\delta) V_n(s')+\delta \min_x V_n(x)-f(V_n)-V_n(s))\nn\\
    &= V_n(s)+\alpha_n ({\mathbf T}V_n(s)-f(V_n)-V_n(s)+M_n(s)),
\end{align}
where
\begin{align}
     M_n(s)=r(s,a)+(1-\delta) V_n(s')+\delta \min_x V_n(x)-{\mathbf T}V_n(s), 
\end{align}
and 
\begin{align}
    {\mathbf T}V_n(s)=\sum_a \pi(a|s)\bigg(r(s,a)+(1-\delta)\sum_{s'}\kp^a_{s,s'}V_n(s')+\delta \min_x V_n(x)\bigg). 
\end{align}
Thus, 
\begin{align}
    \mE[M_n(s)]&=\mE\big[r(s,a)+(1-\delta) V_n(s')+\delta \min_x V_n(x) \big]-\sum_a \pi(a|s)\bigg(r(s,a)+(1-\delta)\sum_{s'}\kp^a_{s,s'}V_n(s')+\delta \min_x V_n(x)\bigg)\nn\\
    &=\sum_a \pi(a|s) \bigg( r(s,a)+(1-\delta) \sum_{s'}\kp^a_{s,s'}V_n(s')+\delta \min_x V_n(x)\bigg)\nn\\
    &\quad\quad\quad\quad\quad\quad\quad\quad\quad-\sum_a \pi(a|s)\bigg(r(s,a)+(1-\delta)\sum_{s'}\kp^a_{s,s'}V_n(s')+\delta \min_x V_n(x))\bigg)\nn\\
    &=0.
\end{align}
Hence, the operator is unbiased. 

We also have that
\begin{align}
    \mE[|M_n(s)|^2]&=\mE\bigg[ \bigg(r(s,a)+(1-\delta) V_n(s')+\delta \min_x V_n(x)-{\mathbf T}V_n(s) \bigg)^2 \bigg]\nn\\
    &\leq 2\mE\bigg[ \bigg(r(s,a)+(1-\delta) V_n(s')+\delta \min_x V_n(x)\bigg)^2 \bigg]+ 2\mE[({\mathbf T}V_n(s))^2]\nn\\
    &\overset{(a)}{\leq} 8 + 8\|V_n\|^2\nn\\
    &\leq 8(1+\|V_n\|^2),
\end{align}
where $(a)$ is from the fact that $\mE\big[ \big((1-\delta) V_n(s')+\delta \min_x V_n(x)\big)^2 \big]= \mE\big[ \big|(1-\delta) V_n(s')+\delta \min_x V_n(x)\big|^2 \big]\leq\mE\big[ \big( \big|(1-\delta) V_n(s')\big|+\big|\delta \min_x V_n(x)\big|\big)^2 \big] \leq \mE\big[ \big((1-\delta) \|V_n\|+\big(\delta \|V_n\| \big)^2 \big]\leq \|V_n\|^2$.

The proof is completed. 
\end{proof}

\subsection{Total Variation Uncertainty Set}
The estimator under the total variation uncertainty set can be written as
\begin{align}\label{eq:p_tv}
    \hat{\sigma}_{\cp^a_s}(V)=\max_{\mu\geq 0} \big(\hat{\kp}^{a,1}_{s,N+1}(V-\mu)-\delta \spa(V-\mu) \big)+\frac{\Delta_N(V)}{p_N},
\end{align}
where 
\begin{align}
    &\Delta_N(V)= \max_{\mu\geq 0}\big(\hat{\kp}^{a}_{s,N+1}(V-\mu)-\delta \spa(V-\mu) \big)\nn\\
    &\quad-\frac{1}{2}\max_{\mu\geq 0}\big(\hat{\kp}^{a,O}_{s,N+1}(V-\mu)-\delta \spa(V-\mu) \big)\nn\\
    &\quad-\frac{1}{2}\max_{\mu\geq 0}\big(\hat{\kp}^{a,E}_{s,N+1}(V-\mu)-\delta \spa(V-\mu) \big).
\end{align}

\begin{theorem}\label{thm:tv unbiased}
The estimated operator $\hat{\sigma}_{\cp^a_s}$ defined in \eqref{eq:p_tv} is unbiased, i.e., 
\begin{align}
    \mE[\hat{\sigma}_{\cp^a_s}V]=\sigma_{\cp^a_s}(V). 
\end{align}
\end{theorem}
\begin{proof}
First, denote the dual function \eqref{eq:tv dual} by $g$:
\begin{align}
g_{s,a}^V(\mu)=\kp^a_s(V-\mu)-\delta \spa(V-\mu),
\end{align}
and denote its optimal solution by $\mu^V_{s,a}$: 
\begin{align}
    \mu^V_{s,a}=\arg\max_{\mu\geq 0}\big( \kp^a_s(V-\mu)-\delta \spa(V-\mu)\big).
\end{align}
Then, the support function $\sigma_{\cp^a_s}(V)=g_{s,a}^V(\mu^V_{s,a})$. 
Similarly, define the empirical function  
\begin{align}
{g}_{s,a,N+1}^V(\mu)&=\hat{\kp}^{a}_{s,N+1}(V-\mu)-\delta \spa(V-\mu),\\
{g}_{s,a,N+1,O}^V(\mu)&=\hat{\kp}^{a,O}_{s,N+1}(V-\mu)-\delta \spa(V-\mu),\\
{g}_{s,a,N+1,E}^V(\mu)&=\hat{\kp}^{a,E}_{s,N+1}(V-\mu)-\delta \spa(V-\mu),
\end{align}
and their optimal solutions are denoted by $\mu^V_{s,a,N+1},\mu^V_{s,a,N+1,O},\mu^V_{s,a,N+1,E}$. We have that
\begin{align}
    \mE[\hat{\sigma}_{\cp^a_s}V]&=\mE\bigg[\max_{\mu\geq 0} \big(\hat{\kp}^{a,1}_{s,N+1}(V-\mu)-\delta \spa(V-\mu) \big)  +\frac{\Delta_N(V)}{p_N}\bigg]\nn\\
    &=\mE[g^V_{s,a,0}(\mu^V_{s,a,0})]+\mE\bigg[\frac{\Delta_N(V)}{p_N} \bigg]\nn\\
    &=\mE[g^V_{s,a,0}(\mu^V_{s,a,0})]+\sum^\infty_{n=0}p(N=n)\mE\bigg[\frac{\Delta_N(V)}{p_N}|N=n \bigg]\nn\\
    &=\mE[g^V_{s,a,0}(\mu^V_{s,a,0})]+\sum^\infty_{n=0}\mE[\Delta_n(V)]\nn\\
    &=\mE[g^V_{s,a,0}(\mu^V_{s,a,0})]+\sum^\infty_{n=0}\mE\bigg[g^V_{s,a,n+1}(\mu^V_{s,a,n+1})-\frac{g^V_{s,a,n+1,O}(\mu^V_{s,a,n+1,O})+g^V_{s,a,n+1,E}(\mu^V_{s,a,n+1,E})}{2}\bigg]\nn\\
    &=\mE[g^V_{s,a,0}(\mu^V_{s,a,0})]+\sum^\infty_{n=0}\mE\bigg[g^V_{s,a,n+1}(\mu^V_{s,a,n+1})-g^V_{s,a,n}(\mu^V_{s,a,n})\bigg],
\end{align}
where the last inequality is from Lemma \ref{lemma:Eo=Ee}. The last equation can be further rewritten as
\begin{align}
    \mE[\hat{\sigma}_{\cp^a_s}V]&=\mE[g^V_{s,a,0}(\mu^V_{s,a,0})]+\sum^\infty_{n=0}\mE\bigg[g^V_{s,a,n+1}(\mu^V_{s,a,n+1})-g^V_{s,a,n}(\mu^V_{s,a,n})\bigg]\nn\\
    &=\lim_{n\to\infty}\mE\bigg[g^V_{s,a,n}(\mu^V_{s,a,n})\bigg].
\end{align}
To show that $\hat{\sigma}_{\cp^a_s}$ is unbiased, it suffices to prove that 
\begin{align}
    \lim_{n\to\infty}\mE\bigg[g^V_{s,a,n}(\mu^V_{s,a,n})\bigg]=g^V_{s,a}(\mu^V_{s,a}). 
\end{align}
For any arbitrary i.i.d. samples $\{X_i\}$ and its corresponding function $g^V_{s,a,n}$, together with Lemma \ref{lemma:tv bounded}, we have that
\begin{align}\label{eq:99}
    &|g^V_{s,a,n}(\mu^V_{s,a,n})-g^V_{s,a}(\mu^V_{s,a})|\nn\\
    &=|\max_{0\leq \mu \leq V+\|V\|e}g^V_{s,a}(\mu) -\max_{0\leq \mu \leq V+\|V\|e}g^V_{s,a,n}(\mu) |\nn\\
    &\leq \max_{0\leq \mu \leq V+\|V\|e} |g^V_{s,a}(\mu)-g^V_{s,a,n}(\mu)|\nn\\
    &=\max_{0\leq \mu \leq V+\|V\|e} |\kp^a_s(V-\mu)-\delta\spa(V-\mu) - \hat{\kp}^{a}_{s,n}(V-\mu)+\delta\spa(V-\mu)|\nn\\
    &=\max_{0\leq \mu \leq V+\|V\|e} |\kp^a_s(V-\mu)-\hat{\kp}^{a}_{s,n}(V-\mu)|\nn\\
    &\leq \max_{0\leq \mu \leq V+\|V\|e} \|V-\mu\|\|\kp^a_s-\hat{\kp}^{a}_{s,n}\|_1\nn\\
    &\leq 3\|V\|\|\kp^a_s-\hat{\kp}^{a}_{s,n}\|_1. 
\end{align}
Thus, by Hoeffding's inequality and Theorem 3.7 from \cite{liu2022distributionally}, 
\begin{align}
    \mE[|g^V_{s,a,n}(\mu^V_{s,a,n})-g^V_{s,a}(\mu^V_{s,a})|]\leq 3\|V\| \frac{|\mcs|^2\sqrt{\pi}}{2^{\frac{n+1}{2}}}, 
\end{align}
which implies that 
\begin{align}
    \lim_{n\to\infty}\mE\bigg[g^V_{s,a,n}(\mu^V_{s,a,n})\bigg]=g^V_{s,a}(\mu^V_{s,a}),
\end{align}
completing the proof.
\end{proof}

\begin{theorem}\label{thm:tv var}
The estimated operator $\hat{\sigma}_{\cp^a_s}$ defined in \eqref{eq:p_tv} has bounded variance, i.e., there exists a constant $C_0$, such that 
\begin{align}
     \text{Var}(\hat{\sigma}_{\cp^a_s}V)\leq (1+18(1+2\delta)^2+2C_0)\|V\|^2.
\end{align}
\end{theorem}
\begin{proof}
Similar to Theorem \ref{thm:tv unbiased}, we have that
\begin{align}
    &\text{Var}(\hat{\sigma}_{\cp^a_s}V) \nn\\
    &=\mE[(\hat{\sigma}_{\cp^a_s}V)^2]-\sigma_{\cp^a_s}(V)^2\nn\\
    &\leq  \mE\bigg[\bigg(g^V_{s,a,0}(\mu^V_{s,a,0})+\frac{\Delta_N(V)}{p_N}\bigg)^2\bigg]+(\sigma_{\cp^a_s}(V))^2\nn\\
    &\leq 2\mE\bigg[\bigg(g^V_{s,a,0}(\mu^V_{s,a,0})\bigg)^2\bigg]+2\mE\bigg[\bigg(\frac{\Delta_N(V)}{p_N}\bigg)^2\bigg]+(\sigma_{\cp^a_s}(V))^2\nn\\
    &\leq (1+18(1+2\delta)^2)\|V\|^2+2\sum^\infty_{i=0}\frac{\mE[(\Delta_i(V))^2]}{p_i},
\end{align}
where the last inequality is from Lemma \ref{lemma:tv bounded}. 
For any $n\geq 1$, we have that
\begin{align}
    \mE[(\Delta_n(V))^2]&=\mE\bigg[ \bigg({g}^V_{s,a,n}(\mu^V_{s,a,n}) -\frac{{g}^V_{s,a,n,E}(\mu^V_{s,a,n,E})+{g}^V_{s,a,n,O}(\mu^V_{s,a,n,O})}{2} \bigg)^2\bigg]\nn\\
    &=\mE\bigg[ \bigg({g}^V_{s,a,n}(\mu^V_{s,a,n})-g^V_{s,a}(\mu^V_{s,a}) +g^V_{s,a}(\mu^V_{s,a})-\frac{{g}^V_{s,a,n,E}(\mu^V_{s,a,n,E})+{g}^V_{s,a,n,O}(\mu^V_{s,a,n,O})}{2} \bigg)^2\bigg]\nn\\
    &\leq 2\mE[({g}^V_{s,a,n}(\mu^V_{s,a,n})-g^V_{s,a}(\mu^V_{s,a}))^2]+ 2\mE\bigg[\bigg(g^V_{s,a}(\mu^V_{s,a})-\frac{{g}^V_{s,a,n,E}(\mu^V_{s,a,n,E})+{g}^V_{s,a,n,O}(\mu^V_{s,a,n,O})}{2}\bigg)^2 \bigg]\nn\\
    &\overset{(a)}{=}2\mE[({g}^V_{s,a,n}(\mu^V_{s,a,n})-g^V_{s,a}(\mu^V_{s,a}))^2]+2\mE[({g}^V_{s,a,n-1}(\mu^V_{s,a,n-1})-g^V_{s,a}(\mu^V_{s,a}))^2]\nn\\
    &\leq 18\|V\|^2 \mE[\|\kp^a_s-\hat{\kp}^{a}_{s,n}\|_1^2]+18\|V\|^2 \mE[\|\kp^a_s-\hat{\kp}^{a}_{s,n-1}\|_1^2],
\end{align}
where $(a)$ is due to Lemma \ref{lemma:Eo=Ee} and the last inequality follows a similar argument to \eqref{eq:99}. Note that $p_n=\Psi(1-\Psi)^n$ for $\Psi\in(0,0.5)$, thus similar to Theorem 3.7 of \cite{liu2022distributionally}, we can show that there exists a constant $C_0$, such that 
\begin{align}
    \sum^\infty_{i=0}\frac{\mE[(\Delta_i(V))^2]}{p_i}\leq C_0\|V\|^2. 
\end{align}
Thus, 
\begin{align}
     \text{Var}(\hat{\sigma}_{\cp^a_s}V)\leq (1+18(1+2\delta)^2)\|V\|^2+2C_0\|V\|^2=(1+18(1+2\delta)^2+2C_0)\|V\|^2\:.
\end{align}
\end{proof}

\subsection{Chi-Square Uncertainty Set}
The estimator under the Chi-square uncertainty set can be written as
\begin{align}\label{eq:p_cs}
    \hat{\sigma}_{\cp^a_s}V&= \max_{\mu\geq 0}\big(\hat{\kp}^{a,1}_{s,N+1}(V-\mu)-\sqrt{\delta \text{Var}_{\hat{\kp}^{a,1}_{s,N+1}} (V-\mu)} \big)\nn\\
    &\quad+\frac{\Delta_N(V)}{p_N},
\end{align}
where 
\begin{align}
    &\Delta_N(V)= \max_{\mu\geq 0}\big(\mE_{\hat{\kp}^{a}_{s,N+1}}[V-\mu]-\sqrt{\delta \text{Var}_{\hat{\kp}^{a}_{s,N+1}} (V-\mu)} \big) \nn\\
    &\quad-\frac{1}{2}\max_{\mu\geq 0}\big(\mE_{\hat{\kp}^{a,O}_{s,N+1}}[V-\mu]-\sqrt{\delta \text{Var}_{\hat{\kp}^{a,O}_{s,N+1}} (V-\mu)} \big)\nn\\
    &\quad-\frac{1}{2}\max_{\mu\geq 0}\big(\mE_{\hat{\kp}^{a,E}_{s,N+1}}[V-\mu]-\sqrt{\delta \text{Var}_{\hat{\kp}^{a,E}_{s,N+1}} (V-\mu)} \big).\nn 
\end{align}
\begin{theorem}\label{thm:cs unbiased}
The estimated operator defined in \eqref{eq:p_cs} is unbiased, i.e., 
\begin{align}
    \mE[\hat{\sigma}_{\cp^a_s}V]=\sigma_{\cp^a_s}(V). 
\end{align}
\end{theorem}
\begin{proof}
Denote the dual function \eqref{eq:cs dual} by $g$:
\begin{align}
g_{s,a}^V(\mu)=\kp^a_s(V-\mu)-\sqrt{\delta \text{Var}_{\kp^a_s} (V-\mu)} ,
\end{align}
and denote its optimal solution by $\mu^V_{s,a}$: 
\begin{align}
    \mu^V_{s,a}=\arg\max_{\mu\geq 0}\big( \kp^a_s(V-\mu)-\sqrt{\delta \text{Var}_{\kp^a_s} (V-\mu)} \big).
\end{align}
Then, the support function $\sigma_{\cp^a_s}(V)=g_{s,a}^V(\mu^V_{s,a})$. 
Similarly, define the empirical function  
\begin{align}
{g}_{s,a,N+1}^V(\mu)&=\hat{\kp}^{a}_{s,N+1}(V-\mu)-\sqrt{\delta \text{Var}_{\hat{\kp}^a_{s,N+1}} (V-\mu)},\\
{g}_{s,a,N+1,O}^V(\mu)&=\hat{\kp}^{a,O}_{s,N+1}(V-\mu)-\sqrt{\delta \text{Var}_{\hat{\kp}^{a,O}_{s,N+1}} (V-\mu)},\\
{g}_{s,a,N+1,E}^V(\mu)&=\hat{\kp}^{a,E}_{s,N+1}(V-\mu)-\sqrt{\delta \text{Var}_{\hat{\kp}^{a,E}_{s,N+1}} (V-\mu)},
\end{align}
and their optimal solutions are denoted by $\mu^V_{s,a,N+1},\mu^V_{s,a,N+1,O},\mu^V_{s,a,N+1,E}$. We have that
\begin{align}
    \mE[\hat{\sigma}_{\cp^a_s}V]&=\mE[g^V_{s,a,0}(\mu^V_{s,a,0})]+\mE\bigg[\frac{\Delta_N(V)}{p_N} \bigg]\nn\\
    &=\mE[g^V_{s,a,0}(\mu^V_{s,a,0})]+\sum^\infty_{n=0}p(N=n)\mE\bigg[\frac{\Delta_N(V)}{p_N}|N=n \bigg]\nn\\
    &=\mE[g^V_{s,a,0}(\mu^V_{s,a,0})]+\sum^\infty_{n=0}\mE[\Delta_n]\nn\\
    &=\mE[g^V_{s,a,0}(\mu^V_{s,a,0})]+\sum^\infty_{n=0}\mE\bigg[g^V_{s,a,n+1}(\mu^V_{s,a,n+1})-\frac{g^V_{s,a,n+1,O}(\mu^V_{s,a,n+1,O})+g^V_{s,a,n+1,E}(\mu^V_{s,a,n+1,E})}{2}\bigg]\nn\\
    &=\mE[g^V_{s,a,0}(\mu^V_{s,a,0})]+\sum^\infty_{n=0}\mE\bigg[g^V_{s,a,n+1}(\mu^V_{s,a,n+1})-g^V_{s,a,n}(\mu^V_{s,a,n})\bigg],
\end{align}
where the last inequality is from Lemma \ref{lemma:Eo=Ee}. The last equation can be further rewritten as
\begin{align}
    \mE[\hat{\sigma}_{\cp^a_s}V]&=\mE[g^V_{s,a,0}(\mu^V_{s,a,0})]+\sum^\infty_{n=0}\mE\bigg[g^V_{s,a,n+1}(\mu^V_{s,a,n+1})-g^V_{s,a,n}(\mu^V_{s,a,n})\bigg]\nn\\
    &=\lim_{n\to\infty}\mE\bigg[g^V_{s,a,n}(\mu^V_{s,a,n})\bigg].
\end{align}
To show that $\hat{\sigma}_{\cp^a_s}$ is unbiased, it suffices to prove that 
\begin{align}
    \lim_{n\to\infty}\mE\bigg[g^V_{s,a,n}(\mu^V_{s,a,n})\bigg]=g^V_{s,a}(\mu^V_{s,a}). 
\end{align}
For any arbitrary i.i.d. samples $\{X_i\}$ and its corresponding function $g^V_{s,a,n}$, together with Lemma \ref{lemma:cs bounded}, we have that
\begin{align}\label{eq:100}
    &|g^V_{s,a,n}(\mu^V_{s,a,n})-g^V_{s,a}(\mu^V_{s,a})|\nn\\
    &=|\max_{0\leq \mu \leq V+\|V\|e}g^V_{s,a}(\mu) -\max_{0\leq \mu \leq V+\|V\|e}g^V_{s,a,n}(\mu) |\nn\\
    &\leq \max_{0\leq \mu \leq V+\|V\|e} |g^V_{s,a}(\mu)-g^V_{s,a,n}(\mu)|\nn\\
    &=\max_{0\leq \mu \leq V+\|V\|e} \bigg|\kp^a_s(V-\mu) - \hat{\kp}^{a}_{s,n}(V-\mu)- \bigg(\sqrt{\delta \text{Var}_{{\kp^a_s}} (V-\mu)}-\sqrt{\delta \text{Var}_{\hat{\kp}^{a}_{s,n}} (V-\mu)}\bigg)\bigg|\nn\\
    &\leq \max_{0\leq \mu \leq V+\|V\|e} |\kp^a_s(V-\mu)-\hat{\kp}^{a}_{s,n}(V-\mu)|+\max_{0\leq \mu \leq V+\|V\|e}\bigg|\bigg(\sqrt{\delta \text{Var}_{{\kp^a_s}} (V-\mu)}-\sqrt{\delta \text{Var}_{\hat{\kp}^{a}_{s,n}} (V-\mu)}\bigg)\bigg|\nn\\
    &\overset{(a)}{\leq} \max_{0\leq \mu \leq V+\|V\|e} \|V-\mu\|\|\kp^a_s-\hat{\kp}^{a}_{s,n}\|_1+\max_{0\leq \mu \leq V+\|V\|e}  \sqrt{|\delta \text{Var}_{{\kp^a_s}} (V-\mu)-\delta \text{Var}_{\hat{\kp}^{a}_{s,n}} (V-\mu)|},
\end{align}
where $(a)$ is due to $|\sqrt{x}-\sqrt{y}|\leq \sqrt{|x-y|}$. 
Note that for any distribution $p,q\in\Delta(|\mcs|)$ and any random variable $X$,
\begin{align}
    |\text{Var}_p[X]-\text{Var}_q[X]|&=|\mE_p[X^2]-\mE_p[X]^2-\mE_q[X^2]+\mE_q[X]^2| \nn\\
    &\leq |\mE_p[X^2]-\mE_q[X^2]|+ |(\mE_p[X]+\mE_q[X])(\mE_p[X]-\mE_q[X])|\nn\\
    &\leq \sup|X^2|\|p-q\|_1+2(\sup|X|)^2\|p-q\|_1.
\end{align}
Hence, 
\begin{align}
    \sqrt{|\delta \text{Var}_{{\kp^a_s}} (V-\mu)-\delta \text{Var}_{\hat{\kp}^{a}_{s,n}} (V-\mu)|}&\leq \sqrt{3\delta\|V-\mu\|^2\|\kp^a_s-\hat{\kp}^{a}_{s,n} \|_1}.
\end{align}

Thus, by Hoeffding's inequality and Theorem 3.7 from \cite{liu2022distributionally}, 
\begin{align}
    \mE[|g^V_{s,a,n}(\mu^V_{s,a,n})-g^V_{s,a}(\mu^V_{s,a})|]\leq 3\|V\| \left( \frac{|\mcs|^2\sqrt{\pi}}{2^{\frac{n+1}{2}}}+\sqrt{\frac{3\delta|\mcs|^2\sqrt{\pi}}{2^{\frac{n+1}{2}}}} \right), 
\end{align}
which implies that 
\begin{align}
    \lim_{n\to\infty}\mE\bigg[g^V_{s,a,n}(\mu^V_{s,a,n})\bigg]=g^V_{s,a}(\mu^V_{s,a}),
\end{align}
which completes the proof.
\end{proof}

\begin{theorem}\label{thm:cs var}
The estimated operator $\hat{\sigma}_{\cp^a_s}$ defined in \eqref{eq:p_cs} has bounded variance, i.e., there exists a constant $C_0$, such that 
\begin{align}
     \text{Var}(\hat{\sigma}_{\cp^a_s}V)\leq (1+18(1+\sqrt{2\delta})^2+2C_0)\|V\|^2. 
\end{align}
\end{theorem}
\begin{proof}
We have that
\begin{align}
    &\text{Var}(\hat{\sigma}_{\cp^a_s}V)\nn\\
    &=\mE[(\hat{\sigma}_{\cp^a_s}V)^2]-\sigma_{\cp^a_s}(V)^2\nn\\
    &\leq  \mE\bigg[\bigg(g^V_{s,a,0}(\mu^V_{s,a,0})+\frac{\Delta_N(V)}{p_N}\bigg)^2\bigg]+(\sigma_{\cp^a_s}(V))^2\nn\\
    &\leq 2\mE\bigg[\bigg(g^V_{s,a,0}(\mu^V_{s,a,0})\bigg)^2\bigg]+2\mE\bigg[\bigg(\frac{\Delta_N(V)}{p_N}\bigg)^2\bigg]+(\sigma_{\cp^a_s}(V))^2\nn\\
    &\leq (1+18(1+\sqrt{2\delta})^2)\|V\|^2+2\sum^\infty_{i=0}\frac{\mE[(\Delta_i(V))^2]}{p_i},
\end{align}
where the last inequality is from Lemma \ref{lemma:cs bounded}. 
For any $n\geq 1$, we have that
\begin{align}
    \mE[(\Delta_n(V))^2]&=\mE\bigg[ \bigg({g}^V_{s,a,n}(\mu^V_{s,a,n}) -\frac{{g}^V_{s,a,n,E}(\mu^V_{s,a,n,E})+{g}^V_{s,a,n,O}(\mu^V_{s,a,n,O})}{2} \bigg)^2\bigg]\nn\\
    &=\mE\bigg[ \bigg({g}^V_{s,a,n}(\mu^V_{s,a,n})-g^V_{s,a}(\mu^V_{s,a}) +g^V_{s,a}(\mu^V_{s,a})-\frac{{g}^V_{s,a,n,E}(\mu^V_{s,a,n,E})+{g}^V_{s,a,n,O}(\mu^V_{s,a,n,O})}{2} \bigg)^2\bigg]\nn\\
    &\leq 2\mE[({g}^V_{s,a,n}(\mu^V_{s,a,n})-g^V_{s,a}(\mu^V_{s,a}))^2]+ 2\mE\bigg[\bigg(g^V_{s,a}(\mu^V_{s,a})-\frac{{g}^V_{s,a,n,E}(\mu^V_{s,a,n,E})+{g}^V_{s,a,n,O}(\mu^V_{s,a,n,O})}{2}\bigg)^2 \bigg]\nn\\
    &\overset{(a)}{=}2\mE[({g}^V_{s,a,n}(\mu^V_{s,a,n})-g^V_{s,a}(\mu^V_{s,a}))^2]+2\mE[({g}^V_{s,a,n-1}(\mu^V_{s,a,n-1})-g^V_{s,a}(\mu^V_{s,a}))^2]\nn\\
    &\leq 18(1+\sqrt{3\delta})^2\|V\|^2 \mE[\|\kp^a_s-\hat{\kp}^{a}_{s,n}\|_1^2+\|\kp^a_s-\hat{\kp}^{a}_{s,n}\|_1]\nn\\
    &\quad+18(1+\sqrt{3\delta})^2\|V\|^2 \mE[\|\kp^a_s-\hat{\kp}^{a}_{s,n-1}\|_1^2+\|\kp^a_s-\hat{\kp}^{a}_{s,n-1}\|_1],
\end{align}
where $(a)$ is due to Lemma \ref{lemma:Eo=Ee} and the last inequality follows a similar argument to \eqref{eq:100}. Note that $p_n=\Psi(1-\Psi)^n$ for $\Psi\in\big(0,1-\frac{\sqrt{2}}{2}\big)$. Thus, similar to Theorem 3.7 of \cite{liu2022distributionally}, we can show that there exists a constant $C_0$, such that 
\begin{align}
    \sum^\infty_{i=0}\frac{\mE[(\Delta_i(V))^2]}{p_i}\leq C_0\|V\|^2. 
\end{align}
Thus, 
\begin{align}
     \text{Var}(\hat{\sigma}_{\cp^a_s}V)\leq (1+18(1+\sqrt{2\delta})^2)\|V\|^2+2C_0\|V\|^2=(1+18(1+\sqrt{2\delta})^2+2C_0)\|V\|^2.
\end{align}
\end{proof}

\subsection{KL-Divergence Uncertainty Sets}
The estimator under the KL-Divergence uncertainty set can be written as
\begin{align*}
    \hat{\sigma}_{\cp^a_s}V&\triangleq -\min_{\alpha\geq 0} \bigg(\delta\alpha+\alpha\log \left( e^{\frac{-V(s'_1)}{\alpha}}\right) \bigg)+\frac{\Delta_N(V)}{p_N},
\end{align*}
where 
\begin{align}
    &\Delta_N(V)= -\min_{\alpha\geq 0} \left(\delta\alpha+\alpha\log \left( \mE_{\hat{\kp}^{a}_{s,N+1}}\big[ e^{\frac{-V}{\alpha}}\big]\right) \right) \nn\\
    &\quad+\frac{1}{2}\min_{\alpha\geq 0} \left(\delta\alpha+\alpha\log \left( \mE_{\hat{\kp}^{a,O}_{s,N+1}}\big[ e^{\frac{-V}{\alpha}}\big]\right) \right)\nn\\
    &\quad+\frac{1}{2}\min_{\alpha\geq 0} \left(\delta\alpha+\alpha\log \left( \mE_{\hat{\kp}^{a,E}_{s,N+1}}\big[ e^{\frac{-V}{\alpha}}\big]\right) \right). 
\end{align}

\begin{theorem}\cite{liu2022distributionally}\label{thm:kl}
The estimated operator $\hat{\sigma}_{\cp^a_s}$ is unbiased and has bounded variance, i.e., there exists a constant $C_0$, such that $ \text{Var}(\hat{\sigma}_{\cp^a_s}V)\leq C_0(1+\|V\|^2).$
\end{theorem}

\subsection{Wasserstein Distance Uncertainty Sets}
To study the support function w.r.t. this uncertainty model, we first introduce some notation. 
\begin{definition}
For any function $f: \mathcal{Z}\to \mathbb{R}$, $\lambda\geq 0$ and $x\in\mathcal{Z}$, define the regularization operator 
\begin{align}
    \Phi(\lambda, x)\triangleq \inf_{y\in\mathcal{Z}}( \lambda d(x,y)^l+f(y)).
\end{align}

The growth rate $\kappa$ of function $f$ and any distribution $q$ over $\mathcal{Z}$ is defined as
\begin{align}
    \kappa_q\triangleq \inf \bigg(\lambda\geq 0: \sum_{x\in\mathcal{Z}}  q(x)\Phi(\lambda,x) >-\infty\bigg). 
\end{align}
\end{definition}

\begin{lemma}\cite{gao2022distributionally}
Consider the distributional robust optimization of a function $f$:
\begin{align}\label{eq:wd primal}
    \inf_{W_l(q,p)\leq \delta} \mE_{x\sim q}[f(x)],
\end{align}
and define its dual problem as
\begin{align}
    \sup_{\lambda\geq 0} (-\lambda\delta^l+\sum_{x\in\mathcal{Z}} p(x) \inf_{y\in\mathcal{Z}}(f(y)+\lambda d(x,y)^l )). 
\end{align}
If $\kappa_p<\infty$, then the strong duality holds, i.e., \begin{align}
    \inf_{W_l(q,p)\leq \delta} \mE_{x\sim q}[f(x)]= \sup_{\lambda\geq 0} (-\lambda\delta^l+\sum_{x\in\mathcal{Z}} p(x) \inf_{y\in\mathcal{Z}}(f(y)+\lambda d(x,y)^l )). 
\end{align}
\end{lemma}

We first verify that this strong duality holds for our support function.
\begin{lemma}(Restatement of \Cref{eq:wd dual})
It holds that 
\begin{align}
    \sigma_{\cp^a_s}(V)=\sup_{\lambda\geq 0}\bigg(-\lambda\delta^l+\sum_x \kp^a_{s,x}\inf_{y}(V(y)+\lambda d(x,y)^l ) \bigg).
\end{align}
\end{lemma}
\begin{proof}
In our case, the  regularization operator is 
\begin{align}
    \Phi(\lambda,x)=\inf_{s\in\mcs}(\lambda d(s,x)^l+V(s) ).
\end{align}
Note that for any $\lambda\geq 0$, 
\begin{align}
    \sum_{x\in\mcs} \kp^a_s(x)\Phi(\lambda,x) = \sum_{x\in\mcs} \kp^a_s(x)\inf_{s\in\mcs}(\lambda d(s,x)^l+V(s))\geq -\|V\|>-\infty.
\end{align}
Hence, the growth rate $\kappa_{\kp^a_s}=0<\infty$. Thus, the strong duality holds.
\end{proof}

Then, the estimator under the Wasserstein distance uncertainty set can be constructed as
\begin{align}\label{eq:p_wd}
    \hat{\sigma}_{\cp^a_s}V&\triangleq \sup_{\lambda\geq 0}\bigg(-\lambda\delta^l+\inf_{y}(V(y)+\lambda d(s_1',y)^l ) \bigg)+\frac{\Delta_N(V)}{p_N}+r(s,a),
\end{align}
where
\begin{align*}
    &\Delta_N(V)\\
    &= \sup_{\lambda\geq 0}\bigg(-\lambda\delta^l+\mE_{\hat{\kp}^{a}_{s,N+1}}\bigg[\inf_{y}(V(y)+\lambda d(S,y)^l )\bigg] \bigg) \nn\\
    &\quad- \sup_{\lambda\geq 0}\bigg(-\lambda\delta^l+\mE_{\hat{\kp}^{a,O}_{s,N+1}}\bigg[\inf_{y}(V(y)+\lambda d(S,y)^l )\bigg] \bigg)\nn\\
    &\quad- \sup_{\lambda\geq 0}\bigg(-\lambda\delta^l+\mE_{\hat{\kp}^{a,E}_{s,N+1}}\bigg[\inf_{y}(V(y)+\lambda d(S,y)^l )\bigg] \bigg)\nn. 
\end{align*}

\begin{theorem}\label{thm:wd unbiased}
The estimated operator defined in \eqref{eq:p_wd} is unbiased, i.e., 
\begin{align}
    \mE[\hat{\sigma}_{\cp^a_s}V]=\sigma_{\cp^a_s}(V). 
\end{align}
\end{theorem}
\begin{proof}
Denote the dual function \eqref{eq:wd dual} by $g$:
\begin{align}
g_{s,a}^V(\lambda)=-\lambda\delta^l+\mE_{S\sim\kp^a_s}[\inf_{x\in\mcs} (V(x)+\lambda d(S,x)^l)], 
\end{align}
and denote its optimal solution by $\lambda^V_{s,a}$: 
\begin{align}
    \lambda^V_{s,a}=\arg\max_{\lambda\geq 0}\bigg( -\lambda\delta^l+\mE_{S\sim\kp^a_s}[\inf_{x\in\mcs} (V(x)+\lambda d(S,x)^l)]\bigg).
\end{align}
Then, the support function $\sigma_{\cp^a_s}(V)=g_{s,a}^V(\lambda^V_{s,a})$. 
Similarly, define the empirical function  
${g}_{s,a,N+1}^V,
{g}_{s,a,N+1,O}^V, 
{g}_{s,a,N+1,E}^V$,
and denote their optimal solutions by $\lambda^V_{s,a,N+1},\lambda^V_{s,a,N+1,O},\lambda^V_{s,a,N+1,E}$. We have that
\begin{align}
    \mE[\hat{\sigma}_{\cp^a_s}V]&=\mE[g^V_{s,a,0}(\lambda^V_{s,a,0})]+\mE\bigg[\frac{\Delta_N(V)}{p_N} \bigg]\nn\\
    &=\mE[g^V_{s,a,0}(\lambda^V_{s,a,0})]+\sum^\infty_{n=0}p(N=n)\mE\bigg[\frac{\Delta_N(V)}{p_N}|N=n \bigg]\nn\\
    &=\mE[g^V_{s,a,0}(\lambda^V_{s,a,0})]+\sum^\infty_{n=0}\mE[\Delta_n]\nn\\
    &=\mE[g^V_{s,a,0}(\lambda^V_{s,a,0})]+\sum^\infty_{n=0}\mE\bigg[g^V_{s,a,n+1}(\lambda^V_{s,a,n+1})-\frac{g^V_{s,a,n+1,O}(\lambda^V_{s,a,n+1,O})+g^V_{s,a,n+1,E}(\lambda^V_{s,a,n+1,E})}{2}\bigg]\nn\\
    &=\mE[g^V_{s,a,0}(\lambda^V_{s,a,0})]+\sum^\infty_{n=0}\mE\bigg[g^V_{s,a,n+1}(\lambda^V_{s,a,n+1})-g^V_{s,a,n}(\lambda^V_{s,a,n})\bigg],
\end{align}
where the last inequality is from Lemma \ref{lemma:Eo=Ee}. The last equation can be further rewritten as
\begin{align}
    \mE[\hat{\sigma}_{\cp^a_s}V]&=\mE[g^V_{s,a,0}(\lambda^V_{s,a,0})]+\sum^\infty_{n=0}\mE\bigg[g^V_{s,a,n+1}(\lambda^V_{s,a,n+1})-g^V_{s,a,n}(\lambda^V_{s,a,n})\bigg]\nn\\
    &=\lim_{n\to\infty}\mE\bigg[g^V_{s,a,n}(\lambda^V_{s,a,n})\bigg].
\end{align}
To show that $\hat{\sigma}_{\cp^a_s}$ is unbiased, it suffices to prove that 
\begin{align}
    \lim_{n\to\infty}\mE\bigg[g^V_{s,a,n}(\lambda^V_{s,a,n})\bigg]=g^V_{s,a}(\lambda^V_{s,a}). 
\end{align}
For any arbitrary i.i.d. samples $\{X_i\}$ and its corresponding function $g^V_{s,a,n}$, together with Lemma \ref{lemma:wd bounded}, we have that
\begin{align} 
    &|g^V_{s,a,n}(\lambda^V_{s,a,n})-g^V_{s,a}(\lambda^V_{s,a})|\nn\\
    &=|\max_{0\leq \lambda \leq \frac{2\|V\|}{\delta^l}}g^V_{s,a}(\lambda) -\max_{0\leq \lambda \leq \frac{2\|V\|}{\delta^l}}g^V_{s,a,n}(\lambda) |\nn\\
    &\leq \max_{0\leq \lambda \leq \frac{2\|V\|}{\delta^l}} |g^V_{s,a}(\lambda)-g^V_{s,a,n}(\lambda)|\nn\\
    &=\max_{0\leq \lambda \leq \frac{2\|V\|}{\delta^l}} \bigg|\mE_{S\sim\kp^a_s}[\inf_{x\in\mcs} (V(x)+\lambda d(S,x)^l)]- \mE_{S\sim\hat{\kp}^{a}_{s,n}}[\inf_{x\in\mcs} (V(x)+\lambda d(S,x)^l )] \bigg|\nn\\
    &\leq \max_{0\leq \lambda \leq \frac{2\|V\|}{\delta^l}}\|\kp^a_s-\hat{\kp}^{a}_{s,n} \|_1 \sup_{x,S\in\mcs}(|V(x)+\lambda d(S,x)^l|) \nn\\
   &\leq \bigg(1+\frac{2D^l}{\delta^l}\bigg)\|V\|\|\kp^a_s-\hat{\kp}^{a}_{s,n} \|_1,
\end{align}
where the last inequality is from the bound on $\lambda$ and $D$ is the diameter of the metric space $(\mcs,d)$.

By Hoeffding's inequality and similar to the previous proofs, we have that
\begin{align}\label{eq:135}
    \mE[|g^V_{s,a,n}(\lambda^V_{s,a,n})-g^V_{s,a}(\lambda^V_{s,a})|]\leq \bigg(1+\frac{2D^l}{\delta^l}\bigg)\bigg(\frac{|\mcs|^2\sqrt{\pi}}{2^{\frac{n+1}{2}}}\bigg)\|V\|,
\end{align}
which implies that 
\begin{align}
    \lim_{n\to\infty}\mE\bigg[g^V_{s,a,n}(\lambda^V_{s,a,n})\bigg]=g^V_{s,a}(\lambda^V_{s,a}).
\end{align}
This completes the proof.
\end{proof}

\begin{theorem}\label{thm:wd var}
The estimated operator $\hat{\sigma}_{\cp^a_s}$  defined in \eqref{eq:p_wd}  has bounded variance, i.e., there exists a constant $C_0$, such that 
\begin{align}
     \text{Var}(\hat{\sigma}_{\cp^a_s}V)\leq (3+2C_0)\|V\|^2. 
\end{align}
\end{theorem}
\begin{proof}
We first have that
\begin{align}
    &\text{Var}(\hat{\sigma}_{\cp^a_s}V)\nn\\
    &=\mE[(\hat{\sigma}_{\cp^a_s}V)^2]-\sigma_{\cp^a_s}(V)^2\nn\\
    &\leq  \mE\bigg[\bigg(g^V_{s,a,0}(\lambda^V_{s,a,0})+\frac{\Delta_N(V)}{p_N}\bigg)^2\bigg]+(\sigma_{\cp^a_s}(V))^2\nn\\
    &\leq 2\mE\bigg[\bigg(g^V_{s,a,0}(\lambda^V_{s,a,0})\bigg)^2\bigg]+2\mE\bigg[\bigg(\frac{\Delta_N(V)}{p_N}\bigg)^2\bigg]+(\sigma_{\cp^a_s}(V))^2\nn\\
    &\leq 3\|V\|^2+2\sum^\infty_{i=0}\frac{\mE[(\Delta_i(V))^2]}{p_i},
\end{align}
where the last inequality is from Lemma \ref{lemma:cs bounded}. 
For any $n\geq 1$, we have that
\begin{align}
    \mE[(\Delta_n(V))^2]&=\mE\bigg[ \bigg({g}^V_{s,a,n}(\lambda^V_{s,a,n}) -\frac{{g}^V_{s,a,n,E}(\lambda^V_{s,a,n,E})+{g}^V_{s,a,n,O}(\lambda^V_{s,a,n,O})}{2} \bigg)^2\bigg]\nn\\
    &=\mE\bigg[ \bigg({g}^V_{s,a,n}(\lambda^V_{s,a,n})-g^V_{s,a}(\lambda^V_{s,a}) +g^V_{s,a}(\lambda^V_{s,a})-\frac{{g}^V_{s,a,n,E}(\lambda^V_{s,a,n,E})+{g}^V_{s,a,n,O}(\lambda^V_{s,a,n,O})}{2} \bigg)^2\bigg]\nn\\
    &\leq 2\mE[({g}^V_{s,a,n}(\lambda^V_{s,a,n})-g^V_{s,a}(\lambda^V_{s,a}))^2]+ 2\mE\bigg[\bigg(g^V_{s,a}(\lambda^V_{s,a})-\frac{{g}^V_{s,a,n,E}(\lambda^V_{s,a,n,E})+{g}^V_{s,a,n,O}(\lambda^V_{s,a,n,O})}{2}\bigg)^2 \bigg]\nn\\
    &\overset{(a)}{=}2\mE[({g}^V_{s,a,n}(\lambda^V_{s,a,n})-g^V_{s,a}(\lambda^V_{s,a}))^2]+2\mE[({g}^V_{s,a,n-1}(\lambda^V_{s,a,n-1})-g^V_{s,a}(\lambda^V_{s,a}))^2]\nn\\
    &\leq 2\bigg(1+\frac{2D^l}{\delta^l}\bigg)^2\|V\|^2 \mE[\|\kp^a_s-\hat{\kp}^{a}_{s,n}\|_1^2]+2\bigg(1+\frac{2D^l}{\delta^l}\bigg)^2\|V\|^2\mE[\|\kp^a_s-\hat{\kp}^{a}_{s,n-1}\|_1^2],
\end{align}
where $(a)$ is due to Lemma \ref{lemma:Eo=Ee} and the last inequality follows a similar argument to \eqref{eq:135}. Note that $p_n=\Psi(1-\Psi)^n$ for $\Psi\in\big(0,0.5\big)$, thus similar to Theorem 3.7 of \cite{liu2022distributionally}, we can show that there exists a constant $C_0$, such that 
\begin{align}
    \sum^\infty_{i=0}\frac{\mE[(\Delta_i(V))^2]}{p_i}\leq C_0\|V\|^2. 
\end{align}
Thus, we have that
\begin{align}
     \text{Var}(\hat{\sigma}_{\cp^a_s}V)\leq 3\|V\|^2+2C_0\|V\|^2=(3+2C_0)\|V\|^2.
\end{align}
\end{proof}

\section{Case Studies for Robust RVI Q-Learning}\label{sec:case for Q}
In this section, we provide the proof of the second part of \Cref{thm:case thm}, i.e., $\hat{\mathbf H}$ is bounded and unbiased under each uncertainty model. We note that the proofs in this part can be easily derived by following the ones in \Cref{sec:case for TD}.

We first prove a lemma necessary to the proofs in this section. 
\begin{lemma}\label{lemma:vq}
    It holds that 
    \begin{align}
        \|V_Q\|\leq \|Q\|.
    \end{align}
\end{lemma}
\begin{proof}
    From the definition of $V_Q$, we have that 
    \begin{align}
    \|V_Q\|=\max_s |V_Q(s)|=\max_s |\max_a Q(s,a)| \triangleq |Q(s^*,a^*)|.
    \end{align}
    Clearly, $|Q(s^*,a^*)|\leq \max_{s,a} |Q(s,a)|$, hence 
        \begin{align}
        \|V_Q\|\leq \|Q\|.
    \end{align}
\end{proof}

Similar to \Cref{sec:case for TD}, the propositions of $\hat{\mathbf H}$ can be reduced to the ones of $\hat{\sigma}_{\cp^a_s}$. 
\begin{lemma}
    If 
    $\mE[\hat{\sigma}_{\cp^a_s}V]=\sigma_{\cp^a_s}(V)$,
    and moreover there exists a constant $C$, such that for any $s,a$, 
        $\text{Var}(\hat{\sigma}_{\cp^a_s}V)\leq C(1+\|V\|^2)$,
    then 
    $\mE[\hat{\mathbf{H}} Q(s,a)]=\mathbf H Q(s,a),$
    and $
        \text{Var}(\hat{\mathbf H}Q(s,a))\leq C(1+\|Q\|^2).$
\end{lemma}
\begin{proof}
First, we have that
    \begin{align}
    \mE[\hat{\mathbf{H}} Q(s,a)]=\mE[r(s,a)+\hat{\sigma}_{\cp^a_s} V_Q(s)]=r(s,a)+\sigma_{\cp^a_s}(V_Q)=\mathbf{H}Q(s,a). 
    \end{align}
For boundedness, note that 
\begin{align}
    \text{Var}(\hat{\mathbf H}Q(s,a))&=\mE\bigg[(\hat{\mathbf H}Q(s,a)-\mathbf{H}Q(s,a))^2 \bigg]\nn\\
    &=\mE\big[\big(\hat{\sigma}_{\cp^a_s} V_Q(s)-\sigma_{\cp^a_s}(V_Q) \big)^2\big]\nn\\
    &\leq C(1+\|V_Q\|^2)\nn\\
    &\leq C(1+\|Q\|^2),
\end{align}
where the last inequality is from Lemma \ref{lemma:vq}. 
\end{proof}
This implies that the problem is reduced to verifying whether $\hat{\sigma}_{\cp^a_s}$ is unbiased and has bounded variance, which is identical to the results in \Cref{sec:case for TD}. We thus omit the proofs for this part.

\section{Technical Lemmas}
\begin{lemma}\label{lemma:x=y+r}
For a robust Bellman operator $\mathbf T$, define 
\begin{align}
    {\mathbf T}'V&\triangleq {\mathbf T}V-f(V)e,\\
    \tilde{\mathbf T}V&\triangleq {\mathbf T}V-g^\pi_\cp e.
\end{align}
Assume that $x(t),y(t)$ are the solutions to equations 
\begin{align}
    \dot{x}&={\mathbf T}'x-x,\\
    \dot{y}&=\tilde{\mathbf T}y-y,
\end{align}
with the same initial value $x(0)=y(0)=x_0$. Then, 
\begin{align}
    x(t)=y(t)+r(t)e,
\end{align}
where $r(t)$ satisfies 
\begin{align}
    \dot{r}(t)=-r(t)+g^\pi_\cp-f(y(t)). 
\end{align}
\end{lemma}
\begin{proof}
    Note that ${\mathbf T}'V=\tilde{\mathbf T}V+(g^\pi_\cp-f(V))e$, then from the variation of constants formula, we have that
\begin{align}
    x(t)&=x_0e^{-t}+\int^t_0e^{-(t-s)}\tilde{\mathbf T}(x(s))ds+ \left( \int^t_0 e^{-(t-s)} (g^\pi_\cp-f(x(s)))ds\right)e,\\
    y(t)&=x_0e^{-t}+\int^t_0 e^{-(t-s)}\tilde{\mathbf T}(y(s))ds. 
\end{align}
Hence, the maximal and minimal components of $x(t)-y(t)$ can be bounded as:
\begin{align}
    \max_i (x_i(t)-y_i(t))&\leq \int^t_0 e^{-(t-s)}\max_i (\tilde{\mathbf T}_i(x(s))-\tilde{\mathbf T}_i(y(s)))ds+\int^t_0 e^{-(t-s)} (g^\pi_\cp-f(x(s)))ds,\nn\\
    \min_i (x_i(t)-y_i(t))&\geq \int^t_0 e^{-(t-s)}\min_i (\tilde{\mathbf T}_i(x(s))-\tilde{\mathbf T}_i(y(s)))ds+\int^t_0 e^{-(t-s)} (g^\pi_\cp-f(x(s)))ds. 
\end{align}
This hence implies that
\begin{align}
    \spa(x(t)-y(t))&\leq \int^t_0 e^{-(t-s)} \spa(\tilde{\mathbf T}(x(s))-\tilde{\mathbf T}(y(s)))ds\nn\\
    &\leq \int^t_0 e^{-(t-s)} \spa(x(s)-y(s))ds,
\end{align}
where the last inequality is because $\tilde{\mathbf T}$ is  non-expansive w.r.t. the span semi-norm \cite{wang2023robust}.

Gronwall's inequality implies that  $\spa(x(t)-y(t))\leq 0\cdot \int^t_0 e^{-(t-s)}ds=0$ for any $t\geq 0$. However, since $\spa$ is non-negative, then $\spa(x(t)-y(t))= 0$. Hence, we have that $x(t)=y(t)+r(t)e$ for some $r(t)$ satisfying $r(0)=0$. 

Also note that the differential of $r(t)$ can be written as 
\begin{align}
    \dot{r}(t)e&=\dot{x}(t)-\dot{y}(t)\nn\\
    &=\tilde{\mathbf T}x(t)+(g^\pi_\cp-f(x(t)))e-x(t)-\tilde{\mathbf T}y(t)+y(t)\nn\\
    &=(-r(t)+g^\pi_\cp-f(y(t)))e,
\end{align}
where the last equation is because 
\begin{align}
    \tilde{\mathbf T}x(t)&=\tilde{\mathbf T}(y(t)+r(t)e)=\tilde{\mathbf T}(y(t))+r(t)e,\\
    f(x(t))&=f(y(t)+r(t)e)=f(y(t))+r(t). 
\end{align}
This completes the proof. 
\end{proof}

\begin{lemma}\label{lemma:Eo=Ee}
For any function $g: \Delta(|\mcs|)\to \mathbb{R}$, assume there are $2^{n+1}$ i.i.d. samples $X_i\sim q$. Denote the empirical distributions from samples $\{X_i: i=1,...,2^{n+1}\},\{X_{2i-1}: i=1,...,2^{n}\},\{X_{2i}: i=1,...,2^{n}\}$ by $\hat{q}_{n+1},\hat{q}_{n+1,O},\hat{q}_{n+1,E}$. Then, 
\begin{align}
    \mE[g(\hat{q}_{n+1,O})]= \mE[g(\hat{q}_{n+1,E})]= \mE[g(\hat{q}_{n})]. 
\end{align}
\end{lemma}
\begin{proof}
Note that 
\begin{align}
    \hat{q}_{n+1,O}(s)=\frac{\sum_{i=1}^{2^n}\mathbbm{1}_{X_{2i-1}=s}}{2^n},
\end{align}
hence,
\begin{align}
    \mE[g(\hat{q}_{n+1,O})]&=\sum_{p=(p_1,...,p_{|\mcs|})\in\Delta(|\mcs|)} g(p)\mathbb{P}(\hat{q}_{n+1,O}=p)\nn\\
    &=\sum_{p=(p_1,...,p_{|\mcs|})\in\Delta(|\mcs|)} \mathbb{P}\bigg(\frac{\sum^{2^n}_{i=1}\mathbbm{1}_{X_{2i-1}=s_1}}{2^n}=p_1,...,\frac{\sum^{2^n}_{i=1}\mathbbm{1}_{X_{2i-1}=s_{|\mcs|}}}{2^n}=p_{|\mcs|}\bigg|q\bigg) g(p),
\end{align}
where $2^np_i\in\mathbb{N}$ and $\sum_{i=1}^{|\mcs|} p_i=1$. 
On the other hand, 
\begin{align}
    \mE[g(\hat{q}_{n})]&=\sum_{p=(p_1,...,p_{|\mcs|})\in\Delta(|\mcs|)} g(p)\mathbb{P}(\hat{q}_{n}=p)\nn\\
    &=\sum_{p=(p_1,...,p_{|\mcs|})\in\Delta(|\mcs|)} \mathbb{P}\bigg(\frac{\sum^{2^n}_{i=1}\mathbbm{1}_{X_{i}=s_1}}{2^n}=p_1,...,\frac{\sum^{2^n}_{i=1}\mathbbm{1}_{X_{i}=s_{|\mcs|}}}{2^n}=p_{|\mcs|}\bigg|q\bigg) g(p).
\end{align}
Note that $X_i$ are i.i.d., hence, 
\begin{align}
    \mathbb{P}\bigg(\frac{\sum^{2^n}_{i=1}\mathbbm{1}_{X_{i}=s_1}}{2^n}=p_1,...,\frac{\sum^{2^n}_{i=1}\mathbbm{1}_{X_{i}=s_{|\mcs|}}}{2^n}=p_{|\mcs|}\bigg|q\bigg)=\mathbb{P}\bigg(\frac{\sum^{2^n}_{i=1}\mathbbm{1}_{X_{2i-1}=s_1}}{2^n}=p_1,...,\frac{\sum^{2^n}_{i=1}\mathbbm{1}_{X_{2i-1}=s_{|\mcs|}}}{2^n}=p_{|\mcs|}\bigg|q\bigg). \nn
\end{align}
Thus, 
\begin{align}
    \mE[g(\hat{q}_{n+1,O})]=\mE[g(\hat{q}_{n})]. 
\end{align}
Similarly, $\mE[g(\hat{q}_{n+1,E})]=\mE[g(\hat{q}_{n})]$ and hence it completes the proof. 
\end{proof}

\begin{lemma}\label{lemma:tv bounded}
Under the total variation uncertainty model, the optimal solution and optimal value for $g_{s,a}^V,g_{s,a,N+1}^V,g_{s,a,N+1,E}^V,g_{s,a,N+1,O}^V$  are bounded. Specifically,
\begin{align}
     \mu^V_{s,a},\mu^V_{s,a,N+1},\mu^V_{s,a,N+1,E},\mu^V_{s,a,N+1,O}&\leq V+\|V\|e,\\
    \|\mu^V_{s,a}\|,\|\mu^V_{s,a,N+1}\|,\|\mu^V_{s,a,N+1,E}\|,\|\mu^V_{s,a,N+1,O}\|&\leq 2\|V\|,\\
      |g_{s,a}^V(\mu^V_{s,a})|,|g_{s,a,N+1}^V(\mu^V_{s,a,N+1})|,|g_{s,a,N+1,O}^V(\mu^V_{s,a,N+1,O})|,|g_{s,a,N+1,E}^V(\mu^V_{s,a,N+1,E})|&\leq 3(1+2\delta)\|V\|. 
\end{align}
\end{lemma}
\begin{proof}
First we show the bounds on the optimal solutions.
If we denote the minimal entry of $V$ by $w$: $w=\min_s V(s)$, then $W\triangleq V-we\geq 0$. 
Note that, 
\begin{align}
    \mu^W_{s,a}&=\arg\max_{\mu\geq 0}\big( \kp^a_s(W-\mu)-\delta \spa(W-\mu)\big)\nn\\
    &=\arg\max_{\mu\geq 0}\big(-w+ \kp^a_s(V-\mu)-\delta \spa(V-\mu)\big),
\end{align}
which is because $\spa(V+ke)=\spa(V)$ and $\kp^a_s(V+ke)=k+\kp^a_s V$. Hence, $\mu^W_{s,a}=\mu^V_{s,a}$. Moreover note that $W\geq 0$, hence $\mu^W_{s,a}$ is bounded: $\mu^W_{s,a}\leq W$, this further implies that
\begin{align}
    \|\mu^V_{s,a}\|=\|\mu^W_{s,a}\|\leq \|W\|\leq 2\|V\|. 
\end{align}
The bounds on $\mu^V_{s,a,N+1},\mu^V_{s,a,N+1,O},\mu^V_{s,a,N+1,E}$ can be similarly derived. 

We then consider the optimal value. Note that,
\begin{align}
    g_{s,a}^V(\mu^V_{s,a})&=\kp^a_s(V-\mu^V_{s,a})-\delta \spa(V-\mu^V_{s,a})\nn\\
    &\leq \|V\|+\|\mu^V_{s,a}\|+\delta |\max_i (V(i)-\mu^V_{s,a}(i))|+\delta|\min_i (V(i)-\mu^V_{s,a}(i))|\nn\\
    &\leq 3\|V\|+2\delta(\|V\|+\|\mu^V_{s,a}\|)\nn\\
    &\leq 3(1+2\delta)\|V\|. 
\end{align}
On the other hand, 
\begin{align}\label{eq:g<}
    g_{s,a}^V(\mu^V_{s,a})\geq g_{s,a}^V(0)=\kp^a_s V-\delta\spa(V)=\kp^a_s V-\delta\max_i V(i)+\delta \min_i V(i).
\end{align}
Denote the maximal and minimal entries of $V$ by $V(M)$ and $V(m)$, then we have that
\begin{align}\label{eq:g>}
    &\kp^a_s V-\delta\max_i V(i)+\delta \min_i V(i)\nn\\
    &=\sum_{x} \kp^a_{s,x}V(x)-\delta V(M)+\delta V(m)\nn\\
    &\geq -\|V\|-2\delta \|V\|,
\end{align}
where the last inequality is from $\|V\|\geq V(i)\geq -\|V\|$ for any entry $i$. Thus, combining \eqref{eq:g<} and \eqref{eq:g>} implies that 
\begin{align}
     -(1+2\delta)\|V\|\leq g_{s,a}^V(\mu^V_{s,a})\leq 3(1+2\delta)\|V\|. 
\end{align}
Similarly, the bounds on $g_{s,a,N+1}^V(\mu^V_{s,a,N+1}),g_{s,a,N+1,O}^V(\mu^V_{s,a,N+1,O}),g_{s,a,N+1,E}^V(\mu^V_{s,a,N+1,E})$ can be derived. 
\end{proof}

\begin{lemma}\label{lemma:cs bounded}
Under the chi-square uncertainty model, the optimal solution and optimal value for $g_{s,a}^V,g_{s,a,N+1}^V,g_{s,a,N+1,E}^V,g_{s,a,N+1,O}^V$  are bounded. Specifically,
\begin{align}
     \mu^V_{s,a},\mu^V_{s,a,N+1},\mu^V_{s,a,N+1,E},\mu^V_{s,a,N+1,O}&\leq V+\|V\|e,\\
    \|\mu^V_{s,a}\|,\|\mu^V_{s,a,N+1}\|,\|\mu^V_{s,a,N+1,E}\|,\|\mu^V_{s,a,N+1,O}\|&\leq 2\|V\|,\\
      |g_{s,a}^V(\mu^V_{s,a})|,|g_{s,a,N+1}^V(\mu^V_{s,a,N+1})|,|g_{s,a,N+1,O}^V(\mu^V_{s,a,N+1,O})|,|g_{s,a,N+1,E}^V(\mu^V_{s,a,N+1,E})|&\leq 3(1+\sqrt{2 \delta})\|V\|. 
\end{align}
\end{lemma}
\begin{proof}
First, we show the bounds on the optimal solutions.
If we denote the minimal entry of $V$ by $w$: $w=\min_s V(s)$, then $W\triangleq V-we\geq 0$. 
Note that, 
\begin{align}
    \mu^W_{s,a}&=\arg\max_{W\geq\mu\geq 0}\big( \kp^a_s(W-\mu)- \sqrt{\delta\text{Var}_{\kp^a_s} (W-\mu)}\big)\nn\\
    &=\arg\max_{W\geq \mu\geq 0}\big(-w+ \kp^a_s(V-\mu)- \sqrt{\delta\text{Var}_{\kp^a_s} (V-\mu)}\big),
\end{align}
which is because $\text{Var}_{\kp^a_s} (V-\mu-we)=\text{Var}_{\kp^a_s}(V-\mu)+\text{Var}_{\kp^a_s}(we)-2\text{Cov}_{\kp^a_s}(V-\mu,we)=\text{Var}_{\kp^a_s} (V-\mu)$. Hence $\mu^W_{s,a}=\mu^V_{s,a}$. Moreover note that $W\geq 0$, hence $\mu^W_{s,a}$ is bounded: $\mu^W_{s,a}\leq W$, this further implies that
\begin{align}
    \|\mu^V_{s,a}\|=\|\mu^W_{s,a}\|\leq \|W\|\leq 2\|V\|. 
\end{align}
The bounds on $\mu^V_{s,a,N+1},\mu^V_{s,a,N+1,O},\mu^V_{s,a,N+1,E}$ can be similarly derived. 

We then consider the optimal value. Note that,
\begin{align}
    |g_{s,a}^V(\mu^V_{s,a})|&=|\kp^a_s(V-\mu^V_{s,a})- \sqrt{\delta\text{Var}_{\kp^a_s} (V-\mu^V_{s,a})}|\nn\\
    &\leq \|V\|+\|\mu^V_{s,a}\|+  \sqrt{2\delta\|V-\mu^V_{s,a}\|^2}\nn\\
    &\leq 3\|V\|+\sqrt{2\delta}(\|V\|+\|\mu^V_{s,a}\|)\nn\\
    &\leq 3(1+\sqrt{2 \delta})\|V\|. 
\end{align}
Similarly, the bounds on $g_{s,a,N+1}^V(\mu^V_{s,a,N+1}),g_{s,a,N+1,O}^V(\mu^V_{s,a,N+1,O}),g_{s,a,N+1,E}^V(\mu^V_{s,a,N+1,E})$ can be derived. 
\end{proof}

\begin{lemma}\label{lemma:wd bounded}
Under the Wasserstein distance uncertainty model, the optimal solution and optimal value for $g_{s,a}^V,g_{s,a,N+1}^V,g_{s,a,N+1,E}^V,g_{s,a,N+1,O}^V$  are bounded. Specifically,
\begin{align}
     &\lambda^V_{s,a},\lambda^V_{s,a,n},\lambda^V_{s,a,n,O},\lambda^V_{s,a,n,E} \leq \frac{2\|V\|}{\delta^l}, \\
     & |g^V_{s,a}(\lambda^V_{s,a})|, |g^V_{s,a,n}(\lambda^V_{s,a,n})|, |g^V_{s,a,n,O}(\lambda^V_{s,a,n,O})|,|g^V_{s,a,n,E}(\lambda^V_{s,a,n,E})| \leq \|V\|. 
\end{align}
\end{lemma}
\begin{proof}
First, we show the bounds on the optimal solutions.
Denote the optimal solution to $\max_{\lambda\geq 0}g_{s,a}^V(\lambda)$ by $\lambda^V_{s,a}$. Moreover, for each state $y\in\mcs$ and any $\lambda \geq 0$, denote $s^y_{\lambda}\triangleq \arg\min_{x\in\mcs} \{\lambda d(x,y)^l+V(x) \}$. Hence, 
\begin{align}
    g^V_{s,a}(\lambda)=-\lambda\delta^l+\mE_{S\sim\kp^a_s}[\lambda d(S,s^S_\lambda)^l+V(s^S_\lambda)]. 
\end{align}
Moreover, note that $g^V_{s,a}(\lambda^V_{s,a})=\max_{\lambda\geq 0} g^V_{s,a}(\lambda)$, hence,
\begin{align}\label{eq:130}
    -\lambda^V_{s,a}\delta^l+\mE_{S\sim\kp^a_s}[\lambda^V_{s,a} d(S,s^S_{\lambda^V_{s,a}})^l+V(s^S_{\lambda^V_{s,a}})]\geq g^V_{s,a}(0)=\mE_{S\sim\kp^a_s}[V(s^S_0)]=\min_x V(x),
\end{align}
where the last equation is due to the fact that $s^S_0=\arg\min_{x\in\mcs}\{V(x) \}=\min_x V(x)$. 
Now consider the inner problem $\mE_{S\sim\kp^a_s}[\lambda^V_{s,a} d(S,s^S_{\lambda^V_{s,a}})^l+V(s^S_{\lambda^V_{s,a}})]$. Note that,
\begin{align}\label{eq:131}
    &\mE_{S\sim\kp^a_s}[\lambda^V_{s,a} d(S,s^S_{\lambda^V_{s,a}})^l+V(s^S_{\lambda^V_{s,a}})]\nn\\
    &=\sum_{x}\kp^a_{s,x} \big(\lambda^V_{s,a} d(x,s^x_{\lambda^V_{s,a}})^l+V(s^x_{\lambda^V_{s,a}})\big)\nn\\
    &\overset{(a)}{\leq}\sum_{x}\kp^a_{s,x} \big(\lambda^V_{s,a} d(x,x)^l+V(x)\big)\nn\\
    &=\mE_{\kp^a_s}[V(S)],
\end{align}
where $(a)$ is because $s^x_{\lambda^V_{s,a}}=\arg\min_{y\in\mcs} \{\lambda^V_{s,a} d(x,y)^l+V(y) \}$ and hence $\lambda^V_{s,a} d(x,s^x_{\lambda^V_{s,a}})^l+V(s^x_{\lambda^V_{s,a}})\leq \lambda^V_{s,a} d(x,x)^l+V(x)$. 

Combine \eqref{eq:130} and \eqref{eq:131}, then we further have that
\begin{align}
\min_x V(x)\leq -\lambda^V_{s,a}\delta^l+\mE_{S\sim\kp^a_s}[\lambda^V_{s,a} d(S,s^S_{\lambda^V_{s,a}})^l+V(s^S_{\lambda^V_{s,a}})]\leq -\lambda^V_{s,a}\delta^l+\mE_{\kp^a_s}[V(S)].
\end{align}
This implies that
\begin{align}
    \lambda^V_{s,a}\leq \frac{\mE_{\kp^a_s}[V(S)]-\min_xV(x)}{\delta^l}\leq \frac{2\|V\|}{\delta^l}, 
\end{align}
and hence $\lambda^V_{s,a}$ is bounded. 

On the other hand, note that $g^V_{s,a}(\lambda^V_{s,a})=\sigma_{\cp^a_s}[V(S)]$, hence, 
\begin{align}
    \big| g^V_{s,a}(\lambda^V_{s,a})\big| \leq \|V\|. 
\end{align}
Same bound can be similarly derived for $\lambda^V_{s,a,n},\lambda^V_{s,a,n,O},\lambda^V_{s,a,n,E},g^V_{s,a,n}(\lambda^V_{s,a,n}),g^V_{s,a,n,O}(\lambda^V_{s,a,n,O}),g^V_{s,a,n,E}(\lambda^V_{s,a,n,E})$.
\end{proof}

\begin{lemma}\label{lemma:optimal x=y+r}
For an optimal robust Bellman operator: 
    ${\mathbf H}Q(s,a)=r(s,a)+\sigma_{\cp^a_s}(V_Q),$
define 
\begin{align}
    {\mathbf H}'Q&\triangleq {\mathbf H}Q-f(Q)e,\\
    \tilde{\mathbf H}Q&\triangleq {\mathbf H}Q-g^*_\cp e.
\end{align}
Assume that $x(t),y(t)$ are the solutions to equations 
\begin{align}
    \dot{x}&={\mathbf H}'x-x,\\
    \dot{y}&=\tilde{\mathbf H}y-y,
\end{align}
with the same initial value $x(0)=y(0)$. Then 
    $x(t)=y(t)+r(t)e,$
where $r(t)$ satisfies 
    $\dot{r}(t)=-r(t)+g^*_\cp-f(y(t)). $
\end{lemma}
\begin{proof}
    The proof follows exactly that of Lemma \ref{lemma:x=y+r} if we show that $\tilde{\mathbf H}$ is non-expansion w.r.t. the span semi-norm. 

    Following Theorem 17 of \cite{wang2023robust}, it can be shown that 
    \begin{align}
        \spa(\tilde{\mathbf H}(Q_1)-\tilde{\mathbf H}(Q_2))\leq \spa(V_{Q_1}-V_{Q_2}). 
    \end{align}
    Let 
    \begin{align}
        s=\arg\max_i \{\max_a Q_1(i,a)-\max_a Q_2(i,a)\},\\
        t=\arg\min_i \{\max_a Q_1(i,a)-\max_a Q_2(i,a)\}.
    \end{align}
Then, 
\begin{align}
    \spa(V_{Q_1}-V_{Q_2})&=(\max_a Q_1(s,a)-\max_a Q_2(s,a))-(\max_a Q_1(t,a)-\max_a Q_2(t,a))\nn\\
    &\leq Q_1(s,a_s)-Q_2(s,a_s) - ( Q_1(t,a_t)- Q_2(t,a_t))\nn\\
    &\leq \max_{x,b} (Q_1(x,b)-Q_2(x,b)) - \min_{x,b} (Q_1(x,b)-Q_2(x,b))\nn\\
    &=\spa(Q_1-Q_2). 
\end{align}
where $a_s=\arg\max_a Q_1(S,a) $ and $a_t=\arg\max_a Q_2(t,a)$. This completes the proof. 
\end{proof}

\newpage
\section{Additional Experiments}\label{sec:add_exp}
In this section, we first show the additional experiments on the Garnet problem in \Cref{sec:exp_gar}. Then, we further verify our theoretical results using some additional experiments. 

\subsection{Garnet Problem}
We first verify the convergence of our robust RVI TD and robust RVI Q-learning under the Garnet problem with the same setting as in \Cref{sec:exp_gar}. Our results show that both our algorithms converge to the (optimal) robust average-reward under the other three uncertainty sets. 
\begin{figure}[!h]
\begin{center}
\subfigure{
\label{Fig.conTD}
\includegraphics[width=0.3\linewidth]{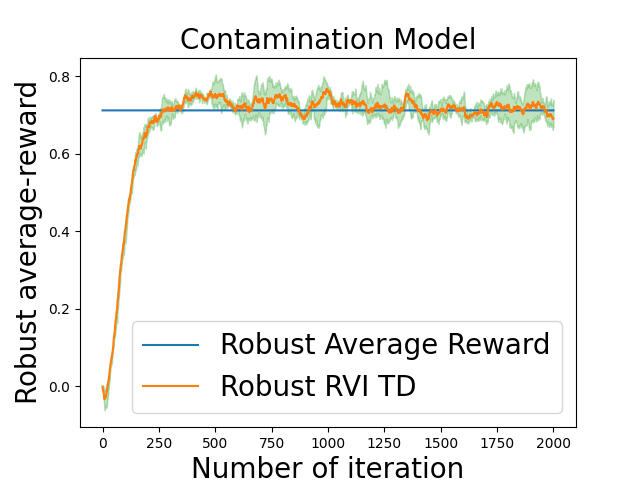}}
\subfigure{
\label{Fig.TVTD}
\includegraphics[width=0.3\linewidth]{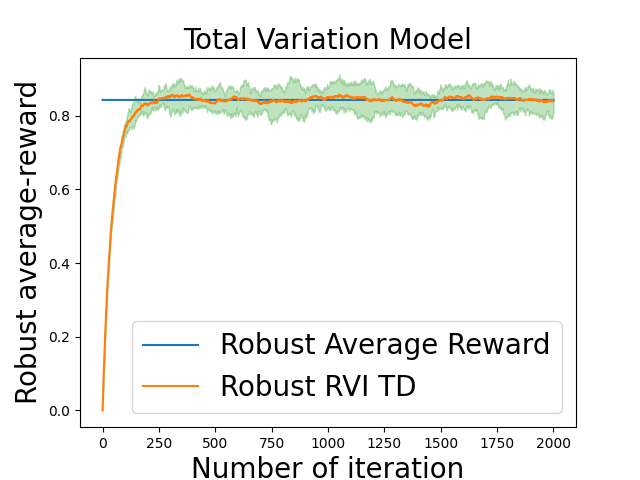}}
\subfigure{
\label{Fig.KLTD}
\includegraphics[width=0.3\linewidth]{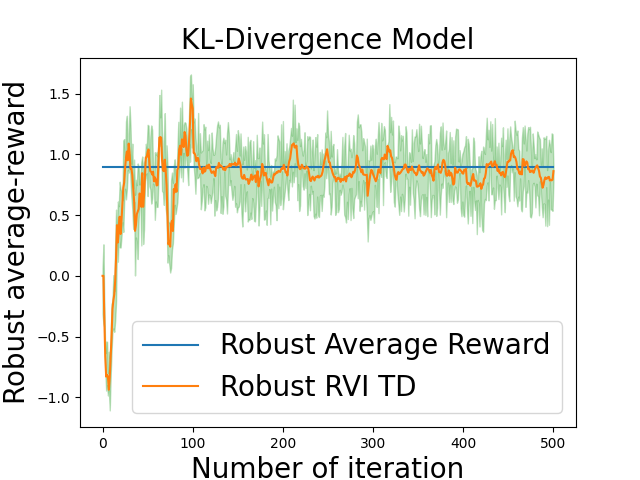}}
\caption{Robust RVI TD Algorithm under Garnet Problem.}
\end{center}
\vskip -0.2in
\end{figure}

\begin{figure}[!h]
\begin{center}
\subfigure{
\label{Fig.conQ}
\includegraphics[width=0.3\linewidth]{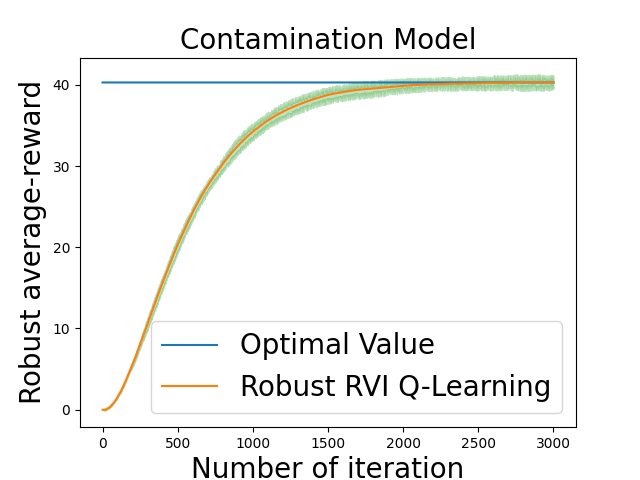}}
\subfigure{
\label{Fig.TVQ}
\includegraphics[width=0.3\linewidth]{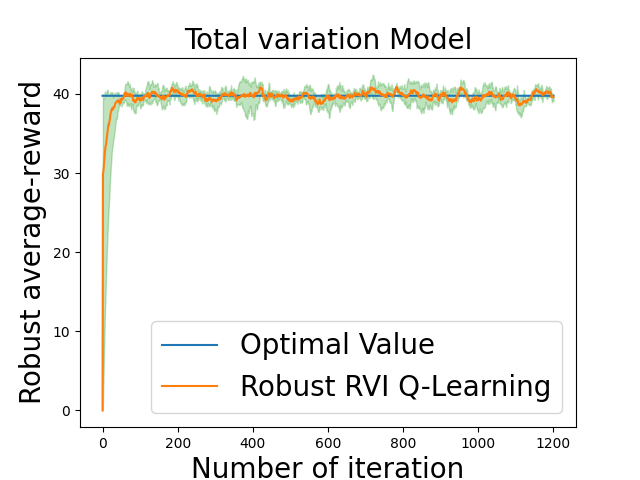}}
\subfigure{
\label{Fig.KLQ}
\includegraphics[width=0.3\linewidth]{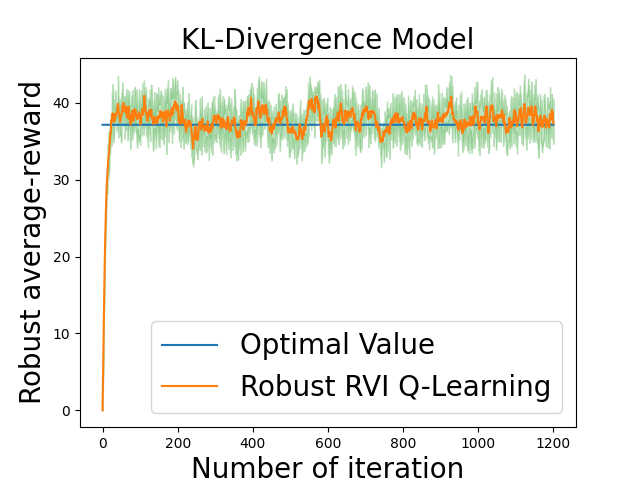}}
\caption{Robust RVI Q-Learning Algorithm under Garnet Problem.}
\end{center}
\vskip -0.2in
\end{figure}

\newpage
\subsection{Frozen-Lake Problem}
We first verify our robust RVI TD algorithm and robust RVI Q-learning under the Frozen-Lake environment of OpenAI \cite{brockman2016openai}. We set the uncertainty radius $\delta=0.4$, $\alpha_n=0.01$ and plot the (optimal) robust average-reward computed using model-based methods in \cite{wang2023robust} as the baseline. We evaluate the uniform policy for the policy evaluation problem, plot the average value of $f(V_t)$ of 30 trajectories and plot the 95/5 percentile as the upper/lower envelope.  For the optimal control problem, we plot the average value of $f(Q_t)$ of 30 trajectories and plot the 95/5 percentile as the upper/lower envelope. The results show that both algorithms converge to the (optimal) robust average-reward. 
\begin{figure}[!htp] 
\begin{center}
\subfigure{
\label{Fig.FLconTD}
\includegraphics[width=0.3\linewidth]{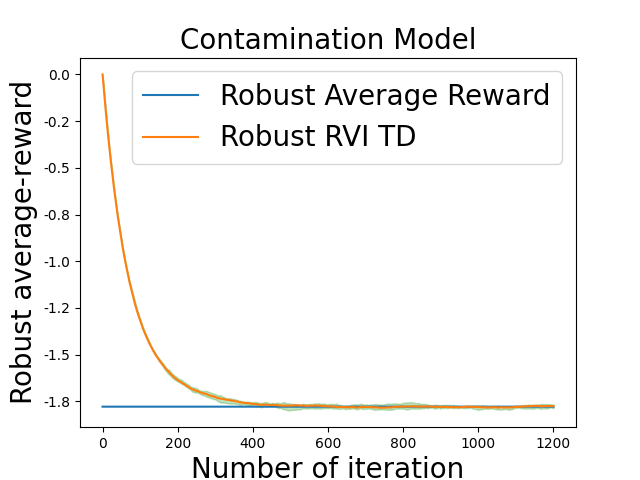}}
\subfigure{
\label{Fig.FLTVTD}
\includegraphics[width=0.3\linewidth]{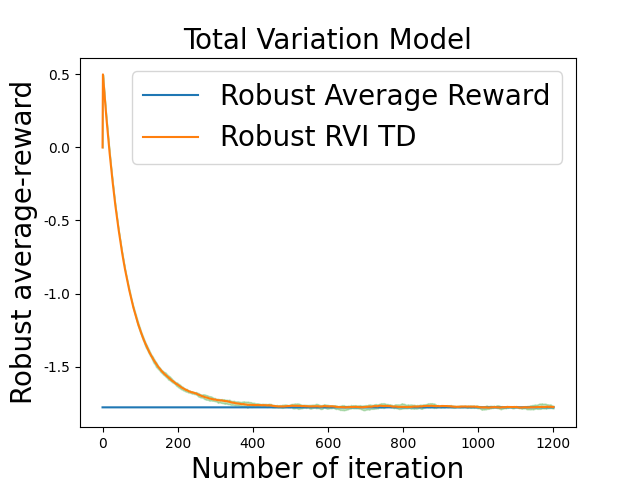}}
\subfigure{
\label{Fig.FLCSTD}
\includegraphics[width=0.3\linewidth]{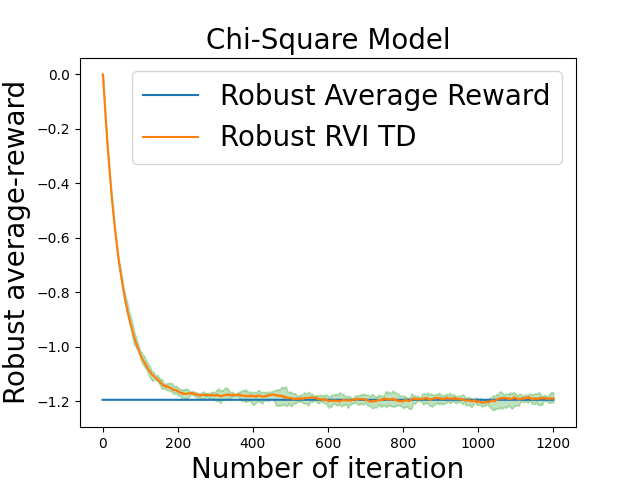}}
\subfigure{
\label{Fig.FLKLTD}
\includegraphics[width=0.3\linewidth]{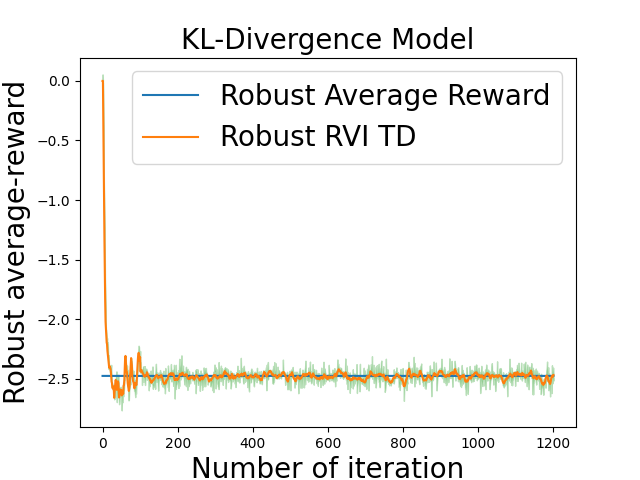}}
\subfigure{
\label{Fig.FLWDTD}
\includegraphics[width=0.3\linewidth]{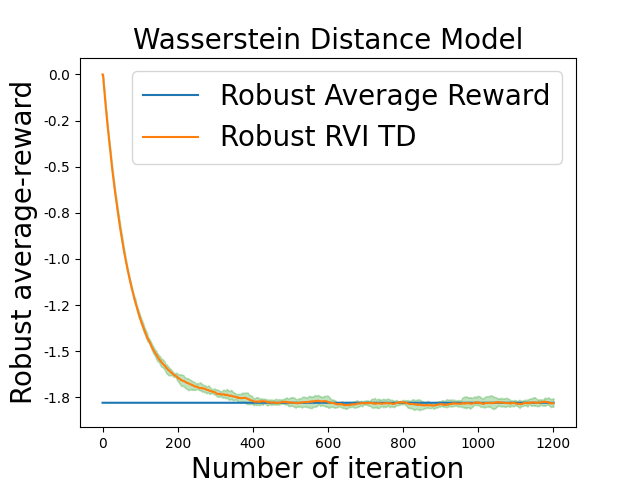}}
\caption{Robust RVI TD Algorithm under Frozen-Lake environment.}
\label{Fig.FLTD}
\end{center}
\vskip -0.2in
\end{figure}

\begin{figure}[!htp] 
\begin{center}
\subfigure{
\label{Fig.FLconQ}
\includegraphics[width=0.3\linewidth]{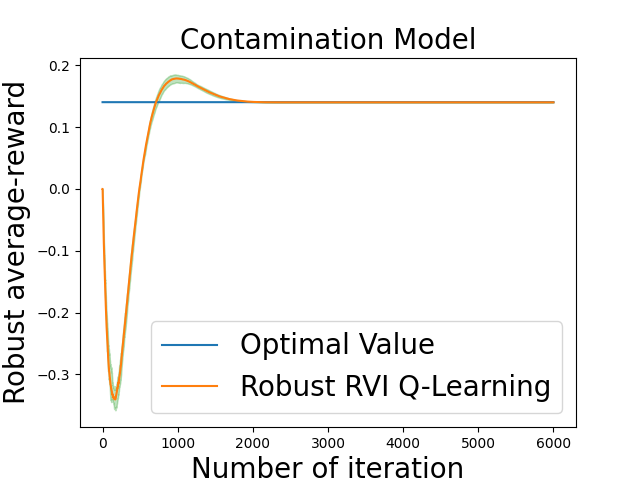}}
\subfigure{
\label{Fig.FLTVQ}
\includegraphics[width=0.3\linewidth]{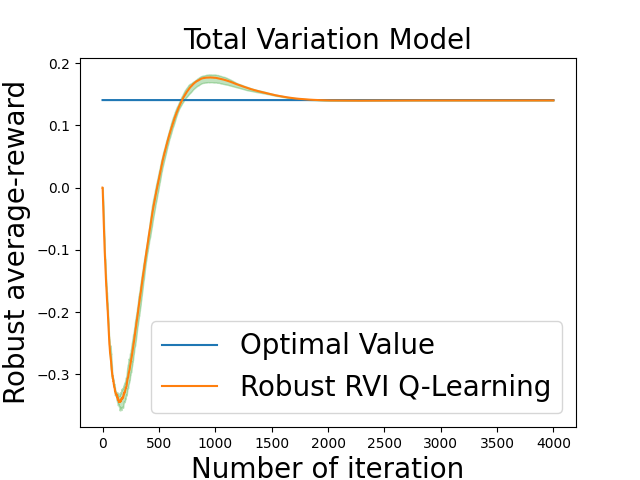}}
\subfigure{
\label{Fig.FLCSQ}
\includegraphics[width=0.3\linewidth]{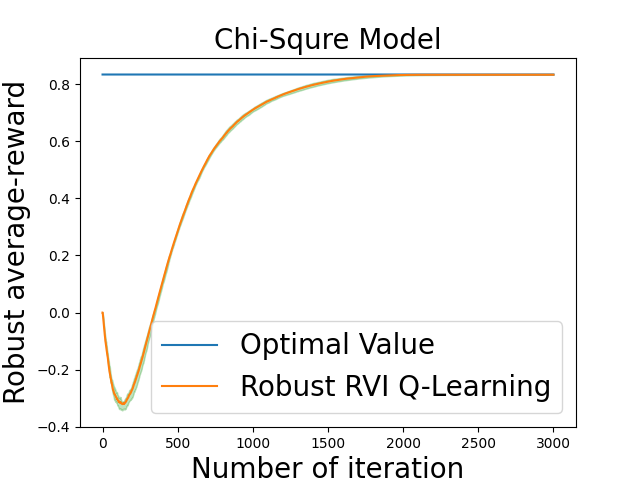}}
\subfigure{
\label{Fig.FLKLQ}
\includegraphics[width=0.3\linewidth]{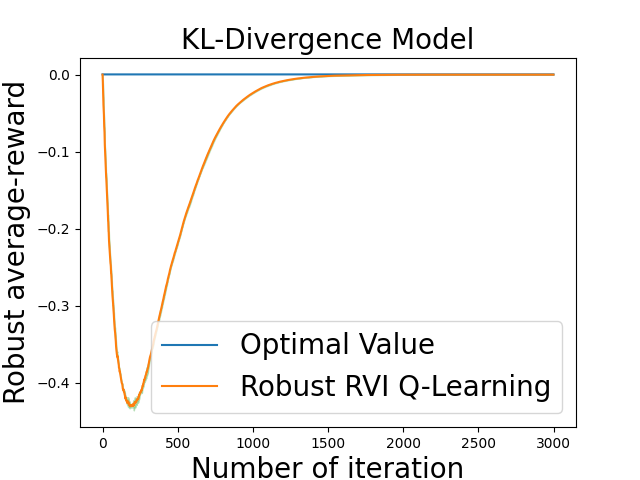}}
\subfigure{
\label{Fig.FLWDQ}
\includegraphics[width=0.3\linewidth]{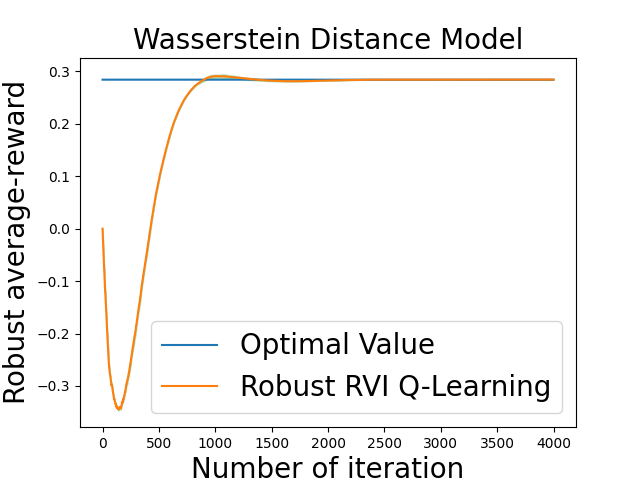}}
\caption{Robust RVI Q-learning Algorithm under Frozen-Lake environment.}
\label{Fig.FLQ}
\end{center}
\vskip -0.2in
\end{figure}

\newpage
\subsection{Robustness of Robust RVI Q-Learning}
We further use the simple, yet widely-used problem, referred to as the one-loop task problem \cite{panaganti2021sample}, to verify the robustness of our robust RVI Q-learning. This environment is widely used to demonstrate that robust methods can learn different optimal polices from the non-robust methods, which are more robust to model uncertainty. The one-loop MDP contains 2 states $s_1,s_2$, and 2 actions $a_l,a_r$ indicating going left or right.

The nominal environment is shown in the left of \Cref{Fig.simple MDP}, where at state $s_1$, going left and right will result in a transition to $s_1$ or $s_2$; and at $s_2$,  going left and right will result in a transition to $s_1$ or $s_2$. 
\begin{figure}[!htp] 
\centering
\begin{tikzpicture}[node distance=1.7cm]
\tikzstyle{place}=[circle,thick,draw=gray!75,fill=gray!20,minimum size=6mm]

\begin{scope}
\node [place] (s1c) {$s_1$};
\node [place, right of=s1c] (s2c) {$s_2$};

\draw[thick,->,shorten >=1pt] (s2c) to [out=225,in=315] (s1c) node[pos=.85,sloped, below=5mm, color=red] {$a_l,0$};
\draw[thick,->,shorten >=1pt] (s1c) to [out=45,in=135] (s2c)node[sloped, above=5mm,color=red] {$a_r,-2$};
\draw[thick,->,shorten >=1pt] (s1c) to [out=90,in=180,loop,looseness=4.8] (s1c)node[pos=.65, sloped, above=6mm, color=red]{$a_l,0$};
\draw[thick,->,shorten >=1pt] (s2c) to [out=-90,in=0,loop,looseness=4.8] (s2c)node[sloped, below=6mm, color=red]{$a_r,1$};

\node [place,right of=s2c] (s3c) {$s_1$};
\node [place, right of=s3c] (s4c) {$s_2$};

\draw[thick,->,shorten >=1pt] (s4c)node[sloped, below=6mm,left=3mm, color=red] {$a_l,0$} to [out=225,in=315] (s3c) ;
\draw[thick,->,shorten >=1pt] (s3c) to [out=45,in=135] (s4c)node[sloped, above=5mm,color=red] {$a_r,-2$};
\draw[thick,->,shorten >=1pt] (s3c) to [out=90,in=180,loop,looseness=4.8] (s3c)node[ sloped, above=6mm, color=red]{$a_l,0$};
\draw[thick,->,shorten >=1pt] (s4c)node[sloped, below=8mm,color=red]{$a_r,1$} to [out=-90,in=-90,loop,looseness=1.8] (s3c);
\end{scope}
\end{tikzpicture}
\caption{One-Loop Task.}
\label{Fig.simple MDP}
\end{figure}

We implement our robust RVI Q-learning and vanilla non-robust Q-learning as the baseline in this environment. At each time step $t$, we plot the difference between $Q_t(s_1,a_l)$ and $Q_t(s_1,a_r)$ in \Cref{Fig.simple1}. If $Q_t(s_1,a_l)-Q_t(s_1,a_r)<0$, the greedy policy will be going right; and if $Q_t(s_1,a_l)-Q_t(s_1,a_r)>0$, the policy will be going left.  As the results show, the vanilla Q-learning will finally learn a policy $\pi(s_1)=a_r$, while our algorithms learn a policy $\pi(s_1)=a_l$. 

To verify the robustness of our method, we test the learned policies under a perturbed testing environment, shown on the right of \Cref{Fig.simple MDP}. We plot the average-reward of policies $\pi_t$ under this perturbed environment. The results are shown in \Cref{Fig.simple2}. 
\begin{figure}[ht]
\begin{center}
\subfigure[$Q(s_1,a_l)-Q(s_2,a_r)$ ]{
\label{Fig.simple1}
\includegraphics[width=0.47\linewidth]{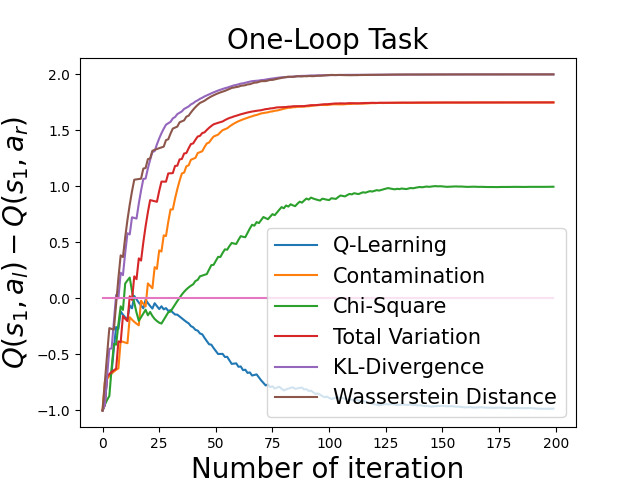}}
\subfigure[Average-Reward under Testing MDP]{
\label{Fig.simple2}
\includegraphics[width=0.47\linewidth]{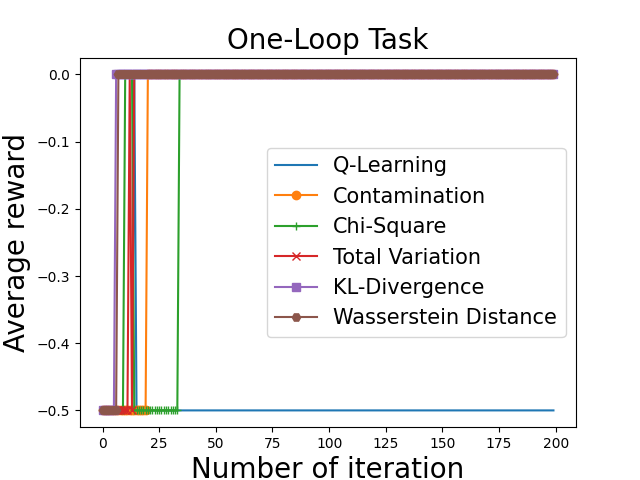}}
\caption{One-Loop Task.}
\label{Fig.simple}
\end{center}
\vskip -0.2in
\end{figure}

As the results show, our robust RVI Q-learning learns a more robust policy under the nominal environment, which obtains a higher reward in the perturbed environment; whereas the non-robust Q-learning learns a policy that is optimal w.r.t. the nominal environment, but less robust when the environment is perturbed. This verifies that our algorithm is more robust than the vanilla method.
\end{document}